\documentclass{article}
\PassOptionsToPackage{numbers,compress}{natbib}
\usepackage[main,final]{neurips_2026}
\usepackage[utf8]{inputenc}
\usepackage[T1]{fontenc}
\usepackage{hyperref}
\usepackage{url}
\usepackage{booktabs}
\usepackage{microtype}
\usepackage{xcolor}
\usepackage{amsmath,amssymb,amsfonts,amsthm,mathtools}
\usepackage{algorithm}
\usepackage{subcaption}
\usepackage{algpseudocode}
\usepackage{enumitem}

\title{Nearly Optimal Fixed-Confidence \\Best-Arm Identification with 1-Bit Feedback}
\author{
Khang Luong$^{1,2}$ \quad Dinh Thai Son$^1$ \quad Hoang Ta$^1$\thanks{Corresponding author: \href{mailto:hoang.taduy@hust.edu.vn}{\texttt{hoang.taduy@hust.edu.vn}}.} \quad Hung The Tran$^2$ \quad Tuan Quang Dam$^1$ \\[3pt]
\normalfont $^1$Hanoi University of Science and Technology \\
\normalfont $^2$Quantum AI \& Cyber Security Institute, FPT Corporation
}

\newtheorem{theorem}{Theorem}
\newtheorem{lemma}[theorem]{Lemma}
\newtheorem{proposition}[theorem]{Proposition}
\newtheorem{corollary}[theorem]{Corollary}
\newtheorem{remark}[theorem]{Remark}

\DeclareMathOperator*{\argmax}{argmax}
\DeclareMathOperator*{\argmin}{argmin}
\newcommand{\E}{\mathbb{E}}
\newcommand{\Var}{\mathrm{Var}}
\newcommand{\Ber}{\mathrm{Ber}}
\newcommand{\R}{\mathbb{R}}
\newcommand{\1}{\mathbf{1}}
\newcommand{\eps}{\varepsilon}
\newcommand{\A}{\mathcal{A}}
\newcommand{\LCB}{\mathrm{LCB}}
\newcommand{\UCB}{\mathrm{UCB}}

\begin{document}
\maketitle

\begin{abstract}
We study fixed-confidence best-arm identification under strict 1-bit feedback constraints. At each round, the learner selects an arm and a query set, and receives only a single bit indicating whether the sampled reward belongs to that set. We consider a distribution-free finite-variance setting with arm-wise localization, where direct empirical mean estimation is no longer available and clipping becomes unavoidable. We first formulate a time-uniform 1-bit mean-estimation primitive based on randomized threshold queries and a clipped tail-integral identity. We then embed this primitive into candidate-challenger best-arm identification algorithms. A fixed-clipping algorithm gives a simple anytime $(\epsilon,\delta)$-PAC guarantee, while a phased adaptive-clipping algorithm matches the clipping level to the current resolution and yields a gap-adaptive sample complexity. We also prove a $K$-arm worst-case information-theoretic lower bound showing that the logarithmic penalty caused by finite-variance 1-bit feedback is intrinsic. This bound matches the leading dependence of the phased algorithm up to logarithmic factors.
\end{abstract}

\section{Introduction}
\label{sec:introduction}
Fixed-confidence best-arm identification (BAI) is the problem of identifying an optimal, or near-optimal, arm with high probability while using as few samples as possible~\citep{evendar2002pac,karnin2013almost,garivier2016optimal}. Classical formulations assume that each pull reveals the reward value itself, so that empirical means and confidence intervals can be formed directly. In this paper, we study BAI under strict \emph{1-bit} feedback constraints: at each round, the learner selects an arm and a query set, and receives only a single bit indicating whether the sampled reward belongs to that set. Two concrete applications instantiate the threshold-query subclass of this model.

\paragraph{Uplink-constrained sensing.} A controller selects one of $K$ sensing locations or configurations. A fresh scalar measurement $X_{k,t}\sim\nu_k$ is compared locally with a programmed threshold $q_t$, and the device sends only $Y_t=\1\{X_{k,t}\ge q_t\}$. The controller seeks a location whose mean pollutant concentration or diagnostic score is within $\eps$ of the largest, with confidence $1-\delta$, before prioritizing an inspection or deploying an expensive high-resolution monitor. Quantized scalar sensing motivates this response mechanism~\citep{ribeiro2006bandwidth}; adaptive arm selection and fixed-confidence identification are our specialization.

\paragraph{Posted pricing.} A seller selects a buyer population $k$, posts a price $q_t$ to a fresh buyer with hidden valuation $X_{k,t}\sim\nu_k$, and observes only the purchase bit $Y_t=\1\{X_{k,t}\ge q_t\}$. This is the accept/reject mechanism of posted-price auctions~\citep{kleinberg2003value}. Screening $K$ populations for an $\eps$-optimal mean valuation gives a threshold-query BAI problem. The objective here is mean valuation; revenue maximization is a different objective.

In both applications, the learner chooses the threshold before the fresh sample is drawn. We constrain the response to one bit and assume the threshold can be supplied through a downlink or a programmable comparator. Thus these mechanisms support the queries used by our algorithms.

We focus on distribution-free fixed-confidence BAI under a finite-variance assumption. The main challenge is that under 1-bit feedback, best-arm identification can no longer rely on direct empirical estimates of arm means. Even estimating the mean of a single arm becomes nontrivial: one must reconstruct a real-valued quantity from binary threshold queries while controlling both stochastic variance and approximation bias. Under only a second-moment assumption, this naturally leads to a clipping or truncation step~\cite{lau2025sequential,laujohnathansequential}, and hence to a nontrivial bias--variance tradeoff. In the bandit setting, this difficulty is compounded by adaptivity: the learner must compare arms and decide when to stop using confidence intervals that remain valid uniformly over time, even though the number of pulls allocated to each arm is itself data-dependent. A fixed-target 1-bit mean estimator is not directly sufficient for a per-pull rule whose local sample counts are random and data-dependent. One needs either a time-uniform confidence sequence, as in Theorem~\ref{thm:confidence-bound}, or an outer wrapper that invokes fixed-target estimators at prespecified checkpoints with simultaneous confidence control, as in phased successive elimination.

Our approach is organized around this clipping difficulty. The clipping level controls a bias-variance tradeoff: a larger level reduces clipping bias but enlarges both the variance proxy and the almost-sure range of the 1-bit estimator. The ideal choice therefore depends on the scale at which arms must be distinguished, a scale governed by the unknown gaps. We address this through a 1-bit estimator with time-uniform validity, embedded into two complementary BAI strategies: a fixed clipping level tuned to the target accuracy, and a phased adaptive scheme that matches the clipping level to the current statistical scale. Finally, we prove a $K$-arm equal-gap lower bound showing that the logarithmic penalty is intrinsic to finite-variance 1-bit feedback. It matches the leading $K$, gap, and confidence dependence of the phased upper bound on these worst-case instances, up to logarithmic factors.

\subsection{Problem setting}

We study fixed-confidence best-arm identification under 1-bit feedback. The learner is given a finite collection of arms $\A = \{1,2,\dots,K\}$.
Pulling arm $k \in \A$ produces an i.i.d. reward drawn from an unknown distribution $\nu_k$ on $\R$ with mean $\mu_k := \E_{X\sim \nu_k}[X]$.

Throughout the paper, we assume that there is a known variance proxy $\sigma>0$ such that
\[
\Var_{X\sim \nu_k}(X) \le \sigma^2
\qquad\text{for all } k\in\A.
\]
In addition, for each arm $k$ we assume that we are given an anchor point $c_k$ satisfying
\[
|\mu_k-c_k|\le \sigma.
\]
We assume that the best arm is unique, and denote it by $k^\star = \argmax_{k\in\A}\mu_k$, and $\mu^\star := \mu_{k^\star}$.

\begin{remark}[Role of assumptions]
\label{rem:assumptions}
The variance bound is the basic moment assumption under which our 1-bit estimator operates. The anchor condition plays a different role: it gives a localization of each arm mean, which is exactly what allows the 1-bit estimator to work efficiently. This localization can itself be obtained by the preliminary 1-bit routine \textup{\textsc{MedianEst}}, using at most $O\left( K \log\frac{K\lambda}{\delta\sigma}\right)$ threshold queries if $\mu \in [-\lambda,\lambda]$ with known $\lambda$; see Appendix~\ref{app:medianest}. Furthermore, while we use $|\mu_k-c_k|\le \sigma$ for convenience, an anchor distance of $|\mu_k-c_k|\le O(\sigma)$ is adequate and only alters the constant factors in the bounds, as detailed in the Appendix~\ref{app:medianest}.
\end{remark}

For every suboptimal arm $k\neq k^\star$, define the gap $\Delta_k := \mu^\star-\mu_k > 0$, and for the optimal arm define $\Delta_{k^\star} := \min_{j\neq k^\star}\Delta_j$.
These gaps quantify the intrinsic difficulty of the instance: small gaps mean that the learner must distinguish arms with very close means, and hence must collect more information before stopping.

\paragraph{1-bit interaction model.}
The learner does not observe rewards directly. Instead, at round $t$, it chooses an arm $k_t\in\A$ together with a measurable query set $A_t\subseteq \R$. The agent then draws a fresh reward sample $X_{k_t,t}\sim \nu_{k_t}$ and returns only the single bit
\[
Y_t = \1\{X_{k_t,t}\in A_t\}\in\{0,1\}.
\]
Both the arm choice and the query set may depend on the entire past history, so the interaction is fully adaptive. This is the same 1-bit set-query model used in the lower-bound construction. In the algorithmic part of the paper, we only use randomized threshold queries, which form a particularly simple subclass of measurable set queries.

\paragraph{Objective.}
Fix an accuracy parameter $\eps>0$. We define the set of $\eps$-optimal arms by
\[
A_\eps(\nu) := \{k\in\A : \mu_k \ge \mu^\star-\eps\}.
\]
A 1-bit best-arm identification algorithm consists of an adaptive query strategy, a stopping time $\tau$, and a recommendation $\widehat{k}\in\A$. It is said to be $(\eps,\delta)$-PAC if $\Pr_\nu(\tau<\infty)=1$ and $\Pr_\nu\bigl(\widehat{k}\notin A_\eps(\nu)\bigr)\le \delta$ for every bandit instance $\nu=(\nu_1,\dots,\nu_K)$ satisfying the above assumptions. Equivalently, with probability at least $1-\delta$, the output arm satisfies $\mu_{\widehat{k}} \ge \mu^\star-\eps$.

This $\eps$-relaxed objective is natural in our setting. It allows the learner to stop once it has ruled out all arms that are worse than the optimum by more than $\eps$, without spending additional samples to separate nearly tied arms. In particular, when $\eps < \Delta_{k^\star},$ the set $A_\eps(\nu)$ reduces to $\{k^\star\}$, and the $(\eps,\delta)$-PAC requirement becomes exact best-arm identification.

\subsection{Contributions}

\paragraph{A time-uniform 1-bit mean-estimation primitive for adaptive BAI.} In Section~\ref{sec:estimator}, we formulate a randomized-threshold mean-estimation primitive based on a clipped tail-integral identity (Lemma~\ref{lem:tail-integral}). The resulting estimator admits confidence sequences that are valid uniformly over time (Theorem~\ref{thm:confidence-bound}), which is essential for BAI because the number of samples collected from each arm is random and data-dependent. As a standalone mean estimator, its sample complexity matches the optimal finite-variance 1-bit mean-estimation rate of~\citet{laujohnathansequential}, up to the standard iterated-logarithmic price of time-uniform validity (Corollary~\ref{cor:estimator-sample-complexity}). The contribution here is structural rather than rate-level: rather than improving the standalone mean-estimation rate, we provide a BAI-ready confidence sequence that remains valid under adaptive arm selection and stopping.

\paragraph{Fixed-confidence BAI algorithms under 1-bit feedback.} In Section~\ref{sec:algorithm}, we show how to turn the above primitive into fixed-confidence BAI algorithms. The first algorithm, \textsc{FixedClip-1BitBAI}, uses a single clipping level tuned to the target accuracy $\eps$. It is simple and fully anytime, but its sample complexity contains an additional mixed term caused by the range of the clipped 1-bit statistic (Theorem~\ref{thm:1-bit-bai-upper}). The second algorithm, \textsc{RestartPhasedClip-1BitBAI}, uses phased adaptive clipping: it proceeds from coarse to fine resolutions and chooses the clipping level according to the current resolution. This removes the fixed-clipping mixed term and yields the main gap-adaptive upper bound of the paper (Theorem~\ref{thm:phased-bai-upper}). Both algorithms are $(\eps,\delta)$-PAC, and their sample complexities are compared in Subsection~\ref{subsec:sample-complexity-discussion}.

\paragraph{A matching information-theoretic lower bound for 1-bit BAI.} In Section~\ref{sec:lower_bound}, we prove a worst-case lower bound for fixed-confidence BAI under 1-bit set queries. The construction is closely related in spirit to multiscale hard instances for finite-variance 1-bit mean estimation of~\citet{laujohnathansequential}, but is formulated directly as a $K$-arm BAI lower bound and combined with a change-of-measure argument for adaptive stopping. For a $K$-arm instance with all suboptimal gaps equal to $\Delta$, we show that any algorithm that is $\delta$-correct over the instance class must have an expected stopping time of at least $\Omega\left( (K-1)\frac{\sigma^2}{\Delta^2} \log\frac{\sigma}{\Delta} \log\frac{1}{\delta} \right)$ (Theorem~\ref{thm:lower-bound}). This demonstrates that the $\Theta(\log\frac{\sigma}{\Delta})$ penalty is an unavoidable consequence of the 1-bit feedback constraint, proving that our algorithm is asymptotically optimal up to logarithmic factors.

\paragraph{Empirical validation.} Section~\ref{sec:numerical} evaluates our 1-bit mean estimator and best-arm identification algorithms. For standalone estimation, our randomized estimator achieves lower mean absolute error and a smaller bias floor than the finite-variance baseline of \citet{laujohnathansequential}. In the tested bandit settings, phased adaptive clipping uses fewer queries than fixed clipping on instances with heterogeneous gaps.

In the remainder of this paper, the main text presents and discusses our primary theoretical and numerical results, while all formal proofs, detailed algorithms, and experimental configurations are deferred to the appendix.
\subsection{Related work}
\label{sec:related_work}

\paragraph{Fixed-confidence best-arm identification.}
The fixed-confidence PAC framework for bandits and foundational sample-complexity lower bounds for pure exploration were developed in~\citep{evendar2002pac,mannor2004sample}. A large subsequent literature developed gap-adaptive algorithms and instance-dependent complexity measures, including successive elimination methods~\citep{karnin2013almost}, LUCB-style subset-selection algorithms~\citep{kalyanakrishnan2012pac}, UGapE-type candidate-challenger algorithms~\citep{ugape}, and time-uniform confidence methods such as lil'UCB~\citep{jamieson2014lilucb}. Information-theoretic lower bounds and asymptotically optimal sampling rules were developed in~\citep{kaufmann2016complexity,garivier2016optimal}. The closest algorithmic
template to ours is the UGapE framework of Gabillon et al.~\citep{ugape}, which compares a candidate arm against its strongest challenger using confidence
intervals. All of these works assume direct access to the realized reward at each round, so empirical means and confidence intervals can be formed directly; our setting drops this assumption.

\paragraph{Quantized feedback and communication-limited learning.} Communication limits have motivated a broad line of work on statistical inference and learning from quantized or few-bit messages~\citep{braverman2016communication,acharya2020inference1,acharya2020inference2,shah2025generalized}. A particularly relevant subproblem is mean estimation from compressed, distributed, or one-bit observations~\citep{kipnis2022mean,cai2024distributed,kumar2025one,abdalla2026robust}. In bandit problems, limited communication has been studied in regret minimization and distributed exploration settings~\citep{vial2020onebit,hanna2022solving,hillel2013distributed}. Our setting differs in that the objective is fixed-confidence best-arm identification, the feedback consists of adaptive one-bit set queries, and the rewards are only assumed to have finite variance. This weak moment assumption is what forces clipping and leads to the logarithmic penalty analyzed in this paper.

\paragraph{1-bit mean estimation and threshold queries.}
The fundamental statistical task underlying our approach is estimating the mean of a real-valued distribution from binary threshold comparisons. \citet{laujohnathansequential} recently established order-optimal rates for this problem in the finite-variance regime, achieving the sample complexity $O((\sigma^2/\beta^2)\log(\sigma/\beta)\log(1/\gamma))$; crucially, they show that the $\log(\sigma/\beta)$ factor is unavoidable with only a second-moment assumption, since the distribution need not have known bounded support. Our 1-bit estimator matches this rate up to the \(\log\log\) overhead of time-uniform validity, and our lower bound shows that the same \(\log(\sigma/\Delta)\) penalty transfers directly to the BAI problem. In contrast to standalone mean estimation, fixed-confidence BAI requires confidence guarantees that remain valid under adaptive arm selection, elimination, and stopping. Thus, the role of our randomized-threshold primitive is not to improve the standalone mean-estimation rate, but to provide an anytime confidence sequence that can be embedded directly into candidate-challenger BAI algorithms. The 1-bit framework also connects to the quantile bandit setting~\citep{lau2025quantile}, which uses 1-bit feedback to identify the arm with the best quantile rather than the best mean.

\section{A 1-bit mean estimator}
\label{sec:estimator}
We develop the 1-bit mean-estimation primitive through a three-stage construction. First, we establish a tail-integral identity that expresses the mean of a clipped variable as an integral of threshold probabilities. Second, we introduce a randomized thresholding scheme that allows us to estimate this integral using only a single bit per sample while maintaining a controlled bias-variance tradeoff. Finally, we derive a time-uniform confidence sequence for the mean, which enables the estimator to be used in adaptive, data-dependent stopping rules. We also discuss the sample complexity of the estimator in Corollary~\ref{cor:estimator-sample-complexity}, and compare with the baseline of \cite{laujohnathansequential} in Remark~\ref{rem:comparison-lau-johnathan}.

Fix one arm and suppress the arm index. Let $X$ have mean $\mu$, variance at most $\sigma^2$, and let $c$ be an anchor point satisfying $|\mu-c|\le \sigma$. Define the centered variable $W := X-c.$ Then
\[
\E[W] = \mu-c,
\qquad
\E[W^2] \le \Var(W)+(\E[W])^2 \le 2\sigma^2.
\]
Thus $W$ has controlled second moment and is centered near zero. Our task is now to estimate $\E[W]$, and hence $\mu$, using only 1-bit threshold information.

\subsection{A tail-integral identity}

The starting point is an identity that rewrites a clipped value in terms of tail events. For $L>0$, define
\[
[w]_L :=
\begin{cases}
-L, & w < -L,\\
w, & -L \le w \le L,\\
L, & w > L.
\end{cases}
\]
The point of clipping is that it separates the estimation problem into two parts: a bounded quantity, which can be estimated reliably, and a bias term, which we will control using the variance bound.

\begin{lemma}[Tail-integral identity]
\label{lem:tail-integral}
For every $w\in\R$ and every $L>0$,
\[
[w]_L
=
\int_0^L\Bigl(\1\{w\ge t\}-\1\{w\le -t\}\Bigr)\,dt.
\]
Consequently, \quad $\E[[W]_L] = \int_0^L\Bigl(\Pr(W\ge t)-\Pr(W\le -t)\Bigr)\,dt.$
\end{lemma}

This identity is the bridge between mean estimation and 1-bit queries. It says that the clipped-variable mean can be written as an integral of differences between upper-tail and lower-tail probabilities. Since each of those probabilities can be accessed through a threshold comparison, the identity suggests that we should estimate the integral by randomizing the threshold $t$ and asking only a single yes/no question per sample. The proof is given in Appendix~\ref{app:proof-tail-integral}.

\subsection{A 1-bit randomized estimator}

To turn the tail-integral identity into an estimator, we need a distribution over $t$. For $L\ge\sigma$, define
\[
\alpha_L := 1+\log\frac{L}{\sigma},
\qquad
g_L(t) := \frac{1}{\alpha_L\max\{\sigma,t\}},
\qquad t\in[0,L].
\]
Then $g_L$ is a density on $[0,L]$, since
\(\int_0^L g_L(t)\,dt=\alpha_L^{-1}\big(\int_0^\sigma dt/\sigma+\int_\sigma^L dt/t\big)=1.\)

The shape of $g_L$ is important. It places relatively more mass on small thresholds $t$, where the query is informative, while still assigning enough mass to larger thresholds to represent the integral up to $L$.

\paragraph{Randomized query and estimator.}
Draw $S\in\{+1,-1\}$ uniformly, independently of everything else, and draw $T\sim g_L$. Given a fresh reward sample $X$, issue the 1-bit query
\[
Y = \1\{S(X-c)\ge T\}.
\]
Equivalently, when $S=+1$ we ask whether $X\ge c+T$, and when $S=-1$ we ask whether $X\le c-T$. We then define
\[
Z := \frac{2S}{g_L(T)}\,Y.
\]

The next lemma summarizes why this construction works.

\begin{lemma}[Unbiasedness, variance, and range]
\label{lem:estimator-properties}
For every $L\ge \sigma$:

(i) $\E[Z] = \E[[W]_L]$, \quad (ii) $\Var(Z) \le 3\alpha_L\sigma^2$, \quad (iii) $|Z| \le 2\alpha_LL$ almost surely.
\end{lemma}

Each part of the lemma has a different role. The first part shows that $Z$ is exactly unbiased for the clipped-variable mean. The second part gives a variance bound of the correct scale, up to the logarithmic factor $\alpha_L$. The third part provides an almost-sure range bound. Together, the second and third parts are precisely what is needed to apply Bernstein's inequality; the proof is given in Appendix~\ref{app:proof-estimator-properties}.

Of course, $Z$ is unbiased only for the clipped-variable mean $\E[[W]_L]$, not for $\E[W]$ itself. The next lemma quantifies the price paid for clipping.

\begin{lemma}[Clipping bias]
\label{lem:clipping-bias}
For every $L\ge \sigma$,
\[
\bigl|\E[W]-\E[[W]_L]\bigr|
\le \frac{\sigma^2}{2L}.
\]
\end{lemma}

This bound has the expected form: increasing the clipping level $L$ reduces the bias, but making $L$ too large will worsen the variance bound $3\alpha_L\sigma^2$ as well as the range bound $2\alpha_LL$. The estimator therefore has a standard bias-variance tradeoff, and the choice of $L$ needs to balance the two effects. The bias proof is given in Appendix~\ref{app:proof-clipping-bias}.

\subsection{Confidence sequence for the tail-integral estimator}

Let $Z_1,Z_2,\dots$ be i.i.d. copies of the above estimator, each built from an independent reward sample and independent internal randomization. Define
\[
\bar Z_n := \frac1n\sum_{i=1}^n Z_i,
\qquad
\hat\mu_n := c+\bar Z_n.
\]
Since $\bar Z_n$ estimates the clipped-variable mean and $c$ restores the shift, $\hat\mu_n$ is our 1-bit estimate of $\mu$.

\begin{theorem}[Confidence sequence for the tail-integral estimator]
\label{thm:confidence-bound}
For every $L\ge\sigma$ and $\gamma\in(0,1)$, with probability at least $1-\gamma$, $|\hat\mu_n-\mu|\le\beta(n,\gamma,L)$ for all $n\ge1$, where
\[
\beta(n,\gamma,L)=\frac{\sigma^2}{2L}+2\sqrt{\frac{3\alpha_L\sigma^2 x_n(\gamma)}{n}}+\frac{\alpha_L L\,x_n(\gamma)}{n},
\]
with $\alpha_L=1+\log(L/\sigma)$ and $x_n(\gamma):=\log\!\big(\pi^2(1+\lfloor\log_2 n\rfloor)^2/(3\gamma)\big)$.
\end{theorem}

The first term in $\beta(n,\gamma,L)$ is the clipping bias, the second is the main stochastic fluctuation term, and the third is the Bernstein range correction. The only price paid for time-uniformity is the stitched penalty $x_n(\gamma)$, which is of order $\log(1/\gamma)+\log\log(e+n)$. The proof of Theorem~\ref{thm:confidence-bound} is in Appendix~\ref{app:proof-confidence-sequence}, and the sample-complexity inversion below is proved in Appendix~\ref{app:proof-estimator-complexity}.

\begin{corollary}[Sample complexity of the tail-integral estimator]
\label{cor:estimator-sample-complexity}
Fix a target accuracy $\beta \in (0, \sigma)$ and an error probability $\gamma \in (0,1/3)$. By setting the clipping level $L = \frac{2\sigma^2}{\beta}$, the tail-integral estimator $\hat{\mu}_n$ satisfies $|\hat{\mu}_n - \mu| \le \beta$ with probability at least $1-\gamma$ for sample sizes $n$ satisfying:
\[
n = O\left(
	\frac{\sigma^2}{\beta^2}
	\log\frac{\sigma}{\beta} 
	\left[ \log\frac{1}{\gamma} + \log\log\left(e + \frac{\sigma}{\beta}\right) \right] \right).
\]
\end{corollary}

\begin{remark}[Comparison with \cite{laujohnathansequential}]
\label{rem:comparison-lau-johnathan}
In the finite-variance regime, \cite{laujohnathansequential} achieve the optimal rate $O\!\left(\frac{\sigma^2}{\beta^2}\log\frac{\sigma}{\beta}\log\frac{1}{\gamma}\right).$ Our estimator matches this rate up to the lower-order $\log\log$ factor needed for time-uniform validity. Structurally, their estimator partitions the domain into geometric regions and assigns fixed regional budgets, whereas ours samples a single randomized threshold from $g_L$ and yields an anytime confidence sequence. This avoids fixed regional allocations and makes the estimator compatible with adaptive bandit sampling and stopping rules.
\end{remark}

\section{1-bit BAI algorithm and upper bound}
\label{sec:algorithm}



We now convert the time-uniform 1-bit mean confidence sequence from Section~\ref{sec:estimator} into fixed-confidence BAI algorithms. Our algorithms follow the candidate-challenger interval logic of~\cite{ugape}. At each time, we construct arm-wise confidence intervals, select a candidate arm $h$, identify its strongest challenger $\ell$, and stop once the candidate is certified to be $\eps$-optimal. If the stopping condition is not met, the algorithm samples the more uncertain arm among $\{h,\ell\}$, namely the one with the larger confidence radius.

The main new difficulty, compared with the unquantized setting, is the choice of the clipping level $L$. A larger clipping level reduces the bias of the 1-bit mean estimator but increases both its variance factor and its range term. Thus, the clipping level must be matched to the statistical resolution at which arms are being compared.

We present two approaches. The first is a fixed-clipping algorithm, which chooses a single clipping level tuned to the target accuracy $\eps$. This algorithm is simple and fully anytime, but it pays an additional mixed term in the sample complexity; see Theorem~\ref{thm:1-bit-bai-upper}. The second is a phased adaptive-clipping algorithm, which gradually decreases the statistical resolution and chooses the clipping level accordingly. This removes the mixed term and yields the main gap-adaptive upper bound of the paper; see Theorem~\ref{thm:phased-bai-upper}. Both algorithms are $(\eps,\delta)$-PAC, and their sample complexities are compared in Subsection~\ref{subsec:sample-complexity-discussion}.

Consider the non-trivial case $0<\eps\le\sigma$; the case \(\varepsilon>\sigma\) is trivial up to constants and can be handled by the initial phase alone. Define $\widetilde\Delta_k:=\max\{\Delta_k,\eps\},$ and $\overline\Delta_k:=\min\{\widetilde\Delta_k,\sigma\}.$

\subsection{Fixed clipping bias}
\label{subsec:fixed-clipping-bias}

The simplest implementation chooses the clipping level from the target accuracy:
\[
L_\eps:=\frac{2\sigma^2}{\eps},
\qquad
\alpha_\eps:=1+\log\frac{L_\eps}{\sigma}=1+\log\frac{2\sigma}{\eps}.
\]
By Lemma~\ref{lem:clipping-bias}, the clipping bias is at most $\eps/4$. The confidence radius from Theorem~\ref{thm:confidence-bound} becomes
\[
\beta_\eps(n,\gamma)
=
\frac{\eps}{4}
+2\sqrt{\frac{3\alpha_\eps\sigma^2x_n(\gamma)}{n}}
+\frac{2\alpha_\eps\sigma^2x_n(\gamma)}{n\eps}.
\]
Algorithm \textsc{FixedClip-1BitBAI} uses this radius with $\gamma=\delta/K$. At each round it forms confidence intervals $[\LCB_k,\UCB_k]$ of $\mu_k$, sets $B_k=\max_{a\ne k}\UCB_a-\LCB_k$, lets $h\in\argmin_kB_k$ and $\ell\in\argmax_{a\ne h}\UCB_a$, stops if $B_h\le\eps$, and otherwise samples the more uncertain arm in $\{h,\ell\}$ using the 1-bit tail-integral query at clipping level $L_\eps$. Detailed pseudocode is given in Algorithm~\ref{alg:fixed-clip-bai} in Appendix~\ref{app:detailed-algorithms}.

From the expression of $\beta_\eps(n,\gamma)$, the sample requirement is no longer governed only by the usual $n^{-1/2}$ concentration term. In addition, there is a lower-order $n^{-1}\eps^{-1}$ term coming from the Bernstein range correction. This term is a direct consequence of the clipping level $L_\eps$ chosen so that the clipping bias is at most $\eps/4$, affecting the final sample complexity. This leads to the hardness quantity:
\[
H_{\mathrm{fix}}(\eps)
:=
\alpha_\eps
\sum_{k=1}^K
\left(
\frac{\sigma^2}{\widetilde\Delta_k^2}
+
\frac{\sigma^2}{\eps\widetilde\Delta_k}
\right).
\]

With this hardness quantity, we can state the sample complexity upper bound.
\begin{theorem}[Fixed clipping upper bound]
\label{thm:1-bit-bai-upper}
\textup{\textsc{FixedClip-1BitBAI}} is $(\eps,\delta)$-PAC. With probability at least $1-\delta$, its needed number of 1-bit queries satisfies
\[
N_{\mathrm{fix}}
=
O\!\left(
H_{\mathrm{fix}}(\eps)
\left[\log\frac{K}{\delta}+\log\log H_{\mathrm{fix}}(\eps)\right]
\right).
\]
\end{theorem}

The proof of Theorem~\ref{thm:1-bit-bai-upper} is given in Appendix~\ref{app:proof-fixed-clipping}. A finite deterministic query bound also holds on every realization for $\eps>0$, guaranteeing termination even outside the confidence event; together with interval correctness this establishes the PAC guarantee. The deterministic bound and proof are given in Appendix~\ref{app:deterministic-fixed}.


\begin{remark}[The choice of clipping level $L$]
\label{rem:clipping-level}
The fixed choice $L_\eps=2\sigma^2/\eps$ is simple and ensures bias at most $\eps/4$, but it is not instance-optimal, especially when $\eps \ll \Delta_k$ or $\eps=0$. For exact best-arm identification $\eps=0$, the natural scale is $L_\star=\Theta(\sigma^2/\Delta_{k^\star})$; more generally, the oracle arm-wise choice would be $L_k=\Theta(\sigma^2/\widetilde\Delta_k)$, leading to the gap-adaptive hardness
\[
H_{\mathrm{gap}}
=
\sum_{k=1}^K
\left(1+\log\frac{\sigma}{\overline\Delta_k}\right)
\frac{\sigma^2}{\overline\Delta_k^2}.
\]
\end{remark}
The phased algorithm below obtains this behavior without knowing the gaps.

\subsection{Phased adaptive clipping}
\label{subsec:phased-adaptive-clipping}
The intuition behind the phased algorithm is that easy arms should not be estimated at the finest accuracy scale. In \textsc{FixedClip-1BitBAI}, the clipping level is fixed using the target accuracy $\eps$, so even an arm with a large gap pays the range cost associated with resolving $\eps$-level differences. The phased algorithm instead proceeds from coarse to fine resolutions. At phase $j$, it uses a clipping level matched to the current resolution $\rho_j$ and eliminates arms that are already distinguishable at that scale. Only statistically competitive arms survive to finer phases. This removes the fixed-clipping mixed term and yields a bound depending on each arm's own effective gap.

To implement this idea, schedule resolutions
\[
\rho_j:=2^{-j}\sigma,\qquad J:=\lceil\log_2(\sigma/\eps)\rceil,\qquad j=0,1,\dots,J,
\]
and at phase $j$ set $L_j:=32\sigma^2/\rho_j$, $\alpha_j:=1+\log(32\sigma/\rho_j)$, and $r_j:=\rho_j/16$. With $\gamma=\delta/[K(J+1)]$,
\[
\beta_j(n)=\frac{\rho_j}{64}+2\sqrt{\frac{3\alpha_j\sigma^2 x_n(\gamma)}{n}}+\frac{32\alpha_j\sigma^2 x_n(\gamma)}{\rho_j n}.
\]

The choice of $L_j$ makes the clipping bias of phase $j$ at most $\rho_j/64$, while $r_j=\rho_j/16$ is the target statistical radius at that resolution. Thus, phase $j$ should be viewed as a test of whether the currently active arms can already be separated at scale $\rho_j$. If the confidence intervals are still too wide, the algorithm continues sampling; if all active arms have confidence radius at most $r_j$, then arms that are clearly suboptimal at this scale can be safely removed.

Algorithm~\textsc{RestartPhasedClip-1BitBAI} implements this idea using fresh estimators at every phase; detailed pseudocode is given in Algorithm~\ref{alg:phased-clip-bai} in Appendix~\ref{app:detailed-algorithms}. Within phase $j$, it adaptively samples active arms until either some arm is certified to be $\eps$-optimal or all active arms have confidence radius at most $r_j$. In the latter case, it eliminates every arm whose UCB is below the best LCB by more than $\eps$, and then moves to the next, finer phase $j+1$.

\begin{theorem}[Phased adaptive clipping upper bound]
\label{thm:phased-bai-upper}
\textup{\textsc{RestartPhasedClip-1BitBAI}} is $(\eps,\delta)$-PAC. With probability at least $1-\delta$, the needed number of 1-bit queries satisfies
\[
\tau
=
O\!\left(
H_{\rm gap}
\left[
\log\frac{K}{\delta}
+
\log\log\frac{\sigma}{\eps}
\right]
\right),
\]
\end{theorem}

The proof of Theorem~\ref{thm:phased-bai-upper} is given in Appendix~\ref{app:proof-phased-clipping}. A finite deterministic query bound also holds on every realization for $\eps>0$, guaranteeing termination even outside the confidence event; together with interval correctness this establishes the PAC guarantee. The deterministic bound and proof are given in Appendix~\ref{app:deterministic-phased}.


\begin{remark}[Reuse instead of restart]
\label{rem:reuse-version}
The restart version is chosen for proof clarity. In practice one can reuse samples across clipping levels using
\[
[W]_{L_j}
=
\sum_{s=0}^j
\int_{I_s}
\bigl(\1\{W\ge t\}-\1\{W\le -t\}\bigr)dt,
\qquad
I_0=[0,L_0],\quad I_s=(L_{s-1},L_s]\ (s\ge1).
\]
Thus, when moving from phase \(j-1\) to phase \(j\), it suffices to introduce the new slab \(I_j\) and top up only those old slabs whose sample counts are still deficient for the finer phase-\(j\) target radius. Lemma~\ref{lem:reuse-phased-counting} shows that this reuse strategy preserves the same order as Theorem~\ref{thm:phased-bai-upper}.
\end{remark}

\subsection{Discussion of the sample complexity}
\label{subsec:sample-complexity-discussion}

\paragraph{Fixed clipping vs Phased clipping.} The fixed method is simpler and fully anytime, but it pays the mixed term $\sigma^2/(\eps\widetilde\Delta_k)$ because the range of the 1-bit statistic is proportional to $L_\eps=\Theta(\sigma^2/\eps)$. This term is unavoidable for a globally fixed clipping level, but it is wasteful for arms with $\Delta_k\gg\eps$. The phased method matches $L_j$ to the current resolution $\rho_j$, so the range term is balanced at the scale where the arm is actually distinguished; its cost is essentially $(1+\log(\sigma/\overline\Delta_k))\sigma^2/\overline\Delta_k^2$ per arm.

\paragraph{Compared with unquantized fixed-confidence BAI.} The distribution-free hardness of unquantized BAI is $\sum_k\sigma^2/\widetilde\Delta_k^2$ (\cite{ugape}), 1-bit feedback introduces the additional logarithmic factor $1+\log(\sigma/\overline\Delta_k)$ in the phased bound. The lower bound in Section~\ref{sec:lower_bound} shows that this logarithmic penalty is inevitable in the finite-variance regime.

\paragraph{Lau--Scarlett \cite{laujohnathansequential} with phased successive elimination.} At each resolution $\rho_j$, run their fixed-target estimator on each active arm with accuracy $\rho_j/8$ and confidence $1-\delta/[K(J+1)]$, then retain estimates within $\rho_j/2$ of the maximum. Simultaneous confidence control yields total cost $O(H_{\rm gap}\log(K(J+1)/\delta))$, matching our leading rate. Appendix~\ref{app:ls-successive-elimination} gives the correctness and counting argument.

\paragraph{Total sample complexity.} If the anchors are not initially available but the means lie in $[-\lambda,\lambda]$, running \textsc{MedianEst} for each arm adds an additional $O\!\left(K(\log\frac{\lambda}{\sigma}+\log\frac{K}{\delta})\right)$ term to the sample complexity, which does not scale with $\eps$ or the gaps $\Delta_k$.

\section{Lower bound}
\label{sec:lower_bound}

We now establish an information-theoretic lower bound for fixed-confidence best-arm identification under 1-bit set queries. The construction is closely related in spirit to the multiscale hard instances that arise in finite-variance 1-bit mean estimation~\cite{laujohnathansequential}: in both settings, one-bit queries cannot avoid a logarithmic loss under only a second-moment assumption. Here, we formulate this obstruction directly in the BAI setting by embedding opposite perturbations of a multiscale one-dimensional distribution into a $K$-arm instance and combining it with a change-of-measure argument for adaptive stopping. The construction uses a symmetric distribution $P_0$ of mean zero and perturbed distributions $P_-$ and $P_+$ of means $-\Delta$ and $+\Delta$.

The base instance $(P_0,P_-,\ldots,P_-)$ has $K-1$ equal suboptimal gaps $\Delta$. For each $i\ge2$, changing only arm $i$ to $P_+$ changes the best arm to $i$. A change of measure forces sufficient queries to each such arm under the same base instance, and summing gives the factor $K-1$.

\begin{theorem}[$K$-arm worst-case lower bound for 1-bit BAI]
\label{thm:lower-bound}
Fix $K\ge2$, $\sigma>0$, $\delta\in(0,1/4)$, and $\Delta\in(0,\sigma/4]$. There exists a $K$-arm instance $\nu^{(0)}$ with variance at most $\sigma^2$ for every arm and all suboptimal gaps equal to $\Delta$, such that every 1-bit algorithm that is $\delta$-correct for exact BAI over the stated instance class satisfies
\[
\E_{\nu^{(0)}}[\tau]\ge c(K-1)\frac{\sigma^2}{\Delta^2}
\log\frac{\sigma}{\Delta}\log\frac1\delta
\]
for an absolute constant $c>0$. The same bound holds for algorithms that are $(\eps,\delta)$-PAC over this class when $0<\eps<\Delta/2$.
\end{theorem}

All base and alternative distributions admit the same anchor $c_k=0$, so the bound persists when valid anchors are supplied. The construction, uniform one-query KL bound, and adaptive change-of-measure proof appear in Appendices~\ref{app:hard-distributions}--\ref{app:lower-bound-change-measure}. On the equal-gap base instance, $H_{\rm gap}=K(1+\log(\sigma/\Delta))\sigma^2/\Delta^2$. When $\delta\le1/K$, $\log(1/\delta)$ and $\log(K/\delta)$ are equivalent up to constants. Hence the expected lower bound matches the leading $K$, gap, and confidence dependence of the high-probability phased upper bound, up to lower-order $\log\log$ factors. This worst-case comparison does not characterize the optimal expected complexity on every instance.

\section{Numerical experiments}
\label{sec:numerical}

\subsection{1-bit mean estimation}
We compare our tail-integral estimator proposed in Section~\ref{sec:estimator} against the baseline of \cite{laujohnathansequential} in the finite-variance regime. The experiment has two goals: first, to compare the empirical mean absolute error (MAE) of the two estimators; second, to verify the qualitative behavior predicted by Theorem~\ref{thm:confidence-bound}, namely that the error is governed by a decaying stochastic term together with a deterministic bias term induced by clipping/truncation.

\begin{figure}[ht]
\centering
\includegraphics[width=\linewidth]{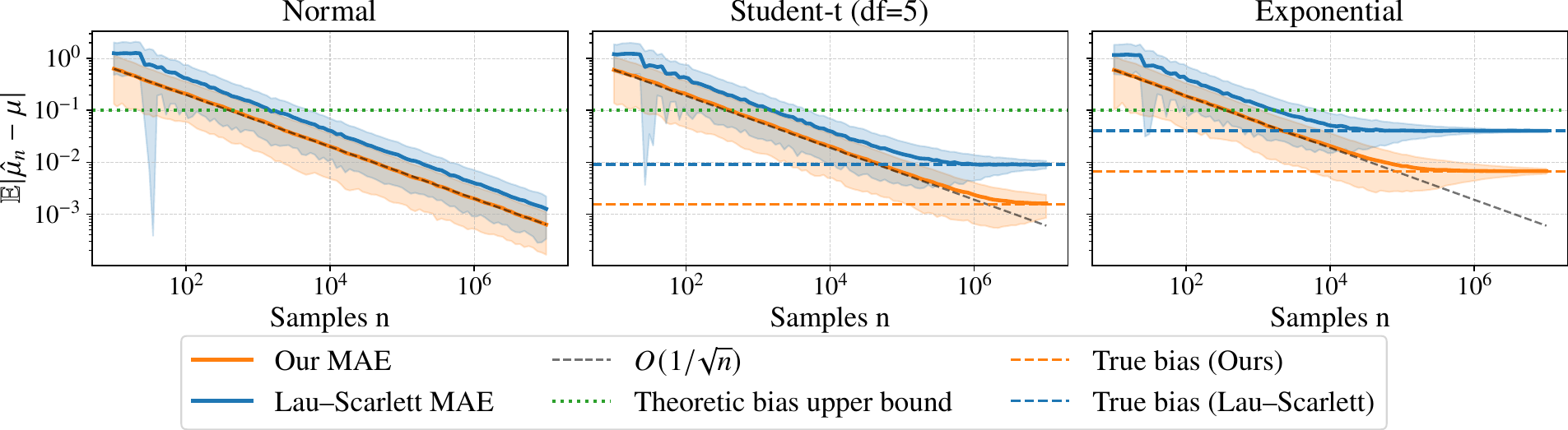}
\caption{Mean absolute error $\mathbb{E}|\hat{\mu}_n-\mu|$ over the number of 1-bit queries $n$ for our estimator and the estimator of \cite{laujohnathansequential} (mean $\pm$ std over runs). For readability, bias lines below $10^{-3}$ are not shown.}
\label{fig:1-bit-estimator-mae}
\end{figure}

Detailed calibration and setup are given in Appendix~\ref{app:numerical-estimation}; Appendix~\ref{app:numerical-variance} compares variances of the signed estimates. Figure~\ref{fig:1-bit-estimator-mae} illustrates the MAE relative to the number of 1-bit queries. While both estimators show the expected $O(n^{-1/2})$ decay in the Gaussian case, our estimator consistently achieves a lower MAE. For Student-$t$ and exponential distributions, our method reaches a lower asymptotic error floor due to its smaller true deterministic bias at the same clipping/truncation level. Under this calibration, the bias term is practically significant and our estimator has lower MAE and clipping bias. The signed-estimate standard deviations are also smaller (Table~\ref{tab:1-bit-variance}); these empirical differences do not imply a rate separation.

\subsection{1-bit best-arm identification}
We evaluate the empirical performance of our proposed 1-bit best-arm identification algorithms across three distinct experimental setups. To analyze the sample complexity and safety of each approach, the confidence parameter $\delta$ is varied from $0.01$ to $0.5$. We compare three variations of our 1-bit algorithm—\textsc{FixedClip}, \textsc{RestartPhase}, \textsc{ReusePhase}—against the full real-valued-feedback \textsc{UGapE} baseline of \citet{ugape}. Appendix~\ref{app:numerical-bai} gives the full setup, query cap, and exploration tuning. The tuned implementations are assessed empirically.

\begin{figure}[ht]
\centering
\includegraphics[width=\linewidth]{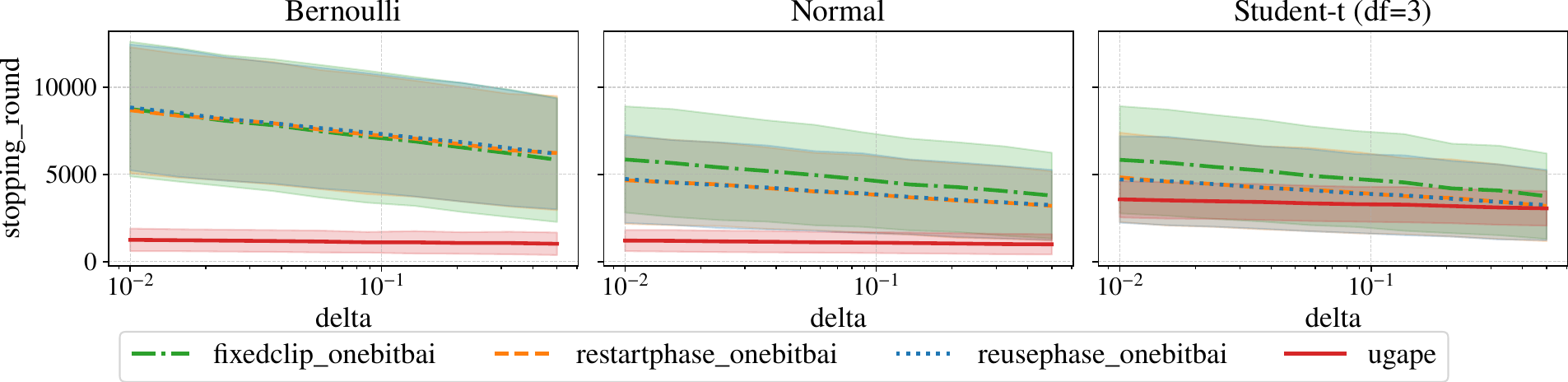}
\caption{Stopping rounds of best-arm identification algorithms as a function of the confidence parameter $\delta$ (mean $\pm$ std over runs).}
\label{fig:1-bit-best-arm-identification}
\end{figure}

Figure \ref{fig:1-bit-best-arm-identification} illustrates the stopping round as a function of $\delta$ across the three experimental settings. The performance gap between the 1-bit algorithms and the unquantized baseline is most significant in the Bernoulli environment and relatively smaller in the Normal and Student-t settings. When gaps are uniform, as in the Bernoulli test where $\Delta_k$ is identical for all suboptimal arms, the three 1-bit algorithms perform similarly because no arms can be eliminated early at coarser resolutions. However, in the Normal and Student-t settings where $\Delta_k$ varies, the phased adaptive-clipping algorithms demonstrate a clear advantage by eliminating arms with larger gaps much earlier than the fixed-clipping approach.

\section{Conclusion}
\label{sec:conclusion}
We introduced a fixed-confidence best-arm identification framework for finite-variance, 1-bit feedback environments. Using a novel randomized-threshold mean estimator, we developed a baseline fixed-clipping algorithm alongside a gap-adaptive phased alternative which achieves worst-case optimality up to logarithmic factors. Addressing current limitations offers clear avenues for future work. First, our algorithms rely on a localization assumption; developing methods that weaken or remove this requirement is an important next step. Second, our analysis currently relies on a K-arm worst-case lower bound, leaving the exact instance-dependent complexity of 1-bit BAI open. Finally, extending our time-uniform 1-bit primitive to broader pure-exploration tasks presents an exciting direction for communication-constrained learning.

\section*{Acknowledgments}
We thank Prof. Jonathan Scarlett for his helpful comments and suggestions. The author Hung The Tran is funded by the Quantum AI \& Cyber Security Institute, FPT Corporation.

\newpage
\bibliographystyle{plainnat}
\bibliography{references}

@inproceedings{ugape,
author = {Gabillon, Victor and Ghavamzadeh, Mohammad and Lazaric, Alessandro},
booktitle = {Advances in Neural Information Processing Systems},
editor = {F. Pereira and C.J. Burges and L. Bottou and K.Q. Weinberger},
publisher = {Curran Associates, Inc.},
title = {Best Arm Identification: A Unified Approach to Fixed Budget and Fixed Confidence},
url = {https://proceedings.neurips.cc/paper_files/paper/2012/file/8b0d268963dd0cfb808aac48a549829f-Paper.pdf},
volume = {25},
year = {2012}
}

@InProceedings{noisy_binary_search,
author =	{Gretta, Lucas and Price, Eric},
title =	{{Sharp Noisy Binary Search with Monotonic Probabilities}},
booktitle =	{51st International Colloquium on Automata, Languages, and Programming (ICALP 2024)},
pages =	{75:1--75:19},
series =	{Leibniz International Proceedings in Informatics (LIPIcs)},
ISBN =	{978-3-95977-322-5},
ISSN =	{1868-8969},
year =	{2024},
volume =	{297},
editor =	{Bringmann, Karl and Grohe, Martin and Puppis, Gabriele and Svensson, Ola},
publisher =	{Schloss Dagstuhl -- Leibniz-Zentrum f{\"u}r Informatik},
address =	{Dagstuhl, Germany},
URL =		{https://drops.dagstuhl.de/entities/document/10.4230/LIPIcs.ICALP.2024.75},
URN =		{urn:nbn:de:0030-drops-202188},
doi =		{10.4230/LIPIcs.ICALP.2024.75}
}

@misc{laujohnathansequential,
title={Order-Optimal Sequential 1-Bit Mean Estimation in General Tail Regimes}, 
author={Ivan Lau and Jonathan Scarlett},
year={2026},
eprint={2604.07796},
archivePrefix={arXiv},
primaryClass={stat.ML},
url={https://arxiv.org/abs/2604.07796}, 
}

@inproceedings{evendar2002pac,
author="Even-Dar, Eyal and Mannor, Shie and Mansour, Yishay",
editor="Kivinen, Jyrki and Sloan, Robert H.",
title="PAC Bounds for Multi-armed Bandit and Markov Decision Processes",
booktitle="Computational Learning Theory",
year="2002",
publisher="Springer Berlin Heidelberg",
address="Berlin, Heidelberg",
pages="255--270",
isbn="978-3-540-45435-9"
}

@InProceedings{karnin2013almost,
title = 	 {Almost Optimal Exploration in Multi-Armed Bandits},
author = 	 {Karnin, Zohar and Koren, Tomer and Somekh, Oren},
booktitle = 	 {Proceedings of the 30th International Conference on Machine Learning},
pages = 	 {1238--1246},
year = 	 {2013},
editor = 	 {Dasgupta, Sanjoy and McAllester, David},
volume = 	 {28},
series = 	 {Proceedings of Machine Learning Research},
address = 	 {Atlanta, Georgia, USA},
month = 	 {17--19 Jun},
publisher =    {PMLR},
url = 	 {https://proceedings.mlr.press/v28/karnin13.html}
}

@inproceedings{garivier2016optimal,
title =  {Optimal Best Arm Identification with Fixed Confidence},  
author =  {Garivier, Aurélien and Kaufmann, Emilie},  
booktitle =  {29th Annual Conference on Learning Theory},  
pages =  {998--1027},  
year =  {2016},  
editor =  {Feldman, Vitaly and Rakhlin, Alexander and Shamir, Ohad},  
volume =  {49},  
series =  {Proceedings of Machine Learning Research},  
address =  {Columbia University, New York, New York, USA},  
month =  {23--26 Jun},  
publisher =    {PMLR},  
url =  {https://proceedings.mlr.press/v49/garivier16a.html}
}

@inproceedings{hanna2022solving,
title = 	 {Solving Multi-Arm Bandit Using a Few Bits of Communication },
author =       {Hanna, Osama A. and Yang, Lin and Fragouli, Christina},
booktitle = 	 {Proceedings of The 25th International Conference on Artificial Intelligence and Statistics},
pages = 	 {11215--11236},
year = 	 {2022},
editor = 	 {Camps-Valls, Gustau and Ruiz, Francisco J. R. and Valera, Isabel},
volume = 	 {151},
series = 	 {Proceedings of Machine Learning Research},
month = 	 {28--30 Mar},
publisher =    {PMLR},
url = 	 {https://proceedings.mlr.press/v151/hanna22a.html}
}

@article{vial2020onebit,
publtype={informal},
author={Daniel Vial and Sanjay Shakkottai and R. Srikant},
title={One-bit feedback is sufficient for upper confidence bound policies},
year={2020},
cdate={1577836800000},
journal={CoRR},
volume={abs/2012.02876},
url={https://arxiv.org/abs/2012.02876}
}

@inproceedings{lau2025quantile,
title={Quantile Multi-Armed Bandits with 1-bit Feedback},
author={Ivan Lau and Jonathan Scarlett},
booktitle={36th International Conference on Algorithmic Learning Theory},
year={2025},
url={https://openreview.net/forum?id=u5xNVJlAse}
}

@article{kaufmann2016complexity,
author  = {Emilie Kaufmann and Olivier Capp{{\'e}} and Aur{{\'e}}lien Garivier},
title   = {On the Complexity of Best-Arm Identification in Multi-Armed Bandit Models},
journal = {Journal of Machine Learning Research},
year    = {2016},
volume  = {17},
number  = {1},
pages   = {1--42},
url     = {http://jmlr.org/papers/v17/kaufman16a.html}
}

@inproceedings{hillel2013distributed,
author = {Hillel, Eshcar and Karnin, Zohar S and Koren, Tomer and Lempel, Ronny and Somekh, Oren},
booktitle = {Advances in Neural Information Processing Systems},
editor = {C.J. Burges and L. Bottou and M. Welling and Z. Ghahramani and K. Weinberger},
publisher = {Curran Associates, Inc.},
title = {Distributed Exploration in Multi-Armed Bandits},
url = {https://proceedings.neurips.cc/paper_files/paper/2013/file/598b3e71ec378bd83e0a727608b5db01-Paper.pdf},
volume = {26},
year = {2013}
}

@inproceedings{lau2025sequential,
title={Sequential 1-bit Mean Estimation with Near-Optimal Sample Complexity},
author={Ivan Lau and Jonathan Scarlett},
booktitle={The 29th International Conference on Artificial Intelligence and Statistics},
year={2026},
url={https://openreview.net/forum?id=T8nCp7UYSv}
}

@article{mannor2004sample,
title={The sample complexity of exploration in the multi-armed bandit problem},
author={Mannor, Shie and Tsitsiklis, John N},
journal={Journal of Machine Learning Research},
volume={5},
number={Jun},
pages={623--648},
year={2004},
url={https://www.jmlr.org/papers/volume5/mannor04b/mannor04b.pdf}
}

@inproceedings{kalyanakrishnan2012pac,
author    = {Kalyanakrishnan, Shivaram and Tewari, Ambuj and Auer, Peter and Stone, Peter},
title     = {PAC Subset Selection in Stochastic Multi-armed Bandits},
booktitle = {Proceedings of the 29th International Conference on Machine Learning (ICML 2012)},
pages     = {655--662},
year      = {2012},
url       = {https://icml.cc/2012/papers/359.pdf}
}

@inproceedings{jamieson2014lilucb,
title = 	 {lil' UCB : An Optimal Exploration Algorithm for Multi-Armed Bandits},
author = 	 {Jamieson, Kevin and Malloy, Matthew and Nowak, Robert and Bubeck, Sébastien},
booktitle = 	 {Proceedings of The 27th Conference on Learning Theory},
pages = 	 {423--439},
year = 	 {2014},
editor = 	 {Balcan, Maria Florina and Feldman, Vitaly and Szepesvári, Csaba},
volume = 	 {35},
series = 	 {Proceedings of Machine Learning Research},
address = 	 {Barcelona, Spain},
month = 	 {13--15 Jun},
publisher =    {PMLR},
url = 	 {https://proceedings.mlr.press/v35/jamieson14.html}
}

@inproceedings{braverman2016communication,
author = {Braverman, Mark and Garg, Ankit and Ma, Tengyu and Nguyen, Huy L. and Woodruff, David P.},
title = {Communication lower bounds for statistical estimation problems via a distributed data processing inequality},
year = {2016},
isbn = {9781450341325},
publisher = {Association for Computing Machinery},
address = {New York, NY, USA},
url = {https://doi.org/10.1145/2897518.2897582},
doi = {10.1145/2897518.2897582},
booktitle = {Proceedings of the Forty-Eighth Annual ACM Symposium on Theory of Computing},
pages = {1011–1020},
numpages = {10},
location = {Cambridge, MA, USA},
series = {STOC '16}
}

@inproceedings{acharya2020inference1,
title = 	 {Inference under Information Constraints: Lower Bounds from Chi-Square Contraction},
author =       {Acharya, Jayadev and Canonne, Cl{\'{e}}ment L and Tyagi, Himanshu},
booktitle = 	 {Proceedings of the Thirty-Second Conference on Learning Theory},
pages = 	 {3--17},
year = 	 {2019},
editor = 	 {Beygelzimer, Alina and Hsu, Daniel},
volume = 	 {99},
series = 	 {Proceedings of Machine Learning Research},
month = 	 {25--28 Jun},
publisher =    {PMLR},
url = 	 {https://proceedings.mlr.press/v99/acharya19a.html}
}

@article{acharya2020inference2,
title={Inference under information constraints II: Communication constraints and shared randomness},
author={Acharya, Jayadev and Canonne, Cl{\'e}ment L and Tyagi, Himanshu},
journal={IEEE Transactions on Information Theory},
volume={66},
number={12},
pages={7856--7877},
year={2020},
publisher={IEEE}
}

@article{kipnis2022mean,
author={Kipnis, Alon and Duchi, John C.},
journal={IEEE Transactions on Information Theory}, 
title={Mean Estimation From One-Bit Measurements}, 
year={2022},
volume={68},
number={9},
pages={6276-6296},
doi={10.1109/TIT.2022.3174409}
}

@article{cai2024distributed,
author = {Cai, T. Tony and Wei, Hongji},
title = {Distributed Gaussian mean estimation under communication constraints: optimal rates and communication-efficient algorithms},
year = {2024},
issue_date = {January 2024},
publisher = {JMLR.org},
volume = {25},
number = {1},
issn = {1532-4435},
journal = {J. Mach. Learn. Res.},
month = jan,
articleno = {37},
numpages = {63}
}

@misc{abdalla2026robust,
title={Robust Mean Estimation under Quantization}, 
author={Pedro Abdalla and Junren Chen},
year={2026},
eprint={2601.07074},
archivePrefix={arXiv},
primaryClass={stat.ML},
url={https://arxiv.org/abs/2601.07074}, 
}

@article{kumar2025one,
title={One-Bit Distributed Mean Estimation with Unknown Variance},
author={Ritesh Kumar and Shashank Vatedka},
journal={Transactions on Machine Learning Research},
issn={2835-8856},
year={2026},
url={https://openreview.net/forum?id=g95C4zIEPg},
note={J2C Certification}
}

@inproceedings{shah2025generalized,
author={Shah, Jaimin and Cardone, Martina and Rush, Cynthia and Dytso, Alex},
booktitle={ICASSP 2025 - 2025 IEEE International Conference on Acoustics, Speech and Signal Processing (ICASSP)}, 
title={Generalized Linear Models with 1-Bit Measurements: Asymptotics of the Maximum Likelihood Estimator}, 
year={2025},
pages={1-5},
doi={10.1109/ICASSP49660.2025.10888611}
}

@article{ribeiro2006bandwidth,
author={Ribeiro, A. and Giannakis, G.B.},
journal={IEEE Transactions on Signal Processing}, 
title={Bandwidth-constrained distributed estimation for wireless sensor Networks-part I: Gaussian case}, 
year={2006},
volume={54},
number={3},
pages={1131-1143},
doi={10.1109/TSP.2005.863009}
}

@inproceedings{kleinberg2003value,
  author = {Kleinberg, Robert and Leighton, F. Thomson},
  title = {The Value of Knowing a Demand Curve: Bounds on Regret for Online Posted-Price Auctions},
  booktitle = {Proceedings of the 44th Annual IEEE Symposium on Foundations of Computer Science},
  pages = {594--605},
  year = {2003},
  url = {https://www.cs.cornell.edu/~rdk/papers/oppa.pdf}
}

\newpage
\appendix

\section*{Appendix}
\phantomsection
\label{app:appendix}
This appendix provides additional proofs, full algorithm pseudo-codes, and experimental details referenced in the main text.
Unless stated otherwise, notation matches the main text.

\subsection*{Contents}
\phantomsection
\label{app:appendix-contents}
\begin{center}
\begin{tabular}{@{}p{0.85\linewidth}r@{}}
\toprule
\textbf{Item} & \textbf{Page} \\
\midrule
\hyperref[app:medianest]{Appendix~\ref*{app:medianest}: The MedianEst routine} & \pageref*{app:medianest} \\
\hyperref[app:estimator]{Appendix~\ref*{app:estimator}: Proofs for Section~\ref*{sec:estimator} (A 1-bit mean estimator)} & \pageref*{app:estimator} \\
\hyperref[app:algorithm-proofs]{Appendix~\ref*{app:algorithm-proofs}: Proofs for Section~\ref*{sec:algorithm} (1-bit BAI algorithm and upper bound)} & \pageref*{app:algorithm-proofs} \\
\hyperref[app:lower_bound]{Appendix~\ref*{app:lower_bound}: Proofs for Section~\ref*{sec:lower_bound} (Lower bound)} & \pageref*{app:lower_bound} \\
\hyperref[app:detailed-algorithms]{Appendix~\ref*{app:detailed-algorithms}: Detailed algorithms} & \pageref*{app:detailed-algorithms} \\
\hyperref[app:numerical]{Appendix~\ref*{app:numerical}: Numerical experiment details} & \pageref*{app:numerical} \\
\bottomrule
\end{tabular}
\end{center}
\bigskip

\section{The MedianEst routine}
\label{app:medianest}

\subsection{Bounding the mean via the median}
In Section~\ref{sec:estimator}, we assumed that for each arm $k$, the learner is provided with an anchor point $c_k$ satisfying $|\mu_k - c_k| \le \sigma$. We can rigorously satisfy this assumption by estimating the median of the reward distribution.

Let $c$ be the true median of a random variable $X$ with mean $\mu$ and variance bounded by $\sigma^2$. By Cantelli's inequality (the one-sided Chebyshev inequality), for any $a > 0$, we have:
\[
\Pr(X - \mu \ge a) \le \frac{\sigma^2}{\sigma^2 + a^2}.
\]
If the median is greater than the mean ($c > \mu$), we can substitute $a = c - \mu$ to obtain:
\[
\Pr(X \ge c) \le \frac{\sigma^2}{\sigma^2 + (c-\mu)^2}.
\]
By the definition of the median, $\Pr(X \ge c) \ge \frac{1}{2}$. Combining these yields:
\[
\frac{1}{2} \le \frac{\sigma^2}{\sigma^2 + (c-\mu)^2} \implies \sigma^2 + (c-\mu)^2 \le 2\sigma^2 \implies (c-\mu)^2 \le \sigma^2.
\]
Taking the square root gives $|c - \mu| \le \sigma$. A symmetric argument using $\Pr(X \le c)$ applies when $c < \mu$. Therefore, finding a point close to the true median naturally provides a valid anchor point $c_k$ that is guaranteed to be at most $O(\sigma)$ away from the true mean.

\subsection{Noisy binary search}
To find this median, assume we know a bounded range $[-\lambda, \lambda]$ containing the median. We discretize this interval into a uniform grid of $n$ points $x_1 < x_2 < \dots < x_n$. Querying a threshold $x_i$ yields a 1-bit response $Y = \1\{X \le x_i\}$, which is a Bernoulli random variable with success probability $p_i = \Pr(X \le x_i)$. Because the cumulative distribution function is monotonic, the probabilities $p_i$ are non-decreasing.

Our goal is to find a threshold $x_{i^*}$ where the probability $p_{i^*}$ crosses $\tau = 1/2$. This exactly matches the \textsc{MonotonicNBS}($\tau, \eps'$) framework \cite{noisy_binary_search}, where the algorithm must find an $\eps'$-good interval despite coins (thresholds) near the target being indistinguishable from it. As established by \cite{noisy_binary_search}, for the median target $\tau = 1/2$, tolerance $\eps' \in (0, 1/4)$, and failure probability $\gamma \in (0, 1/2)$, the \textsc{MedianEst} algorithm successfully identifies a valid index $\hat{\imath}$ such that the interval $[\Pr(X \le x_{\hat{\imath}}), \Pr(X \le x_{\hat{\imath}+1})]$ intersects the target window $(1/2 - \eps', 1/2 + \eps')$ with probability at least $1-\gamma$.

To achieve this, the total number of 1-bit queries used by the algorithm is bounded by the sample complexity:
\[
M(n, \eps', \gamma) = O\left( \frac{1}{(\eps')^2} \left( \log n + \log^{2/3}n \log^{1/3}\frac{1}{\gamma} + \log\frac{1}{\gamma} \right) \right).
\]

\subsection{The \textsc{MedianEst} algorithm}
The \textsc{MedianEst} algorithm is shown in Algorithm~\ref{alg:medianest}, which adapts \textsc{BayesianScreeningSearch} to our threshold grid. Because the iteration bound in \cite{noisy_binary_search} includes explicit constant factors that complicate the expression, we omit the exact loop count in Algorithm~\ref{alg:medianest} to focus on the asymptotic scaling laws.

\begin{algorithm}[t]
\caption{\textsc{MedianEst}: 1-Bit Anchor Estimation}
\label{alg:medianest}
\begin{algorithmic}[1]
\Require Uniformly discretized grid $\mathbf{x} = (x_1, \dots, x_n)$ with step size $\Delta x$, tolerance $\eps' \in (0, 1/4)$, failure probability $\gamma \in (0, 1/2)$.
\State Initialize uniform prior: $w_1(i) = \frac{1}{n-1}$ for all $i \in \{1, \dots, n-1\}$.
\State Set sample budget $M = M(n, \eps', \gamma)$, and subsampling rate $r = \Theta(1/\log n)$.
\State Initialize queried sequence list $L = \emptyset$.
\For{$t = 1, \dots, M$}
    \State Compute the median index of the current posterior:
    \[
    j_t = \min \left\{ j \in \{1, \dots, n-1\} : \sum_{i=1}^j w_t(i) \ge \frac{1}{2} \right\}.
    \]
    \State Append $j_t$ to $L$.
    \State Query the threshold $x_{j_t}$ and observe 1-bit feedback $Y_t = \1\{X \le x_{j_t}\}$.
    \State Compute the unnormalized posterior $v_t$ by scaling weights based on $Y_t$:
    \[
    v_t(i) = \begin{cases} 
    w_t(i) \big(1 + 2\eps'(2Y_t - 1)\big) & \text{if } i \le j_t \\
    w_t(i) \big(1 - 2\eps'(2Y_t - 1)\big) & \text{if } i > j_t 
    \end{cases} \quad \forall i \in \{1, \dots, n-1\}.
    \]
    \State Normalize the posterior: 
    \[
    w_{t+1}(i) = \frac{v_t(i)}{\sum_{k=1}^{n-1} v_t(k)} \quad \forall i \in \{1, \dots, n-1\}.
    \]
\EndFor
\State Extract candidate set $\mathcal{C} = \left\{ L_{\lceil k \cdot rM \rceil} : k = 1, 2, \dots, \lfloor 1/r \rfloor \right\}$ from the history $L$.
\State Set verification budget $N = \Theta\left(\frac{1}{(\eps')^2} \log \frac{|\mathcal{C}|}{\gamma}\right)$.
\For{each candidate index $c \in \mathcal{C}$}
    \State Draw $N$ fresh samples $X^{(1)}, \dots, X^{(N)}$ to compute the empirical estimate:
    \[
    \widehat{p}_c = \frac{1}{N}\sum_{s=1}^N \1\{X^{(s)} \le x_c\}.
    \]
    \If{$\left| \widehat{p}_c - \frac{1}{2} \right| \le \frac{3\eps'}{4}$}
        \State \Return $\widehat{c} = \frac{x_c + x_{c+1}}{2}$ \Comment{Return the midpoint of the verified interval}
    \EndIf
\EndFor
\end{algorithmic}
\end{algorithm}

\subsection{Parameter selection and estimate validity}

\paragraph{Choosing the Tolerance ($\eps'$).} 
Because the sample complexity $M(n, \eps', \gamma)$ scales aggressively as $O(1/(\eps')^2)$, we cannot afford to let $\eps'$ shrink to $0$ alongside the main BAI target accuracy $\eps$. Instead, we fix $\eps'$ as a moderate constant by setting $\eps' = 0.05$. This choice avoids an exploding sample cost while maintaining a strictly bounded localization window.

\paragraph{Choosing the Grid Density ($n$).} 
Discretize the bounded domain $[-\lambda, \lambda]$ using a uniform step size $\Delta x = \kappa \cdot \sigma$ for a small constant $\kappa > 0$. This results in a grid of size $n = \Theta(\lambda/\sigma)$. Crucially, because $n$ appears inside a logarithm in the sample complexity $M(n, \eps', \gamma)$, this domain discretization contributes only an additive logarithmic term to the sample budget. The following lemma transforms the probability guarantees of \textsc{MedianEst} into rigorous bounds of the distance to the true mean.

\begin{lemma}[Estimate Validity via \textsc{MedianEst}]
\label{lem:anchor-validity}
Let $X$ be a random variable with mean $\mu \in [-\lambda, \lambda]$ and variance proxy $\sigma^2$. Let the continuous domain $[-\lambda, \lambda]$ be discretized into a uniform grid with step size $\Delta x = \kappa \cdot \sigma$. If the \textup{\textsc{MedianEst}} routine successfully identifies an $\eps'$-good interval $[x_{\hat{\imath}}, x_{\hat{\imath}+1}]$ for the target $\tau = 1/2$, then the returned midpoint $\widehat{c} = (x_{\hat{\imath}} + x_{\hat{\imath}+1})/2$ satisfies:
\[
|\widehat{c} - \mu| \le \left( \sqrt{\frac{1/2 + \eps'}{1/2 - \eps'}} + \frac{\kappa}{2} \right) \sigma.
\]
\end{lemma}

\begin{proof}
Because the returned interval is $\eps'$-good, the probability range $[\Pr(X \le x_{\hat{\imath}}), \Pr(X \le x_{\hat{\imath}+1})]$ must intersect the target window $(1/2 - \eps', 1/2 + \eps')$. Since cumulative probabilities are non-decreasing, this intersection gives us two conditions:
\[
\Pr(X \le x_{\hat{\imath}}) < \frac{1}{2} + \eps' \quad \text{and} \quad \Pr(X \le x_{\hat{\imath}+1}) > \frac{1}{2} - \eps'
\]

We can use these to bound the tails. The upper tail satisfies $\Pr(X > x_{\hat{\imath}}) = 1 - \Pr(X \le x_{\hat{\imath}}) > 1/2 - \eps'$, while the lower tail directly satisfies $\Pr(X \le x_{\hat{\imath}+1}) > 1/2 - \eps'$. We evaluate the distance from $\mu$ to the grid step by considering three cases:

\textbf{Case 1 ($\mu < x_{\hat{\imath}}$):} Applying Cantelli's inequality gives $\Pr(X \ge x_{\hat{\imath}}) \le \frac{\sigma^2}{\sigma^2 + (x_{\hat{\imath}} - \mu)^2}$. Combining with our known tail bound, we get:
\[
\frac{1}{2} - \eps' < \Pr(X > x_{\hat{\imath}}) \le \Pr(X \ge x_{\hat{\imath}}) \le \frac{\sigma^2}{\sigma^2 + (x_{\hat{\imath}} - \mu)^2}
\]
Solving for the distance yields $x_{\hat{\imath}} - \mu < \sigma \sqrt{\frac{1/2 + \eps'}{1/2 - \eps'}}$. Because the returned estimate $\widehat{c}$ is the midpoint of the grid interval, its distance to the true mean is 
\[
|\widehat{c} - \mu| = \frac{\Delta x}{2} + (x_{\hat{\imath}} - \mu) < \frac{\kappa\sigma}{2} + \sigma \sqrt{\frac{1/2 + \eps'}{1/2 - \eps'}}.
\]

\textbf{Case 2 ($\mu > x_{\hat{\imath}+1}$):} Applying a symmetric Cantelli bound gives $\Pr(X \le x_{\hat{\imath}+1}) \le \frac{\sigma^2}{\sigma^2 + (\mu - x_{\hat{\imath}+1})^2}$. Combining with our lower tail condition, we have:
\[
\frac{1}{2} - \eps' < \Pr(X \le x_{\hat{\imath}+1}) \le \frac{\sigma^2}{\sigma^2 + (\mu - x_{\hat{\imath}+1})^2}
\]
Rearranging the inequality yields $\mu - x_{\hat{\imath}+1} < \sigma \sqrt{\frac{1/2 + \eps'}{1/2 - \eps'}}$. The distance from the midpoint is 
\[
|\widehat{c} - \mu| = (\mu - x_{\hat{\imath}+1}) + \frac{\Delta x}{2} < \sigma \sqrt{\frac{1/2 + \eps'}{1/2 - \eps'}} + \frac{\kappa\sigma}{2}.
\]

\textbf{Case 3 ($x_{\hat{\imath}} \le \mu \le x_{\hat{\imath}+1}$):} The true mean is inside the identified interval. The absolute maximum distance from the midpoint $\widehat{c}$ to any point inside this grid step is trivially bounded by half the step size, $\frac{\Delta x}{2} = \frac{\kappa\sigma}{2}$. This is strictly smaller than the bounds derived in Cases 1 and 2.

We conclude that in all possible scenarios, the initial estimate satisfies:
$$|\widehat{c} - \mu| \le \left( \sqrt{\frac{1/2 + \eps'}{1/2 - \eps'}} + \frac{\kappa}{2} \right) \sigma$$
\end{proof}

With $\eps' = 0.05$ and $\kappa = 0.2$, the guaranteed bound becomes $|\widehat{c} - \mu| \le 1.205\sigma$. We can, in principle, choose $\eps'$ and $\kappa$ to bring this bound arbitrarily close to $\sigma$, but doing so will negatively affect the sample complexity. Because the sample requirement of the \textsc{MedianEst} routine scales as $O(1/(\eps')^2)$ and logarithmically with the grid size $n \propto 1/\kappa$, shrinking these parameters heavily inflates the initialization cost. Ultimately, pushing for this precision is unnecessary; a difference of $O(\sigma)$ in the anchor's distance from the mean will only affect the constant factors in the estimator's time-uniform confidence sequence (Theorem~\ref{thm:confidence-bound}) and the final best-arm identification sample complexity bounds (Theorem~\ref{thm:1-bit-bai-upper} and Theorem~\ref{thm:phased-bai-upper}).

\begin{lemma}[Total Initialization Complexity]
\label{lem:init-complexity}
By applying a union bound over the $K$ arms and allocating a failure budget of $\gamma = \delta/(2K)$ to each \textup{\textsc{MedianEst}} instance, the total initialization burn-in phase is bounded by:
\[
T_{\text{init}} = O\left( K \left( \log\frac{\lambda}{\sigma} + \log\frac{K}{\delta} \right) \right).
\]
\end{lemma}

\begin{proof}
For each of the $K$ arms, we run the \textsc{MedianEst} routine. With our fixed $\eps'$ and the grid size $n = \Theta(\lambda/\sigma)$, the sample complexity per arm is bounded by $M(n, \eps', \delta/(2K))$. The additive terms in the query complexity are:
\[
O\left( \log n + \log^{2/3}n \log^{1/3}\frac{1}{\gamma} + \log\frac{1}{\gamma} \right).
\]
By Young's inequality, the cross term is bounded by the linear combination of the individual logarithmic terms:
\[
\log^{2/3}n \log^{1/3}\frac{1}{\gamma} \le \frac{2}{3}\log n + \frac{1}{3}\log\frac{1}{\gamma} = O\left(\log n + \log\frac{1}{\gamma}\right).
\]
Thus, the sample complexity per arm simplifies to $O\left(\log n + \log(1/\gamma)\right)$. Substituting $n = \Theta(\lambda/\sigma)$ and $\gamma = \delta/(2K)$ yields $O\left( \log\frac{\lambda}{\sigma} + \log\frac{K}{\delta} \right)$ samples per arm. Multiplying this per-arm cost by the $K$ parallel instances gives the total bound $T_{\text{\upshape init}}$. Because $T_{\text{\upshape init}}$ scales strictly with the number of arms $K$ and is entirely independent of the arm gaps $\Delta_k$, it operates as a purely additive initialization cost.
\end{proof}
\newpage
\section{Proofs for Section~\ref{sec:estimator} (A 1-bit mean estimator)}
\label{app:estimator}
Throughout all the proofs, we write \(a \wedge b := \min\{a,b\}\).
\subsection{Proof of Lemma \ref{lem:tail-integral} (Tail-integral identity)}
\label{app:proof-tail-integral}
\begin{proof}
Fix \(w\in\R\) and \(L>0\).

If \(w\ge 0\), then \(\1\{w\le -t\}=0\) for every \(t\ge 0\), and hence
\[
\int_0^L \Bigl(\1\{w\ge t\}-\1\{w\le -t\}\Bigr)\,dt
=
\int_0^L \1\{w\ge t\}\,dt
=
\int_0^{w\wedge L} 1\,dt
=
w\wedge L
=
[w]_L.
\]

If \(w<0\), then \(\1\{w\ge t\}=0\) for every \(t\ge 0\), and therefore
\[
\int_0^L \Bigl(\1\{w\ge t\}-\1\{w\le -t\}\Bigr)\,dt
=
-\int_0^L \1\{w\le -t\}\,dt
=
-\int_0^{(-w)\wedge L} 1\,dt
=
[w]_L.
\]

This proves the pointwise identity. Taking expectations gives
\[
\E[[W]_L]
=
\int_0^L \Bigl(\Pr(W\ge t)-\Pr(W\le -t)\Bigr)\,dt,
\]
since the integrand is bounded in absolute value by \(1\) on a finite interval.
\end{proof}

\subsection{Proof of Lemma~\ref{lem:estimator-properties}}
\label{app:proof-estimator-properties}
We fix \(L\ge \sigma\) in the following proofs.

\begin{proof}[Proof of (i)]
Condition on \(W=w\). Since \(S\) is uniform on \(\{+1,-1\}\), independent of \(T\), and
\[
Y=\1\{S(X-c)\ge T\}=\1\{SW\ge T\},
\]
we have
\begin{align*}
\E[Z\mid W=w]
&=
\frac12 \int_0^L \frac{2}{g_L(t)}\, g_L(t)\,\1\{w\ge t\}\,dt
-
\frac12 \int_0^L \frac{2}{g_L(t)}\, g_L(t)\,\1\{w\le -t\}\,dt \\
&=
\int_0^L \Bigl(\1\{w\ge t\}-\1\{w\le -t\}\Bigr)\,dt.
\end{align*}
By Lemma \ref{lem:tail-integral}, this equals \([w]_L\). Therefore
\[
\E[Z\mid W]=[W]_L
\qquad\text{almost surely.}
\]
and taking expectations yields
\[
\E[Z]=\E[[W]_L].
\]
\end{proof}

\begin{proof}[Proof of (ii)]
Again condition on \(W=w\), and set
\[
u:=|w|\wedge L.
\]
Since \(Z=\dfrac{2S}{g_L(T)}Y\), we have \(Z^2=\dfrac{4}{g_L(T)^2}Y\). Hence
\begin{align*}
\E[Z^2\mid W=w]
&=
\frac12 \int_0^L \frac{4}{g_L(t)^2}\,g_L(t)\,\1\{w\ge t\}\,dt
+
\frac12 \int_0^L \frac{4}{g_L(t)^2}\,g_L(t)\,\1\{w\le -t\}\,dt \\
&=
2\int_0^L \frac{1}{g_L(t)}
\Bigl(\1\{w\ge t\}+\1\{w\le -t\}\Bigr)\,dt \\
&=
2\int_0^{u} \frac{dt}{g_L(t)}.
\end{align*}
Using
\[
\frac{1}{g_L(t)}=\alpha_L\max\{\sigma,t\},
\]
we obtain
\[
\E[Z^2\mid W=w]
=
2\alpha_L\int_0^{u}\max\{\sigma,t\}\,dt.
\]

If \(u\le \sigma\), then
\[
\E[Z^2\mid W=w]
=
2\alpha_L \sigma u
\le
\alpha_L(\sigma^2+u^2),
\]
by the inequality \(2ab\le a^2+b^2\).

If \(u\ge \sigma\), then
\begin{align*}
\E[Z^2\mid W=w]
&=
2\alpha_L\left(\int_0^\sigma \sigma\,dt + \int_\sigma^u t\,dt\right) \\
&=
2\alpha_L\left(\sigma^2 + \frac{u^2-\sigma^2}{2}\right)
=
\alpha_L(\sigma^2+u^2).
\end{align*}

Thus, in all cases,
\[
\E[Z^2\mid W=w]
\le
\alpha_L\bigl(\sigma^2 + (|w|\wedge L)^2\bigr)
=
\alpha_L\bigl(\sigma^2 + ([|w|]_L)^2\bigr).
\]
Taking expectations gives
\[
\E[Z^2]
\le
\alpha_L\Bigl(\sigma^2 + \E\bigl[([|W|]_L)^2\bigr]\Bigr).
\]
Since \(([|W|]_L)^2\le W^2\) and \(\E[W^2]\le 2\sigma^2\), we conclude that
\[
\E[Z^2]\le \alpha_L(\sigma^2+2\sigma^2)=3\alpha_L\sigma^2.
\]
Therefore
\[
\Var(Z)\le \E[Z^2]\le 3\alpha_L\sigma^2.
\]
\end{proof}

\begin{proof}[Proof of (iii)]
Since \(Y\in\{0,1\}\),
\[
|Z|
=
\left|\frac{2S}{g_L(T)}Y\right|
\le
\frac{2}{g_L(T)}
=
2\alpha_L\max\{\sigma,T\}
\le
2\alpha_L L,
\]
because \(T\in[0,L]\).
\end{proof}

\subsection{Proof of Lemma~\ref{lem:clipping-bias} (Clipping bias)}
\label{app:proof-clipping-bias}
\begin{proof}
For every real \(w\),
\[
|w-[w]_L| = (|w|-L)_+.
\]
Also, for every \(x\ge 0\),
\[
(x-L)_+ \le \frac{x^2}{4L}.
\]
Indeed, the inequality is trivial when \(x\le L\). If \(x>L\), then
\[
x-L \le \frac{x^2}{4L}
\iff
0\le (x-2L)^2.
\]
Applying this with \(x=|W|\), we get
\[
|W-[W]_L|
=
(|W|-L)_+
\le
\frac{W^2}{4L}.
\]
Therefore
\[
\bigl|\E[W]-\E[[W]_L]\bigr|
=
\bigl|\E[W-[W]_L]\bigr|
\le
\E\bigl[\,|W-[W]_L|\,\bigr]
\le
\frac{\E[W^2]}{4L}
\le
\frac{\sigma^2}{2L},
\]
since \(\E[W^2]\le 2\sigma^2\).
\end{proof}

\subsection{Proof of Theorem~\ref{thm:confidence-bound} (Confidence sequence for the 1-bit estimator)}
\label{app:proof-confidence-sequence}

We split the proof into four lemmas. The first lemma records the mean, variance, and range properties of the centered 1-bit variables. The second lemma turns these properties into an exponential supermartingale. The third lemma derives an epoch-wise Bernstein boundary by combining Ville's inequality with an explicit optimization over the exponential parameter. The fourth lemma stitches the epoch-wise bounds into a genuine time-uniform confidence sequence for the clipped target. The theorem then follows by adding the clipping-bias term.

\begin{lemma}[Centered process properties]
\label{lem:centered-process-properties}
Fix \(L \ge \sigma\), and define
\[
\theta_L := \E[[W]_L],
\qquad
X_i := Z_i - \theta_L,
\qquad
S_n := \sum_{i=1}^n X_i.
\]
Then \(X_1,X_2,\dots\) are i.i.d. and mean-zero, and satisfy
\[
\E[X_i]=0,
\qquad
\Var(X_i)\le v_L,
\qquad
|X_i|\le b_L
\quad\text{a.s.},
\]
where
\[
v_L := 3\alpha_L \sigma^2,
\qquad
b_L := 3\alpha_L L.
\]
\end{lemma}

\begin{proof}
By Lemma~\ref{lem:estimator-properties}(i),
\[
\E[Z_i]=\E[[W]_L]=\theta_L,
\]
and therefore
\[
\E[X_i]=\E[Z_i-\theta_L]=0.
\]
Since subtracting a constant does not change the variance, Lemma~\ref{lem:estimator-properties}(ii)
gives
\[
\Var(X_i)=\Var(Z_i)\le 3\alpha_L\sigma^2=v_L.
\]

It remains to prove the range bound. By Lemma~\ref{lem:estimator-properties}(iii),
\[
|Z_i|\le 2\alpha_L L
\qquad\text{a.s.}
\]
Also, since \(|[W]_L|\le L\),
\[
|\theta_L|
=
|\E[[W]_L]|
\le
\E |[W]_L|
\le
L.
\]
Hence
\[
|X_i|
=
|Z_i-\theta_L|
\le
|Z_i|+|\theta_L|
\le
2\alpha_L L + L
\le
3\alpha_L L
=
b_L,
\]
where we used \(\alpha_L \ge 1\). This proves the claim.
\end{proof}

\begin{lemma}[Bernstein supermartingale]
\label{lem:bernstein-supermartingale}
For every \(\lambda \in [0,3/b_L)\), define
\[
\psi_L(\lambda)
:=
\frac{v_L \lambda^2}{2(1-b_L\lambda/3)},
\qquad
M_n(\lambda)
:=
\exp\!\bigl(\lambda S_n - n\psi_L(\lambda)\bigr).
\]
Then \(\{M_n(\lambda)\}_{n\ge 0}\) is a nonnegative supermartingale.
\end{lemma}

\begin{proof}
Since \(\E[X_i]=0\), \(\Var(X_i)\le v_L\), and \(|X_i|\le b_L\) by
Lemma~\ref{lem:centered-process-properties}, the standard Bernstein mgf bound yields
\[
\E[e^{\lambda X_i}]
\le
\exp\!\left(
\frac{v_L\lambda^2}{2(1-b_L\lambda/3)}
\right)
=
e^{\psi_L(\lambda)}
\qquad
\text{for all } \lambda\in[0,3/b_L).
\]
Therefore, conditioning on \(X_1,\dots,X_{n-1}\),
\begin{align*}
\E[M_n(\lambda)\mid X_1,\dots,X_{n-1}]
&=
\E\!\left[
\exp\!\bigl(\lambda S_{n-1}+\lambda X_n-n\psi_L(\lambda)\bigr)
\mid X_1,\dots,X_{n-1}
\right] \\
&=
M_{n-1}(\lambda)\, e^{-\psi_L(\lambda)} \E[e^{\lambda X_n}] \\
&\le
M_{n-1}(\lambda).
\end{align*}
Thus \(\{M_n(\lambda)\}_{n\ge 0}\) is a supermartingale. Since \(M_n(\lambda)\) is an
exponential of a real-valued random variable, it is strictly positive almost surely, hence
nonnegative.
\end{proof}

\begin{lemma}[Epoch-wise stitched Bernstein bound]
\label{lem:epoch-wise-bernstein}
For each integer \(j\ge 0\), let
\[
N_j := 2^{j+1},
\qquad
x_j(\gamma)
:=
\log\!\left(\frac{\pi^2 (j+1)^2}{3\gamma}\right).
\]
Then
\[
\Pr\!\left(
\exists n\in\{2^j,\dots,2^{j+1}-1\} :
S_n
\ge
2\sqrt{v_L n\, x_j(\gamma)}
+
\frac{b_L x_j(\gamma)}{3}
\right)
\le
e^{-x_j(\gamma)}.
\]
\end{lemma}

\begin{proof}
Fix \(j\ge 0\), and write
\[
N:=N_j=2^{j+1},
\qquad
x:=x_j(\gamma).
\]
By Ville's inequality applied to the nonnegative supermartingale
\(\{M_n(\lambda)\}_{n\ge 0}\) from Lemma~\ref{lem:bernstein-supermartingale},
for every \(\lambda\in[0,3/b_L)\),
\[
\Pr\!\left(\exists n\le N : M_n(\lambda)\ge e^x\right)\le e^{-x}.
\]
Using the definition of \(M_n(\lambda)\), this is equivalent to
\[
\Pr\!\left(
\exists n\le N :
\lambda S_n - n\psi_L(\lambda)\ge x
\right)
\le
e^{-x}.
\]

Now define the event
\[
A_\lambda
:=
\left\{
\exists n\le N :
S_n \ge \frac{x+N\psi_L(\lambda)}{\lambda}
\right\}.
\]
If \(A_\lambda\) occurs, then for some \(n\le N\),
\[
\lambda S_n - n\psi_L(\lambda)
\ge
\lambda\cdot \frac{x+N\psi_L(\lambda)}{\lambda} - N\psi_L(\lambda)
=
x.
\]
Hence
\[
A_\lambda
\subseteq
\left\{
\exists n\le N :
\lambda S_n - n\psi_L(\lambda)\ge x
\right\},
\]
and therefore
\[
\Pr(A_\lambda)\le e^{-x}.
\]
Equivalently,
\[
\Pr\!\left(
\exists n\le N :
S_n \ge \frac{x+N\psi_L(\lambda)}{\lambda}
\right)
\le
e^{-x}
\qquad
\text{for all } \lambda\in[0,3/b_L).
\]

We now optimize the boundary over \(\lambda\). Set
\[
c_L := \frac{b_L}{3}.
\]
Then
\[
\frac{x+N\psi_L(\lambda)}{\lambda}
=
\frac{x}{\lambda}
+
\frac{Nv_L\lambda}{2(1-c_L\lambda)}
=: g(\lambda),
\qquad
0<\lambda<\frac{1}{c_L}.
\]
Differentiating,
\[
g'(\lambda)
=
-\frac{x}{\lambda^2}
+
\frac{Nv_L}{2(1-c_L\lambda)^2}.
\]
Setting \(g'(\lambda)=0\) gives
\[
\frac{x}{\lambda^2}
=
\frac{Nv_L}{2(1-c_L\lambda)^2}.
\]
Taking square roots and rearranging,
\[
\frac{\sqrt{x}}{\lambda}
=
\frac{\sqrt{Nv_L/2}}{1-c_L\lambda}
\quad\Longleftrightarrow\quad
\frac{1}{\lambda}
=
\sqrt{\frac{Nv_L}{2x}}+c_L.
\]
Thus the optimizer is
\[
\lambda_*
=
\frac{1}{\sqrt{Nv_L/(2x)}+c_L}.
\]
Substituting \(\lambda_*\) into \(g\), we obtain
\begin{align*}
g(\lambda_*)
&=
x\!\left(\sqrt{\frac{Nv_L}{2x}}+c_L\right)
+
\frac{Nv_L}{2\sqrt{Nv_L/(2x)}} \\
&=
\sqrt{\frac{Nv_Lx}{2}}
+
\sqrt{\frac{Nv_Lx}{2}}
+
c_Lx \\
&=
\sqrt{2v_LNx}
+
\frac{b_Lx}{3}.
\end{align*}
Therefore
\[
\Pr\!\left(
\exists n\le N :
S_n
\ge
\sqrt{2v_LNx}
+
\frac{b_Lx}{3}
\right)
\le
e^{-x}.
\]
Recalling that \(N=2^{j+1}\) and \(x=x_j(\gamma)\), we get
\[
\Pr\!\left(
\exists n\le 2^{j+1} :
S_n
\ge
\sqrt{2v_L2^{j+1}x_j(\gamma)}
+
\frac{b_Lx_j(\gamma)}{3}
\right)
\le
e^{-x_j(\gamma)}.
\]

Finally, if \(n\in\{2^j,\dots,2^{j+1}-1\}\), then \(n\ge 2^j\), and so
\[
\sqrt{2v_L2^{j+1}x_j(\gamma)}
=
2\sqrt{v_L2^j x_j(\gamma)}
\le
2\sqrt{v_Ln\,x_j(\gamma)}.
\]
Hence
\[
\Pr\!\left(
\exists n\in\{2^j,\dots,2^{j+1}-1\} :
S_n
\ge
2\sqrt{v_Ln\,x_j(\gamma)}
+
\frac{b_Lx_j(\gamma)}{3}
\right)
\le
e^{-x_j(\gamma)},
\]
as claimed.
\end{proof}

\begin{lemma}[Confidence sequence for the clipped target]
\label{lem:cs-clipped-target}
Define
\[
x_n(\gamma)
:=
\log\!\left(
\frac{\pi^2(\lfloor \log_2 n\rfloor+1)^2}{3\gamma}
\right).
\]
Then with probability at least \(1-\gamma\), simultaneously for all \(n\ge 1\),
\[
|\bar Z_n-\theta_L|
\le
2\sqrt{\frac{v_L x_n(\gamma)}{n}}
+
\frac{b_L x_n(\gamma)}{3n}.
\]
\end{lemma}

\begin{proof}
For each \(j\ge 0\), apply Lemma~\ref{lem:epoch-wise-bernstein} to both \(S_n\) and \(-S_n\).
This yields
\[
\Pr\!\left(
\exists n\in\{2^j,\dots,2^{j+1}-1\} :
|S_n|
\ge
2\sqrt{v_Ln\,x_j(\gamma)}
+
\frac{b_Lx_j(\gamma)}{3}
\right)
\le
2e^{-x_j(\gamma)}.
\]
Now
\[
e^{-x_j(\gamma)}
=
\frac{3\gamma}{\pi^2 (j+1)^2},
\]
and therefore
\[
\sum_{j=0}^\infty 2e^{-x_j(\gamma)}
=
\frac{6\gamma}{\pi^2}\sum_{j=0}^\infty \frac{1}{(j+1)^2}
=
\frac{6\gamma}{\pi^2}\cdot \frac{\pi^2}{6}
=
\gamma.
\]
Hence, by a union bound over all dyadic epochs, with probability at least \(1-\gamma\), for every
\(j\ge 0\) and every \(n\in\{2^j,\dots,2^{j+1}-1\}\),
\[
|S_n|
\le
2\sqrt{v_Ln\,x_j(\gamma)}
+
\frac{b_Lx_j(\gamma)}{3}.
\]
Since \(j=\lfloor \log_2 n\rfloor\) on this epoch, this can be written as
\[
|S_n|
\le
2\sqrt{v_Ln\,x_n(\gamma)}
+
\frac{b_Lx_n(\gamma)}{3}
\qquad
\text{for all } n\ge 1.
\]
Dividing by \(n\) gives
\[
|\bar Z_n-\theta_L|
=
\frac{|S_n|}{n}
\le
2\sqrt{\frac{v_Lx_n(\gamma)}{n}}
+
\frac{b_Lx_n(\gamma)}{3n}
\qquad
\text{for all } n\ge 1,
\]
which proves the claim.
\end{proof}

\begin{proof}[Proof of Theorem~\ref{thm:confidence-bound}]
By Lemma~\ref{lem:clipping-bias},
\[
|\theta_L-\E[W]|
\le
\frac{\sigma^2}{2L}.
\]
Also,
\[
\hat\mu_n = c+\bar Z_n,
\qquad
\mu = c+\E[W].
\]
Therefore
\[
|\hat\mu_n-\mu|
=
|\bar Z_n-\E[W]|
\le
|\bar Z_n-\theta_L|
+
|\theta_L-\E[W]|.
\]
Combining this with Lemma~\ref{lem:cs-clipped-target}, we conclude that with probability at least
\(1-\gamma\), simultaneously for all \(n\ge 1\),
\[
|\hat\mu_n-\mu|
\le
\frac{\sigma^2}{2L}
+
2\sqrt{\frac{v_Lx_n(\gamma)}{n}}
+
\frac{b_Lx_n(\gamma)}{3n}.
\]
Finally, substituting
\[
v_L = 3\alpha_L\sigma^2,
\qquad
b_L = 3\alpha_LL,
\]
gives
\[
|\hat\mu_n-\mu|
\le
\frac{\sigma^2}{2L}
+
2\sqrt{\frac{3\alpha_L\sigma^2x_n(\gamma)}{n}}
+
\frac{\alpha_L L\, x_n(\gamma)}{n}
\qquad
\text{for all } n\ge 1.
\]
This is exactly the claimed confidence sequence.
\end{proof}

\subsection{Sample-complexity inversion}
\label{app:proof-estimator-complexity}
\begin{proof}[Proof of Corollary~\ref{cor:estimator-sample-complexity}]
Let $\beta > 0$ be the desired accuracy. To balance the clipping bias and the variance, we choose the clipping level $L = \frac{2\sigma^2}{\beta}$. By Lemma~\ref{lem:clipping-bias}, this ensures the clipping bias is bounded by $\frac{\sigma^2}{2L} = \frac{\beta}{4}$.

From Theorem~\ref{thm:confidence-bound}, with probability at least $1-\gamma$, simultaneously for all $n \ge 1$:
\[
|\hat{\mu}_n - \mu| \le \frac{\beta}{4} + 2\sqrt{\frac{3\alpha_L\sigma^2 x_n(\gamma)}{n}} + \frac{\alpha_L L\, x_n(\gamma)}{n}
\]
where $\alpha_L = 1 + \log\dfrac{2\sigma}{\beta}$ and $x_n(\gamma) = \log\left(\dfrac{\pi^2(1+\lfloor \log_2 n\rfloor)^2}{3\gamma}\right)$.

To guarantee $|\hat{\mu}_n - \mu| \le \beta$, we enforce that the concentration term is bounded by $\frac{\beta}{2}$ and the Bernstein range term by $\frac{\beta}{4}$. Solving these individually for $n$:
\begin{align}
2\sqrt{\frac{3\alpha_L\sigma^2 x_n(\gamma)}{n}} \le \frac{\beta}{2} &\implies n \ge \frac{48 \alpha_L \sigma^2 x_n(\gamma)}{\beta^2} \label{eq:concentration-condition} \\
\frac{\alpha_L L\, x_n(\gamma)}{n} \le \frac{\beta}{4} &\implies n \ge \frac{8 \alpha_L \sigma^2 x_n(\gamma)}{\beta^2}
\end{align}

Condition (\ref{eq:concentration-condition}) is strictly dominant. Let us define
\[
H := \frac{48 \alpha_L \sigma^2}{\beta^2},
\]
so that our goal is to find an integer $n$ satisfying $n \ge H x_n(\gamma)$. Choose
\[
n := \lceil C H \eta \rceil,
\qquad
\eta := \log\frac{1}{\gamma} + \log\log(\max\{H, e\}),
\]
where $C>0$ is a sufficiently large universal constant. Since $\eta \ge 1$, we have
\[
\log H \le H, \qquad \log \eta \le \eta.
\]
Moreover,
\[
\log n \le \log(C+1) + \log H + \log \eta.
\]
Taking another logarithm and using $\log(a+b+c) \le \log 3 + \max\{\log a, \log b, \log c\}$, we get
\[
\log\log n \le c_1 + \max\{\log\log H, \log \eta\} \le c_2 + \eta
\]
for universal constants $c_1, c_2 > 0$, because $\log \eta \le \eta$. Since $1+\lfloor \log_2 n\rfloor \le c_3 \log n$ for some universal constant $c_3>0$, we have
\begin{align*}
x_n(\gamma)
&=
\log\!\left(\frac{\pi^2\bigl(1+\lfloor \log_2 n\rfloor\bigr)^2}{3\gamma}\right) \\
&\le
\log\!\left(\frac{c_4 (\log n)^2}{\gamma}\right) \\
&=
\log c_4 + \log\!\left(\frac{1}{\gamma}\right) + 2\log\log n \\
&\le c_5\left[\log\!\left(\frac{1}{\gamma}\right)+\log\log (\max\{H, e\})\right]
= c_5\eta,
\end{align*}
for universal constants $c_4, c_5 > 0$. Choosing $C \ge c_5$, we obtain
\[
n \ge C H \eta \ge H x_n(\gamma).
\]
Consequently, condition (\ref{eq:concentration-condition}) is satisfied. Since $\alpha_L = 1 + \log\frac{2\sigma}{\beta}$, we have $H = \Theta\left(\frac{\sigma^2}{\beta^2}\log\frac{\sigma}{\beta}\right)$, and therefore
\[
n = O\left( \frac{\sigma^2}{\beta^2}
\log\frac{\sigma}{\beta} 
\left[ \log\frac{1}{\gamma} + \log\log\left(e + \frac{\sigma}{\beta}\right)\right] \right).
\]
This completes the proof.
\end{proof}
\newpage
\section{Proofs for Section~\ref{sec:algorithm} (1-bit BAI algorithm and upper bound)}
\label{app:algorithm-proofs}

\subsection{Fixed clipping: proof of Theorem~\ref{thm:1-bit-bai-upper}}
\label{app:proof-fixed-clipping}
This subsection proves the sample complexity upper bound of the fixed-clipping algorithm. We first establish a high-probability event on which all fixed-clipping confidence intervals are valid uniformly over time (Lemma~\ref{lem:fixed-good-event}). On this event, the interval stopping rule can only return an $\eps$-optimal arm (Lemma~\ref{lem:fixed-stopping-correct}). It remains to control the number of samples. The key structural step is Lemma~\ref{lem:fixed-pulled-arm-radius}: whenever the algorithm samples arm $k$, its stochastic radius must still be larger than a constant fraction of the effective gap. Solving this inequality gives an arm-wise pull bound, and summing these bounds over arms yields the stated sample complexity (Lemma~\ref{lem:fixed-counting}).

Throughout this subsection, fix
\[
L_\eps=\frac{2\sigma^2}{\eps},
\qquad
\alpha_\eps=1+\log\frac{2\sigma}{\eps},
\qquad
\gamma=\frac{\delta}{K}.
\]
For the fixed-clipping algorithm, write
\[
\beta_\eps(n,\gamma)
=
\frac{\eps}{4}
+2\sqrt{\frac{3\alpha_\eps\sigma^2x_n(\gamma)}{n}}
+\frac{2\alpha_\eps\sigma^2x_n(\gamma)}{n\eps},
\qquad
r_\eps(n,\gamma):=\beta_\eps(n,\gamma)-\frac{\eps}{4}.
\]
Let $Z_{k,1},Z_{k,2},\dots$ be i.i.d. copies of the estimator, each built from an independent reward sample and independent internal randomization of arm $k$.

At a generic round, let $N_k$ denote the number of 1-bit queries already made to arm $k$, and let
\[
\widehat\mu_k=c_k+\frac1{N_k}\sum_{i=1}^{N_k}Z_{k,i},
\qquad
\LCB_k=\widehat\mu_k-\beta_\eps(N_k,\gamma),
\qquad
\UCB_k=\widehat\mu_k+\beta_\eps(N_k,\gamma).
\]

\begin{lemma}[Fixed-clipping good event]
\label{lem:fixed-good-event}
With probability at least $1-\delta$,
\[
\mathcal E_{\rm fix}
:=
\bigcap_{k=1}^K\bigcap_{n\ge1}
\left\{\left|\widehat\mu_k(n)-\mu_k\right|\le \beta_\eps(n,\delta/K)\right\}
\]
holds throughout \textup{\textsc{FixedClip-1BitBAI}}.
\end{lemma}

\begin{proof}
For a fixed arm $k$, the samples used by its tail-integral estimator are i.i.d. when indexed by the local count $n$, regardless of the adaptive times at which arm $k$ is selected. Applying Theorem~\ref{thm:confidence-bound} with clipping level $L_\eps$ and failure probability $\delta/K$ gives
\[
\Pr\left(\forall n\ge1,
\left|\widehat\mu_k(n)-\mu_k\right|
\le \beta_\eps(n,\delta/K)\right)
\ge 1-\frac{\delta}{K}.
\]
Taking a union bound over $k\in[K]$ yields
\[
\Pr(\mathcal E_{\rm fix})
\ge 1-\sum_{k=1}^K\frac{\delta}{K}
=1-\delta.
\]
The time-uniformity over $n$ is what makes the event valid under the adaptive sampling rule.
\end{proof}

\begin{lemma}[Fixed-clipping stopping correctness]
\label{lem:fixed-stopping-correct}
On $\mathcal E_{\rm fix}$, if \textup{\textsc{FixedClip-1BitBAI}} stops and returns $\widehat k$, then
\[
\mu_{\widehat k}\ge \mu^\star-\eps.
\]
\end{lemma}

\begin{proof}
Let $h$ be the arm returned by the algorithm. The stopping rule gives
\[
B_h=\max_{\ell\ne h}\UCB_\ell-\LCB_h\le \eps.
\]
Equivalently, for every $\ell\ne h$,
\[
\UCB_\ell-\LCB_h\le \eps.
\]
On $\mathcal E_{\rm fix}$, all intervals are valid, so
\[
\LCB_j\le \mu_j\le \UCB_j
\qquad\text{for every }j\in\A.
\]
Therefore, for every $\ell\ne h$,
\[
\mu_\ell-
\mu_h
\le
\UCB_\ell-
\LCB_h
\le \eps.
\]
If $h=k^\star$, the claim is immediate. Otherwise, choosing $\ell=k^\star$ gives
\[
\mu^\star-
\mu_h
\le \eps,
\]
and hence $\mu_h\ge\mu^\star-\eps$.
\end{proof}

The next lemma is the key interval comparison: an arm can be pulled only while its own confidence radius is still large enough to explain the current ambiguity.

\begin{lemma}[Fixed-clipping structural comparison]
\label{lem:fixed-structural-comparison}
Assume $\mathcal E_{\rm fix}$ holds. Consider a non-stopping round, and write
\[
B_k=\max_{a\ne k}\UCB_a-\LCB_k,
\qquad
h\in\argmin_{k\in\A}B_k,
\qquad
\ell\in\argmax_{a\ne h}\UCB_a.
\]
If the algorithm samples $k\in\{h,\ell\}$, then
\[
B_h
\le
\min\{0,-\Delta_k+2\beta_k\}+2\beta_k,
\]
where $\beta_k:=\beta_\eps(N_k,\delta/K)$.
\end{lemma}

\begin{proof}
For brevity, write
\[
U_i:=\UCB_i,
\qquad
L_i:=\LCB_i,
\qquad
\beta_i:=\beta_\eps(N_i,\delta/K).
\]
Since \(U_i-L_i=2\beta_i\) and the intervals are valid on
\(\mathcal E_{\rm fix}\), we have
\begin{equation}
\label{eq:fixed-valid-bounds}
L_i\le \mu_i\le U_i,
\qquad
U_i\le \mu_i+2\beta_i,
\qquad
L_i\ge \mu_i-2\beta_i
\qquad\text{for all }i\in\A .
\end{equation}
By definition of \(\ell\),
\[
B_h=\max_{a\ne h}U_a-L_h=U_\ell-L_h .
\]

\emph{Case 1: the sampled arm is \(k=\ell\).}
The sampling rule gives \(\beta_\ell\ge\beta_h\). We first show that
\(L_\ell\le L_h\). Suppose, for contradiction, that \(L_\ell>L_h\). Then
\[
U_\ell=L_\ell+2\beta_\ell
>
L_h+2\beta_h=U_h.
\]
Since \(U_\ell=\max_{a\ne h}U_a\), and since \(U_\ell>U_h\), we have
\[
\max_{a\ne \ell}U_a\le U_\ell .
\]
Hence
\[
B_\ell
=
\max_{a\ne \ell}U_a-L_\ell
<
U_\ell-L_h
=
B_h,
\]
contradicting the definition of \(h\) as a minimizer of \(B_k\). Thus
\(L_\ell\le L_h\), and therefore
\begin{equation}
\label{eq:fixed-case-ell-two-beta}
B_h=U_\ell-L_h\le U_\ell-L_\ell=2\beta_\ell .
\end{equation}

We next prove
\[
B_h\le -\Delta_\ell+4\beta_\ell .
\]
If \(h=k^\star\), then \(\ell\ne k^\star\). Using
\eqref{eq:fixed-valid-bounds} and \(\beta_h\le\beta_\ell\),
\[
B_h
=
U_\ell-L_h
\le
(\mu_\ell+2\beta_\ell)-(\mu^\star-2\beta_h)
\le
-\Delta_\ell+4\beta_\ell .
\]

Now suppose \(h\ne k^\star\). Since \(k^\star\ne h\), the definition of
\(\ell\) gives
\[
U_\ell=\max_{a\ne h}U_a\ge U_{k^\star}\ge \mu^\star .
\]
If \(\ell\ne k^\star\), then \(U_\ell\le \mu_\ell+2\beta_\ell\), and hence
\[
2\beta_\ell
\ge
\mu^\star-\mu_\ell
=
\Delta_\ell .
\]
If instead \(\ell=k^\star\), we use the comparison \(L_\ell\le L_h\).
By \eqref{eq:fixed-valid-bounds},
\[
\mu^\star-2\beta_\ell
=
\mu_\ell-2\beta_\ell
\le
L_\ell
\le
L_h
\le
\mu_h .
\]
Since \(h\ne k^\star\), we have
\[
\mu_h
=
\mu^\star-\Delta_h
\le
\mu^\star-\Delta_{k^\star}.
\]
Thus
\[
\mu^\star-2\beta_\ell
\le
\mu^\star-\Delta_{k^\star},
\]
and so
\[
2\beta_\ell\ge \Delta_{k^\star}=\Delta_\ell .
\]
Therefore, in all subcases with \(h\ne k^\star\),
\[
2\beta_\ell\ge \Delta_\ell .
\]
Combining this with \eqref{eq:fixed-case-ell-two-beta} yields
\[
B_h
\le
2\beta_\ell
\le
-\Delta_\ell+4\beta_\ell .
\]
Consequently,
\[
B_h
\le
\min\{2\beta_\ell,-\Delta_\ell+4\beta_\ell\}
=
\min\{0,-\Delta_\ell+2\beta_\ell\}+2\beta_\ell .
\]

\emph{Case 2: the sampled arm is \(k=h\).}
The sampling rule gives \(\beta_h\ge\beta_\ell\). We first show that
\(U_\ell\le U_h\). Suppose, for contradiction, that \(U_\ell>U_h\). Then
\[
L_\ell
=
U_\ell-2\beta_\ell
>
U_h-2\beta_\ell
\ge
U_h-2\beta_h
=
L_h .
\]
Since \(U_\ell=\max_{a\ne h}U_a\), and since \(U_\ell>U_h\), we have
\[
\max_{a\ne \ell}U_a\le U_\ell .
\]
Hence
\[
B_\ell
=
\max_{a\ne\ell}U_a-L_\ell
<
U_\ell-L_h
=
B_h,
\]
contradicting the minimality of \(h\). Thus \(U_\ell\le U_h\), and
\begin{equation}
\label{eq:fixed-case-h-two-beta}
B_h=U_\ell-L_h\le U_h-L_h=2\beta_h .
\end{equation}

We next prove
\[
B_h\le -\Delta_h+4\beta_h .
\]
If \(h=k^\star\), then every \(a\ne h\) satisfies
\(\mu_a\le \mu^\star-\Delta_h\). In particular,
\[
B_h
=
U_\ell-L_h
\le
(\mu_\ell+2\beta_\ell)-(\mu^\star-2\beta_h)
\le
-\Delta_h+2\beta_\ell+2\beta_h
\le
-\Delta_h+4\beta_h .
\]
If \(h\ne k^\star\), then \(k^\star\ne h\), and
\[
U_\ell=\max_{a\ne h}U_a\ge U_{k^\star}\ge \mu^\star .
\]
Since we already proved \(U_h\ge U_\ell\), it follows that
\(U_h\ge\mu^\star\). Combining this with
\(U_h\le\mu_h+2\beta_h\) gives
\[
2\beta_h
\ge
\mu^\star-\mu_h
=
\Delta_h .
\]
Thus
\[
2\beta_h\le -\Delta_h+4\beta_h,
\]
and \eqref{eq:fixed-case-h-two-beta} implies
\[
B_h\le -\Delta_h+4\beta_h .
\]
Therefore
\[
B_h
\le
\min\{2\beta_h,-\Delta_h+4\beta_h\}
=
\min\{0,-\Delta_h+2\beta_h\}+2\beta_h .
\]

The two cases prove the lemma.
\end{proof}

\begin{lemma}[A fixed-clipping sampled arm has large stochastic radius]
\label{lem:fixed-pulled-arm-radius}
On $\mathcal E_{\rm fix}$, if \textup{\textsc{FixedClip-1BitBAI}} samples arm $k$ in a non-stopping round, then
\[
r_\eps(N_k,\delta/K)>\frac{\widetilde\Delta_k}{4}.
\]
\end{lemma}

\begin{proof}
In a non-stopping round, $B_h>\eps$. By Lemma~\ref{lem:fixed-structural-comparison}, for the sampled arm $k$,
\[
\eps
<
\min\{0,-\Delta_k+2\beta_k\}+2\beta_k.
\]
If $2\beta_k\ge\Delta_k$, then the minimum is $0$, so
\[
\beta_k>\frac{\eps}{2}.
\]
If $2\beta_k<\Delta_k$, then the minimum is $-\Delta_k+2\beta_k$, so
\[
\eps<-
\Delta_k+4\beta_k,
\qquad\text{and hence}\qquad
\beta_k>\frac{\Delta_k+
\eps}{4}.
\]
Combining the two cases gives
\[
\beta_k>
\max\left\{\frac{\eps}{2},\frac{\Delta_k+
\eps}{4}\right\}.
\]
Since $r_\eps(N_k,\delta/K)=\beta_k-
\eps/4$, we obtain
\[
r_\eps(N_k,\delta/K)
>
\max\left\{\frac{\eps}{4},\frac{\Delta_k}{4}\right\}
=
\frac{\widetilde\Delta_k}{4}.
\]
\end{proof}

\begin{lemma}[Fixed-clipping counting bound]
\label{lem:fixed-counting}
On $\mathcal E_{\rm fix}$, \textup{\textsc{FixedClip-1BitBAI}} uses at most
\[
O\!\left(
H_{\rm fix}(\eps)
\left[
\log\frac{K}{\delta}+\log\log H_{\rm fix}(\eps)
\right]
\right)
\]
1-bit queries.
\end{lemma}

\begin{proof}
For each arm, define
\[
A_k
:=
\alpha_\eps
\left(
\frac{\sigma^2}{\widetilde\Delta_k^2}
+
\frac{\sigma^2}{\eps\widetilde\Delta_k}
\right),
\qquad
H_{\rm fix}(\eps)=\sum_{k=1}^K A_k.
\]
Let
\[
x_n:=x_n(\delta/K)
=
\log\left(\frac{\pi^2(1+\lfloor\log_2 n\rfloor)^2K}{3\delta}\right).
\]
We first record a deterministic implication. There exists a universal constant $c_0>0$ such that for all $x\ge1$,
\begin{equation}
\label{eq:fixed-threshold-implication}
n\ge c_0A_kx
\quad\Longrightarrow\quad
2\sqrt{\frac{3\alpha_\eps\sigma^2x}{n}}
+
\frac{2\alpha_\eps\sigma^2x}{n\eps}
\le
\frac{\widetilde\Delta_k}{4}.
\end{equation}
Indeed, take $c_0$ large enough. Since $A_k\ge \alpha_\eps\sigma^2/\widetilde\Delta_k^2$, the first term is at most
\[
2\sqrt{\frac{3\alpha_\eps\sigma^2x}{c_0A_kx}}
\le
2\sqrt{\frac{3}{c_0}}\,\widetilde\Delta_k,
\]
which is at most $\widetilde\Delta_k/8$ for large enough $c_0$. Similarly, since $A_k\ge \alpha_\eps\sigma^2/(\eps\widetilde\Delta_k)$,
\[
\frac{2\alpha_\eps\sigma^2x}{c_0A_kx\eps}
\le
\frac{2}{c_0}\widetilde\Delta_k
\le
\frac{\widetilde\Delta_k}{8}
\]
for large enough $c_0$. This proves \eqref{eq:fixed-threshold-implication}.

Fix an integer $T\ge1$ and suppose, toward a contradiction, that the algorithm has not stopped by the time it has made $T$ total queries. Consider any query made to arm $k$ before time $T$. Since $x_n$ is nondecreasing and the local count at that query is at most $T$, implication \eqref{eq:fixed-threshold-implication} with $x=x_T$ shows that if $N_k\ge c_0A_kx_T$, then
\[
r_\eps(N_k,\delta/K)
\le
\frac{\widetilde\Delta_k}{4},
\]
contradicting Lemma~\ref{lem:fixed-pulled-arm-radius}. Hence every arm $k$ can be queried only while
\[
N_k<c_0A_kx_T.
\]
Summing over arms gives
\[
T
=
\sum_{k=1}^K N_k
<
c_0H_{\rm fix}(\eps)x_T.
\]
Therefore any $T$ satisfying $T\ge c_0H_{\rm fix}(\eps)x_T$ forces the algorithm to have stopped by time $T$.

Now choose
\[
T=
\left\lceil
C H_{\rm fix}(\eps)
\left[
\log\frac{K}{\delta}+\log\log H_{\rm fix}(\eps)
\right]
\right\rceil
\]
for a sufficiently large universal constant $C$. For this choice,
\[
x_T
=O\!\left(
\log\frac{K}{\delta}+\log\log H_{\rm fix}(\eps)
\right),
\]
because $x_T$ contains only $\log(K/\delta)$ and the stitched term $2\log(1+\lfloor\log_2T\rfloor)=O(\log\log T)$, while $T$ is proportional to $H_{\rm fix}(\eps)$ times logarithmic factors. Taking $C$ large enough gives $T\ge c_0H_{\rm fix}(\eps)x_T$, proving the desired bound.
\end{proof}

\begin{proof}[Proof of Theorem~\ref{thm:1-bit-bai-upper}]
By Lemma~\ref{lem:fixed-good-event}, $\mathcal E_{\rm fix}$ holds with probability at least $1-\delta$. On this event, Lemma~\ref{lem:fixed-stopping-correct} proves that any returned arm is $\eps$-optimal, and Lemma~\ref{lem:fixed-counting} proves that the stopping time is bounded by the displayed sample complexity. Proposition~\ref{prop:deterministic-fixed} additionally establishes termination on every realization. Hence \textup{\textsc{FixedClip-1BitBAI}} is $(\eps,\delta)$-PAC and satisfies the claimed high-probability query bound.
\end{proof}

\subsection{Phased adaptive clipping: proof of Theorem~\ref{thm:phased-bai-upper}}
\label{app:proof-phased-clipping}

We now analyze the restart phased algorithm. We first construct a high-probability event on which all arm--phase confidence intervals are valid simultaneously (Lemma~\ref{lem:phased-good-event}). On this event, the elimination and stopping rules never discard the best arm and can only return an $\eps$-optimal arm (Lemma~\ref{lem:phased-correctness}). The sample complexity analysis then proceeds phase by phase. Lemma~\ref{lem:phase-sample-requirement} gives the number of samples needed to certify a phase-$j$ interval at radius $r_j$. Lemmas~\ref{lem:phase-elimination-scale} and~\ref{lem:phase-last-scale} show that each arm is sampled only down to a resolution comparable to its capped effective gap $\overline\Delta_k$. Summing the resulting dyadic costs over phases and arms gives the claimed bound (Lemma~\ref{lem:phased-counting}).

For each phase $j=0,1,\dots,J=\lceil\log_2(\sigma/\eps)\rceil$, recall
\[
\rho_j=2^{-j}\sigma,
\qquad
L_j=\frac{32\sigma^2}{\rho_j},
\qquad
r_j=\frac{\rho_j}{16},
\qquad
\alpha_j=1+\log\frac{L_j}{\sigma}=1+\log\frac{32\sigma}{\rho_j}.
\]
With $\gamma=\delta/[K(J+1)]$, the phase-$j$ confidence radius is
\[
\beta_j(n)
=
\frac{\rho_j}{64}
+2\sqrt{\frac{3\alpha_j\sigma^2x_n(\gamma)}{n}}
+\frac{32\alpha_j\sigma^2x_n(\gamma)}{\rho_j n}.
\]
The restart version uses independent fresh estimators in each arm--phase pair $(k,j)$.

\begin{lemma}[Phased good event]
\label{lem:phased-good-event}
With probability at least $1-\delta$, the event
\[
\mathcal E_{\rm ph}
:=
\bigcap_{k=1}^K
\bigcap_{j=0}^J
\bigcap_{n\ge1}
\left\{
\left|\widehat\mu_{k,j}(n)-\mu_k\right|
\le \beta_j(n)
\right\}
\]
holds throughout \textup{\textsc{RestartPhasedClip-1BitBAI}}.
\end{lemma}

\begin{proof}
Fix an arm--phase pair $(k,j)$. During phase $j$, the algorithm uses clipping level $L_j$ for arm $k$. Since
\[
\frac{\sigma^2}{2L_j}
=
\frac{\sigma^2}{2\cdot 32\sigma^2/\rho_j}
=
\frac{\rho_j}{64},
\]
Theorem~\ref{thm:confidence-bound} with failure probability $\gamma$ gives
\[
\Pr\left(
\forall n\ge1,
\left|\widehat\mu_{k,j}(n)-\mu_k\right|
\le\beta_j(n)
\right)
\ge 1-\gamma.
\]
There are $K(J+1)$ arm--phase pairs, so the union bound gives
\[
\Pr(\mathcal E_{\rm ph})
\ge
1-K(J+1)\gamma
=
1-\delta.
\]
Again, the confidence sequence is uniform in the local sample count $n$, so adaptive sampling inside the phase is allowed.
\end{proof}

\begin{lemma}[The best arm is never eliminated]
\label{lem:phased-best-never-eliminated}
On $\mathcal E_{\rm ph}$, \textup{\textsc{RestartPhasedClip-1BitBAI}} never eliminates $k^\star$.
\end{lemma}

\begin{proof}
Fix a phase $j$ and suppose $k^\star$ is active when the elimination rule is applied. Let $A$ be the active set at that moment and let
\[
m:=\max_{i\in A}\LCB_{i,j}.
\]
On $\mathcal E_{\rm ph}$, for every active arm $i$,
\[
\LCB_{i,j}\le \mu_i\le \mu^\star.
\]
Therefore
\[
m=\max_{i\in A}\LCB_{i,j}\le \mu^\star.
\]
Also, since $k^\star$ is active and its interval is valid,
\[
\UCB_{k^\star,j}\ge \mu^\star.
\]
Combining the two displays gives
\[
\UCB_{k^\star,j}
\ge
\mu^\star
\ge
m
>
m-\eps.
\]
Thus $k^\star$ cannot satisfy the elimination condition $\UCB_{k,j}<m-\eps$.
\end{proof}

\begin{lemma}[Phased stopping and final-output correctness]
\label{lem:phased-correctness}
On $\mathcal E_{\rm ph}$, every arm returned by \textup{\textsc{RestartPhasedClip-1BitBAI}} is $\eps$-optimal.
\end{lemma}

\begin{proof}
There are three possible ways the algorithm returns.

\emph{Singleton return.} If the active set has size one, Lemma~\ref{lem:phased-best-never-eliminated} implies that the unique active arm must be $k^\star$, hence the output is correct.

\emph{Interval stopping.} Suppose the algorithm returns an arm $h$ in phase $j$ because, with $\ell\in\argmax_{a\in A\setminus\{h\}}\UCB_{a,j}$,
\[
\UCB_{\ell,j}-\LCB_{h,j}\le \eps.
\]
If $h=k^\star$, the claim is immediate. Otherwise, Lemma~\ref{lem:phased-best-never-eliminated} implies that $k^\star$ is still active, and since $k^\star\ne h$,
\[
\UCB_{\ell,j}
=
\max_{a\in A\setminus\{h\}}\UCB_{a,j}
\ge
\UCB_{k^\star,j}
\ge
\mu^\star.
\]
Also $\LCB_{h,j}\le \mu_h$, so
\[
\mu^\star-
\mu_h
\le
\UCB_{\ell,j}-
\LCB_{h,j}
\le
\eps.
\]
Thus $h$ is $\eps$-optimal.

\emph{Final-phase return.} Finally, suppose the algorithm reaches the final phase $J$ and returns $h\in\argmax_{k\in A}\LCB_{k,J}$ after all active arms satisfy $\beta_{k,J}\le r_J$. By Lemma~\ref{lem:phased-best-never-eliminated}, $k^\star$ is active. Hence
\[
\LCB_{h,J}\ge \LCB_{k^\star,J}.
\]
On $\mathcal E_{\rm ph}$ and because $\beta_{k^\star,J}\le r_J$,
\[
\LCB_{k^\star,J}
=
\widehat\mu_{k^\star,J}-\beta_{k^\star,J}
\ge
\mu^\star-2\beta_{k^\star,J}
\ge
\mu^\star-2r_J.
\]
Similarly, $\mu_h\ge \LCB_{h,J}$. Therefore
\[
\mu_h
\ge
\LCB_{h,J}
\ge
\LCB_{k^\star,J}
\ge
\mu^\star-2r_J.
\]
Since $r_J=\rho_J/16$ and $\rho_J\le \eps$, we have $2r_J\le\eps/8\le\eps$. Thus $\mu_h\ge\mu^\star-\eps$.
\end{proof}

\begin{lemma}[Per-phase sample requirement]
\label{lem:phase-sample-requirement}
For each phase $j$, define
\[
\Gamma_j
:=
\log\frac{K(J+1)}{\delta}
+
\log\log\left(e^e\frac{\sigma}{\rho_j}\right)
+
\log\log\left(e^e\frac{K(J+1)}{\delta}\right).
\]
There exists a universal constant $C>0$ such that
\[
n_j^\circ
:=
\left\lceil
C\frac{\alpha_j\sigma^2}{\rho_j^2}\Gamma_j
\right\rceil
\]
satisfies
\[
\beta_j(n_j^\circ)\le r_j=\frac{\rho_j}{16}.
\]
Consequently, in phase $j$, an active arm is queried at most $n_j^\circ$ times before its phase-$j$ radius is at most $r_j$.
\end{lemma}

\begin{proof}
Recall
\[
\beta_j(n)
=
\frac{\rho_j}{64}
+2\sqrt{\frac{3\alpha_j\sigma^2x_n(\gamma)}{n}}
+\frac{32\alpha_j\sigma^2x_n(\gamma)}{\rho_j n}.
\]
The bias term is already $\rho_j/64$. It remains to make each of the two stochastic terms at most $\rho_j/64$.

First we bound the stitched logarithm at $n_j^\circ$. Since
\[
x_n(\gamma)
=
\log\left(\frac{\pi^2(1+\lfloor\log_2 n\rfloor)^2}{3\gamma}\right),
\]
there is a universal constant $c_1>0$ such that for all $n\ge2$,
\[
x_n(\gamma)
\le
c_1+
\log\frac1\gamma
+2\log\log(e^en).
\]
Here $\gamma=\delta/[K(J+1)]$, so
\[
\log\frac1\gamma=\log\frac{K(J+1)}{\delta}.
\]
Moreover,
\[
n_j^\circ
\le
C_1\frac{\alpha_j\sigma^2}{\rho_j^2}\Gamma_j
\]
for a universal constant $C_1$ depending only on $C$, and
\[
\alpha_j=1+\log\frac{32\sigma}{\rho_j}
\le
C_2\log\left(e^e\frac{\sigma}{\rho_j}\right).
\]
Therefore
\[
\log\log(e^en_j^\circ)
\le
C_3
+
\log\log\left(e^e\frac{\sigma}{\rho_j}\right)
+
\log\log\left(e^e\frac{K(J+1)}{\delta}\right),
\]
where the last term absorbs the dependence on $\Gamma_j$. Hence, for a universal constant $C_0>0$,
\begin{equation}
\label{eq:phase-x-bound}
x_{n_j^\circ}(\gamma)
\le
C_0\Gamma_j.
\end{equation}
Now choose the constant $C$ in the definition of $n_j^\circ$ large enough so that
\[
2\sqrt{\frac{3\alpha_j\sigma^2 C_0\Gamma_j}{n_j^\circ}}
\le
\frac{\rho_j}{64}
\]
and
\[
\frac{32\alpha_j\sigma^2 C_0\Gamma_j}{\rho_j n_j^\circ}
\le
\frac{\rho_j}{64}.
\]
Using \eqref{eq:phase-x-bound}, both stochastic terms in $\beta_j(n_j^\circ)$ are at most $\rho_j/64$. Thus
\[
\beta_j(n_j^\circ)
\le
\frac{\rho_j}{64}+\frac{\rho_j}{64}+\frac{\rho_j}{64}
=
\frac{3\rho_j}{64}
<
\frac{\rho_j}{16}=r_j.
\]
In the algorithm, once an arm has $\beta_{k,j}\le r_j$, it is not queried again in that phase unless all active arms already satisfy the same condition, in which case the phase ends. Therefore an arm receives at most $n_j^\circ$ phase-$j$ samples.
\end{proof}

\begin{lemma}[Elimination at the right scale]
\label{lem:phase-elimination-scale}
Assume $\mathcal E_{\rm ph}$ holds. Fix a phase $j$. If a suboptimal active arm $k\ne k^\star$ satisfies
\[
\Delta_k>\eps+\frac{\rho_j}{4},
\]
then by the time phase $j$ ends, either the algorithm has already stopped or arm $k$ has been eliminated.
\end{lemma}

\begin{proof}
If the algorithm stops during phase $j$, there is nothing to prove. Otherwise, phase $j$ ends only after every active arm has phase radius at most $r_j=\rho_j/16$, and the elimination rule has been applied. Let $A$ be the active set at that point and define
\[
m:=\max_{i\in A}\LCB_{i,j}.
\]
By Lemma~\ref{lem:phased-best-never-eliminated}, $k^\star\in A$. Hence, on $\mathcal E_{\rm ph}$,
\[
m
\ge
\LCB_{k^\star,j}
\ge
\mu^\star-2r_j.
\]
For the suboptimal arm $k$,
\[
\UCB_{k,j}
\le
\mu_k+2r_j
=
\mu^\star-
\Delta_k+2r_j.
\]
If $\Delta_k>\eps+4r_j$, then
\[
\mu^\star-
\Delta_k+2r_j
<
\mu^\star-
\eps-2r_j.
\]
Since $4r_j=\rho_j/4$, the condition $\Delta_k>\eps+\rho_j/4$ is exactly $\Delta_k>\eps+4r_j$. Combining the preceding displays,
\[
\UCB_{k,j}
<
\mu^\star-
\eps-2r_j
\le
m-
\eps.
\]
Thus $k$ satisfies the elimination rule $\UCB_{k,j}<m-\eps$ and is eliminated.
\end{proof}

\begin{lemma}[Last relevant phase]
\label{lem:phase-last-scale}
Assume $\mathcal E_{\rm ph}$ holds. For every arm $k$, the resolution of the last phase in which $k$ is sampled is bounded below by a universal constant multiple of
\[
\overline\Delta_k:=\min\{\max\{\Delta_k,\eps\},\sigma\}.
\]
More precisely, there exists a universal constant $c>0$ such that if $\rho$ is the resolution of the last phase in which $k$ is sampled, then
\[
\rho\ge c\,\overline\Delta_k.
\]
\end{lemma}

\begin{proof}
We first consider a suboptimal arm $k\ne k^\star$.

\emph{Case 1: $\Delta_k\ge2\eps$ and $\Delta_k\le\sigma$.} Then $\overline\Delta_k=\Delta_k$. Let $j_k$ be the first phase index such that $\rho_{j_k}\le\Delta_k$. By dyadicity,
\[
\frac{\Delta_k}{2}<\rho_{j_k}\le\Delta_k.
\]
Since $\Delta_k\ge2\eps$, we have $\eps\le\Delta_k/2$. Therefore
\[
\eps+
\frac{\rho_{j_k}}{4}
\le
\frac{\Delta_k}{2}+
\frac{\Delta_k}{4}
=
\frac{3\Delta_k}{4}
<
\Delta_k.
\]
Lemma~\ref{lem:phase-elimination-scale} implies that arm $k$ is eliminated by the end of phase $j_k$, unless the algorithm has already stopped. Thus $k$ is not sampled beyond phase $j_k$, so its last sampled resolution is at least $\rho_{j_k}>\overline\Delta_k/2$.

\emph{Case 2: $\Delta_k\ge2\eps$ and $\Delta_k>\sigma$.} Then $\overline\Delta_k=\sigma$. At phase $0$, $\rho_0=\sigma$. Since $\eps\le\Delta_k/2$ and $\rho_0/4=\sigma/4<\Delta_k/4$,
\[
\eps+
\frac{\rho_0}{4}
<
\frac{\Delta_k}{2}+
\frac{\Delta_k}{4}
=
\frac{3\Delta_k}{4}
<
\Delta_k.
\]
Thus Lemma~\ref{lem:phase-elimination-scale} eliminates arm $k$ by the end of phase $0$, unless the algorithm stops. If $k$ is sampled, its last sampled resolution is therefore $\sigma=\overline\Delta_k$.

\emph{Case 3: $\Delta_k<2\eps$.} Then $\overline\Delta_k\le2\eps$. Also $J=\lceil\log_2(\sigma/\eps)\rceil$, so
\[
\rho_J\le\eps<2\rho_J.
\]
No arm is sampled after phase $J$. Hence the last sampled resolution satisfies $\rho\ge\rho_J>\eps/2\ge\overline\Delta_k/4$.

Now consider the optimal arm $k^\star$, with $\Delta_{k^\star}=\min_{j\ne k^\star}\Delta_j$. If $\Delta_{k^\star}\ge2\eps$ and $\Delta_{k^\star}\le\sigma$, let $j_\star$ be the first phase with $\rho_{j_\star}\le\Delta_{k^\star}$. The same calculation as in Case 1 shows that every suboptimal arm $a$ satisfies
\[
\Delta_a\ge\Delta_{k^\star}>
\eps+
\frac{\rho_{j_\star}}{4},
\]
so all suboptimal arms are eliminated by the end of phase $j_\star$, unless the algorithm has already stopped. The active set then becomes singleton and the algorithm returns, so $k^\star$ is not sampled beyond this phase and its last sampled resolution exceeds $\overline\Delta_{k^\star}/2$. If $\Delta_{k^\star}\ge2\eps$ and $\Delta_{k^\star}>\sigma$, the same argument applies at phase $0$, giving last sampled resolution $\sigma=\overline\Delta_{k^\star}$. If $\Delta_{k^\star}<2\eps$, the final-phase argument in Case 3 gives last sampled resolution greater than $\overline\Delta_{k^\star}/4$. This proves the claim with $c=1/4$. Early stopping can only increase the last sampled resolution, so no upper bound on it is needed or asserted.
\end{proof}

\begin{lemma}[Phased counting bound]
\label{lem:phased-counting}
On $\mathcal E_{\rm ph}$, \textup{\textsc{RestartPhasedClip-1BitBAI}} uses at most
\[
O\!\left(
\sum_{k=1}^K
\Lambda(\overline\Delta_k)
\frac{\sigma^2}{\overline\Delta_k^2}
\Gamma(\overline\Delta_k;K,J,\delta)
\right)
\]
1-bit queries, where
\[
\Lambda(u)=1+\log\frac{\sigma}{u}
\]
and
\[
\Gamma(u;K,J,\delta)
=
\log\frac{K(J+1)}{\delta}
+
\log\log\left(e^e\frac{\sigma}{u}\right)
+
\log\log\left(e^e\frac{K(J+1)}{\delta}\right).
\]
\end{lemma}

\begin{proof}
By Lemma~\ref{lem:phase-sample-requirement}, in phase $j$ each active arm receives at most
\[
n_j^\circ
=
O\!\left(
\frac{\alpha_j\sigma^2}{\rho_j^2}\Gamma_j
\right)
\]
queries. By Lemma~\ref{lem:phase-last-scale}, arm $k$ is sampled only in phases whose resolution is at least a universal constant multiple of $\overline\Delta_k$. Thus the total number of samples of arm $k$ is bounded by
\[
\sum_{j:\,\rho_j\gtrsim\overline\Delta_k}
O\!\left(
\frac{\alpha_j\sigma^2}{\rho_j^2}\Gamma_j
\right).
\]
Let $j_k$ be the last phase in this sum. Since $\rho_j$ decreases geometrically,
\[
\sum_{j=0}^{j_k}\frac{1}{\rho_j^2}
=
\frac1{\sigma^2}\sum_{j=0}^{j_k}4^j
\le
\frac{4}{3}\frac1{\rho_{j_k}^2}.
\]
Moreover, both $\alpha_j$ and $\Gamma_j$ are nondecreasing as $\rho_j$ decreases, so for $j\le j_k$,
\[
\alpha_j\le \alpha_{j_k},
\qquad
\Gamma_j\le \Gamma_{j_k}.
\]
Therefore the contribution of arm $k$ is at most a constant times
\[
\alpha_{j_k}
\frac{\sigma^2}{\rho_{j_k}^2}
\Gamma_{j_k}.
\]
The largest scheduled index $j_k$ satisfying $\rho_{j_k}\gtrsim\overline\Delta_k$ has $\rho_{j_k}=\Theta(\overline\Delta_k)$ by dyadicity and the final-scale cap. This index bounds the cost even if the actual last sampled phase is earlier. Hence
\[
\alpha_{j_k}
=1+\log\frac{32\sigma}{\rho_{j_k}}
=O\!\left(1+\log\frac{\sigma}{\overline\Delta_k}\right)
=O(\Lambda(\overline\Delta_k)),
\]
\[
\frac{\sigma^2}{\rho_{j_k}^2}
=O\!\left(\frac{\sigma^2}{\overline\Delta_k^2}\right),
\]
and
\[
\Gamma_{j_k}
=O\!\left(\Gamma(\overline\Delta_k;K,J,\delta)\right).
\]
Thus arm $k$ contributes at most
\[
O\!\left(
\Lambda(\overline\Delta_k)
\frac{\sigma^2}{\overline\Delta_k^2}
\Gamma(\overline\Delta_k;K,J,\delta)
\right).
\]
Summing this bound over $k=1,\dots,K$ proves the lemma.
\end{proof}

\begin{proof}[Proof of Theorem~\ref{thm:phased-bai-upper}]
By Lemma~\ref{lem:phased-good-event}, $\mathcal E_{\rm ph}$ holds with probability at least $1-\delta$. On this event, Lemma~\ref{lem:phased-correctness} proves that every possible return is $\eps$-optimal. Lemma~\ref{lem:phased-counting} gives the stated high-probability query bound. Proposition~\ref{prop:deterministic-phased} establishes termination on every realization. Hence \textup{\textsc{RestartPhasedClip-1BitBAI}} is $(\eps,\delta)$-PAC and satisfies the claimed sample complexity. The displayed main-text bound follows because $\overline\Delta_k\ge\eps$, $J+1=O(1+\log(\sigma/\eps))$, and the remaining iterated logarithms in $\Gamma$ are absorbed into $\log(K/\delta)+\log\log(e^e\sigma/\eps)$.
\end{proof}

\begin{lemma}[Reusable phased clipping has the same order]
\label{lem:reuse-phased-counting}
Consider the phased algorithm with the same active-arm, stopping, and elimination rules as
\textup{\textsc{RestartPhasedClip-1BitBAI}}, but replace the restarted phase-\(j\) estimator by
the following reusable slab estimator. Let \(L_{-1}:=0\), let
\[
I_0:=[0,L_0],
\qquad
I_s:=(L_{s-1},L_s]\quad (s\ge1),
\]
and define
\[
A_s
:=
\int_{I_s}\frac{dt}{\max\{\sigma,t\}},
\qquad
q_s(t)
:=
\frac{1}{A_s\max\{\sigma,t\}},
\qquad t\in I_s .
\]
For each arm \(k\) and slab \(s\), draw \(S\sim{\rm Unif}\{-1,+1\}\), \(T\sim q_s\), query
\[
Y=\1\{S(X_k-c_k)\ge T\},
\]
and define
\[
Z_{k,s}
:=
\frac{2S}{q_s(T)}Y .
\]
In phase \(j\), form
\[
\widehat\mu^{\rm reuse}_{k,j}
:=
c_k+\sum_{s=0}^j \overline Z_{k,s},
\]
where \(\overline Z_{k,s}\) is the empirical average of the retained samples from slab \(s\).
For
\[
\Gamma_j
:=
\log\frac{K(J+1)}{\delta}
+
\log\log\left(e^e\frac{\sigma}{\rho_j}\right)
+
\log\log\left(e^e\frac{K(J+1)}{\delta}\right),
\]
there is a universal constant \(C_0>0\) such that it is enough to top up each slab \(s\le j\) to
\[
m_{s,j}
:=
\left\lceil
C_0\,\frac{A_s\sigma^2}{\rho_j^2}\Gamma_j
\right\rceil
\]
samples in order to guarantee the phase-\(j\) radius \(r_j=\rho_j/16\). Consequently, on an event
of probability at least \(1-\delta\), the reusable phased algorithm uses at most
\[
O\!\left(
\sum_{k=1}^K
\Lambda(\overline\Delta_k)
\frac{\sigma^2}{\overline\Delta_k^2}
\Gamma(\overline\Delta_k;K,J,\delta)
\right)
\]
1-bit queries, where
\[
\Lambda(u)=1+\log\frac{\sigma}{u}
\]
and
\[
\Gamma(u;K,J,\delta)
=
\log\frac{K(J+1)}{\delta}
+
\log\log\left(e^e\frac{\sigma}{u}\right)
+
\log\log\left(e^e\frac{K(J+1)}{\delta}\right).
\]
Thus reuse preserves the order of Theorem~\ref{thm:phased-bai-upper}.
\end{lemma}

\begin{proof}
Fix an arm \(k\), and write \(W_k=X_k-c_k\). First note that \(q_s\) is a density on \(I_s\), and
\[
\sum_{s=0}^j A_s
=
\int_0^{L_j}\frac{dt}{\max\{\sigma,t\}}
=
1+\log\frac{L_j}{\sigma}
=
\alpha_j .
\]
Moreover, by the same tail-integral calculation used in Lemma~\ref{lem:tail-integral},
\[
\mathbb E[Z_{k,s}]
=
\int_{I_s}
\Bigl(
\Pr(W_k\ge t)-\Pr(W_k\le -t)
\Bigr)dt.
\]
Therefore
\[
\sum_{s=0}^j\mathbb E[Z_{k,s}]
=
\mathbb E[[W_k]_{L_j}],
\]
so the reusable phase-\(j\) estimator is unbiased for the clipped target \(c_k+\mathbb E[[W_k]_{L_j}]\).

We next establish the variance and range bounds needed for Bernstein's inequality. Conditional on
\(W_k=w\),
\[
\mathbb E[Z_{k,s}^2\mid W_k=w]
=
2A_s\int_{I_s}\max\{\sigma,t\}\1\{|w|\ge t\}\,dt .
\]
Hence, writing \(v_{k,s}:=\Var(Z_{k,s})\), we have
\[
\sum_{s=0}^j\frac{v_{k,s}}{A_s}
\le
\sum_{s=0}^j\frac{\mathbb E[Z_{k,s}^2]}{A_s}
\le
2\mathbb E\!\left[
\int_0^{L_j}\max\{\sigma,t\}\1\{|W_k|\ge t\}\,dt
\right].
\]
For any \(x\ge0\),
\[
\int_0^{L_j}\max\{\sigma,t\}\1\{x\ge t\}\,dt
\le
\int_0^x\max\{\sigma,t\}\,dt
\le
\sigma x+\frac{x^2}{2}.
\]
Since \(\mathbb E[W_k^2]\le 2\sigma^2\), this gives
\[
\sum_{s=0}^j\frac{v_{k,s}}{A_s}
\le
C\sigma^2
\]
for a universal constant \(C>0\).

Also,
\[
|Z_{k,s}|
\le
\frac{2}{q_s(T)}
=
2A_s\max\{\sigma,T\}
\le
2A_sL_s.
\]
Because \(A_s\ge \log 2\) for every \(s\), the centered variables satisfy
\[
|Z_{k,s}-\mathbb E Z_{k,s}|
\le
C A_sL_s
\]
for another universal constant \(C>0\).

Now fix a phase \(j\). Suppose each slab \(s\le j\) has been topped up to \(m_{s,j}\) samples. Define
\[
\overline Z_{k,s}
=
\frac1{m_{s,j}}\sum_{i=1}^{m_{s,j}}Z_{k,s,i}.
\]
The centered error of the reusable clipped estimator is
\[
\sum_{s=0}^j
\left(\overline Z_{k,s}-\mathbb E Z_{k,s}\right).
\]
Its Bernstein variance proxy is
\[
V_{k,j}
:=
\sum_{s=0}^j\frac{v_{k,s}}{m_{s,j}},
\]
and its maximum summand range is
\[
B_{k,j}
:=
\max_{s\le j}\frac{C A_sL_s}{m_{s,j}}.
\]
By the choice of \(m_{s,j}\),
\[
V_{k,j}
\le
\frac{\rho_j^2}{C_0\sigma^2\Gamma_j}
\sum_{s=0}^j\frac{v_{k,s}}{A_s}
\le
\frac{C\rho_j^2}{C_0\Gamma_j},
\]
and
\[
B_{k,j}
\le
\max_{s\le j}
\frac{C A_sL_s}{C_0A_s\sigma^2\Gamma_j/\rho_j^2}
\le
\frac{C L_j\rho_j^2}{C_0\sigma^2\Gamma_j}
=
\frac{32C\rho_j}{C_0\Gamma_j},
\]
where the last equality uses \(L_j=32\sigma^2/\rho_j\).

Bernstein's inequality therefore gives, for a universal constant \(c>0\),
\[
\Pr\left(
\left|
\sum_{s=0}^j
\left(\overline Z_{k,s}-\mathbb E Z_{k,s}\right)
\right|
>
2\sqrt{V_{k,j}\Gamma_j}+B_{k,j}\Gamma_j
\right)
\le
e^{-c\Gamma_j}.
\]
Since \(\Gamma_j\ge \log(K(J+1)/\delta)\), a union bound over all \(K(J+1)\) arm--phase pairs
shows that, with probability at least \(1-\delta\), the above Bernstein bound holds simultaneously
for every \(k\) and \(j\). Taking \(C_0\) sufficiently large yields
\[
2\sqrt{V_{k,j}\Gamma_j}+B_{k,j}\Gamma_j
\le
\frac{3\rho_j}{64}.
\]
The clipping bias at level \(L_j\) is
\[
\left|
\mathbb E[W_k]-\mathbb E[[W_k]_{L_j}]
\right|
\le
\frac{\sigma^2}{2L_j}
=
\frac{\rho_j}{64}.
\]
Consequently, on the simultaneous event,
\[
\left|
\widehat\mu^{\rm reuse}_{k,j}-\mu_k
\right|
\le
\frac{3\rho_j}{64}
+
\frac{\rho_j}{64}
=
\frac{\rho_j}{16}
=
r_j
\]
whenever all slabs \(s\le j\) have been topped up to \(m_{s,j}\). Thus the reusable estimator certifies
the same phase-\(j\) radius as the restarted estimator.

It remains to count samples. For a fixed arm \(k\), let \(j_k\) be the last phase in which \(k\) is sampled
by the reusable algorithm. Since samples are retained, the total number of slab samples collected for
arm \(k\) is bounded by the final top-up target at phase \(j_k\), namely
\[
\sum_{s=0}^{j_k}m_{s,j_k}
=
O\!\left(
\frac{\sigma^2}{\rho_{j_k}^2}\Gamma_{j_k}
\sum_{s=0}^{j_k}A_s
\right)
=
O\!\left(
\frac{\alpha_{j_k}\sigma^2}{\rho_{j_k}^2}\Gamma_{j_k}
\right).
\]
On the same valid-interval event, the stopping and elimination arguments from
Lemmas~\ref{lem:phased-best-never-eliminated}--\ref{lem:phase-last-scale} apply verbatim, since
they only use validity of the confidence intervals and the fact that each completed phase reaches
radius \(r_j\). Hence Lemma~\ref{lem:phase-last-scale} gives
\[
\rho_{j_k}\ge c\,\overline\Delta_k .
\]
Therefore
\[
\alpha_{j_k}
=
1+\log\frac{32\sigma}{\rho_{j_k}}
=
O\!\left(1+\log\frac{\sigma}{\overline\Delta_k}\right)
=
O(\Lambda(\overline\Delta_k)),
\]
\[
\frac{\sigma^2}{\rho_{j_k}^2}
=
O\!\left(
\frac{\sigma^2}{\overline\Delta_k^2}
\right),
\]
and
\[
\Gamma_{j_k}
=
O\!\left(
\Gamma(\overline\Delta_k;K,J,\delta)
\right).
\]
Thus arm \(k\) contributes at most
\[
O\!\left(
\Lambda(\overline\Delta_k)
\frac{\sigma^2}{\overline\Delta_k^2}
\Gamma(\overline\Delta_k;K,J,\delta)
\right)
\]
queries. Summing over \(k=1,\dots,K\) proves the claimed bound.
\end{proof}

\subsection{Deterministic finite query bounds}
\label{app:deterministic-bounds}
The high-probability bounds above quantify gap-dependent behavior. For positive accuracy, both algorithms also have finite gap-independent query caps on every realization. These bounds establish termination without conditioning on confidence validity. We assume supplied anchors and $0<\eps\le\sigma$; for larger accuracy one may run the algorithms with target $\sigma$.

\subsubsection{Fixed clipping}
\label{app:deterministic-fixed}
\begin{proposition}[Pathwise fixed-clipping cap]
\label{prop:deterministic-fixed}
Let $A_\eps=\alpha_\eps\sigma^2/\eps^2$, and let $m_\eps$ be the smallest nonnegative integer satisfying
\[
2^{m_\eps}\ge768 A_\eps
\log\left(\frac{\pi^2 K(m_\eps+1)^2}{3\delta}\right).
\]
Set $n_{\rm fix}^\circ=2^{m_\eps}$. Then on every realization,
\[
\tau_{\rm fix}\le K n_{\rm fix}^\circ
=O\!\left(K A_\eps\left[\log\frac K\delta+
\log\log(e^e A_\eps)\right]\right).
\]
\end{proposition}

\begin{proof}
In any non-stopping round, if $k\in\{h,\ell\}$ is sampled, the interval geometry alone gives $B_h\le2\beta_k$. To see this for $k=\ell$, the sampling rule implies $\beta_\ell\ge\beta_h$. If $\LCB_\ell>\LCB_h$, then $\UCB_\ell>\UCB_h$ and $B_\ell<B_h$, contradicting the minimality of $h$. Hence $B_h=\UCB_\ell-\LCB_h\le2\beta_\ell$. For $k=h$, the sampling rule implies $\beta_h\ge\beta_\ell$. If $\UCB_\ell>\UCB_h$, then $\LCB_\ell>\LCB_h$ and again $B_\ell<B_h$, a contradiction. Thus $B_h\le2\beta_h$. Neither argument uses interval validity.

Consequently, every sampled arm has $\beta_\eps(N_k,\delta/K)>\eps/2$. At the deterministic count $n_{\rm fix}^\circ$, its stitched logarithm is exactly the logarithm in the proposition, and
\[
2\sqrt{\frac{3\alpha_\eps\sigma^2 x_{n_{\rm fix}^\circ}(\delta/K)}{n_{\rm fix}^\circ}}
\le\frac\eps8,
\qquad
\frac{2\alpha_\eps\sigma^2 x_{n_{\rm fix}^\circ}(\delta/K)}{n_{\rm fix}^\circ\eps}
\le\frac\eps{384}.
\]
Including the bias $\eps/4$ gives a radius strictly below $\eps/2$. Once an arm reaches this count it cannot be sampled again, since its unchanged estimate and local count retain the same radius. Including the initialization query, each arm receives at most $n_{\rm fix}^\circ$ queries, proving the cap. Exponential growth of $2^m$ versus the logarithm ensures $m_\eps<\infty$. The same inversion as in Appendix~\ref{app:proof-estimator-complexity}, with rounding to a power of two, gives the stated order.
\end{proof}

\subsubsection{Restart phased clipping}
\label{app:deterministic-phased}
\begin{proposition}[Pathwise phased cap]
\label{prop:deterministic-phased}
With $n_j^\circ$ and $\Gamma_j$ from Lemma~\ref{lem:phase-sample-requirement}, on every realization,
\[
\tau_{\rm ph}\le K\sum_{j=0}^J n_j^\circ
=O\!\left(K\left(1+\log\frac\sigma\eps\right)
\frac{\sigma^2}{\eps^2}\Gamma_J\right).
\]
\end{proposition}

\begin{proof}
Lemma~\ref{lem:phase-sample-requirement} gives $\beta_j(n_j^\circ)\le r_j$ deterministically. The sampling rule selects only an arm with radius greater than $r_j$: it samples the larger-radius candidate or challenger when possible, and otherwise an active arm of maximum radius. An arm reaching $n_j^\circ$ is therefore never sampled again in that phase. Once all active radii are at most $r_j$, the algorithm returns or advances to the next phase. The active set cannot become empty, since an arm maximizing the LCB has UCB at least that LCB and survives elimination, even for invalid intervals. There are $J+1$ phases, so their caps sum to at most $K\sum_j n_j^\circ$. Finally, $\rho_J\in(\eps/2,\eps]$, $\alpha_j$ and $\Gamma_j$ are nondecreasing, and $\sum_{j=0}^J4^j\le(4/3)4^J$. These facts yield the displayed order.
\end{proof}

Combining either pathwise cap with its simultaneous confidence event proves almost-sure termination and $(\eps,\delta)$-PAC correctness. The caps are gap-independent; the high-probability theorems give sharper instance-dependent guarantees.

\subsection{Lau--Scarlett estimators with phased successive elimination}
\label{app:ls-successive-elimination}
We record a batched alternative to make the scope of our rate claim explicit. Use $\rho_j=2^{-j}\sigma$, $J=\lceil\log_2(\sigma/\eps)\rceil$, and $\gamma=\delta/[K(J+1)]$. In phase $j$, invoke the finite-variance fixed-target estimator of \citet{laujohnathansequential} separately on each active arm with accuracy $\rho_j/8$ and failure probability $\gamma$, using the supplied anchor. Retain arms whose estimates are within $\rho_j/2$ of the largest completed estimate. Return a singleton immediately, or any survivor at phase $J$.

A union bound gives simultaneous accuracy of all completed arm--phase estimates with probability at least $1-\delta$. On this event, the estimated maximum exceeds the best arm's estimate by at most $\rho_j/4$, so the best arm survives. A suboptimal arm with $\Delta_k>3\rho_j/4$ has estimated deficit greater than $\rho_j/2$ and is removed. Conversely every survivor has $\Delta_k\le3\rho_j/4$, so a final-phase survivor is $\eps$-optimal.

The deterministic cost per active arm in phase $j$ is
\[
m_j=O\!\left(\frac{\sigma^2}{\rho_j^2}
\left(1+\log\frac\sigma{\rho_j}\right)
\log\frac{K(J+1)}\delta\right).
\]
An arm is not sampled at resolutions finer than a constant multiple of $\overline\Delta_k$ on the simultaneous event; for the best arm, all challengers are removed by the corresponding scale. Summing geometrically increasing phase costs gives
\[
O\!\left(\sum_{k=1}^K\frac{\sigma^2}{\overline\Delta_k^2}
\left(1+\log\frac\sigma{\overline\Delta_k}\right)
\log\frac{K(J+1)}\delta\right)
=O\!\left(H_{\rm gap}\left[\log\frac K\delta+
\log\log\left(e^e\frac\sigma\eps\right)\right]\right).
\]
Thus this wrapper matches the leading phased BAI rate. It gives simultaneous confidence at completed checkpoints, while Theorem~\ref{thm:confidence-bound} also provides valid intervals after every intermediate local sample.

\newpage
\section{Proofs for Section~\ref{sec:lower_bound} (Lower bound)}
\label{app:lower_bound}

The proof of Theorem~\ref{thm:lower-bound} establishes an information-theoretic limit by explicitly constructing a ``hard pair'' of distributions \((\nu_0, \nu_1)\). These distributions are carefully engineered to have a mean gap of exactly \(\Delta\), yet they are geometrically structured so that any single 1-bit measurable set query yields at most \(O(\Delta^2 / (J\sigma^2))\) KL information. We then use opposite perturbations to embed this construction into a $K$-arm base instance and alternatives changing one arm at a time, and apply a change-of-measure argument over the predictable $\sigma$-algebra of the algorithm's interaction sequence.

\subsection{Construction of the hard pair}
\label{app:hard-distributions}

Let \(J\) be the largest integer such that
\[
2^{J-1} \le \frac{\sigma}{2\Delta}.
\]
This definition implies \(\dfrac{\sigma}{4\Delta} < 2^{J-1} \le \dfrac{\sigma}{2\Delta}\), so taking the base-2 logarithm yields \(J = \lfloor \log_2(\dfrac{\sigma}{\Delta}) \rfloor\). 
Consequently, the number of scales is \(J = \Theta(\log\dfrac{\sigma}{\Delta})\). 

For each index \(i = 1, \dots, J\), we define the geometric support scales \(m_i\) and their corresponding base probability weights \(w_i\):
\[
m_i := 2^{i-1}, \qquad w_i := \frac{1}{J m_i^2}.
\]
To create the mean shift, we introduce the probability mass perturbations
\[
\xi_i := \frac{\Delta}{2J m_i \sigma}.
\]
To ensure that we can subtract these perturbations without producing negative probabilities, we verify that \(\xi_i \le w_i/4\). By definition,
\[
\xi_i \le \frac{w_i}{4} \iff \frac{\Delta}{2J m_i \sigma} \le \frac{1}{4J m_i^2} \iff m_i \le \frac{\sigma}{2\Delta}.
\]
Since \(m_i = 2^{i-1} \le 2^{J-1} \le \dfrac{\sigma}{2\Delta}\) for all \(i \le J\), the condition holds. We now define the baseline distribution \(\nu_0\) and the perturbed distribution \(\nu_1\) over the finite support \(\{0\} \cup \{\pm m_i\sigma\}_{i=1}^J\):
\begin{align*}
\nu_0 &= \left(1 - \sum_{i=1}^J w_i\right)\delta_0 + \sum_{i=1}^J \frac{w_i}{2}\delta_{m_i\sigma} + \sum_{i=1}^J \frac{w_i}{2}\delta_{-m_i\sigma}, \\
\nu_1 &= \left(1 - \sum_{i=1}^J w_i\right)\delta_0 + \sum_{i=1}^J \left(\frac{w_i}{2} + \xi_i\right)\delta_{m_i\sigma} + \sum_{i=1}^J \left(\frac{w_i}{2} - \xi_i\right)\delta_{-m_i\sigma},
\end{align*}
where \(\delta_x\) denotes the Dirac measure at \(x\). Note that the central mass at \(0\) is strictly positive because \(\sum_{i=1}^J w_i = \frac{1}{J} \sum_{i=1}^J 4^{-(i-1)} < \dfrac{1}{J} \dfrac{4}{3} \le 1\) for \(J \ge 2\).

\begin{lemma}[Moments of the hard pair]
\label{lem:hard-pair-moments}
The distributions \(\nu_0\) and \(\nu_1\) satisfy
\[
\E_{\nu_0}[X]=0, \qquad \E_{\nu_1}[X]=\Delta,
\]
and have variances bounded by \(\sigma^2\).
\end{lemma}

\begin{proof}
For \(\nu_0\), symmetry across the origin immediately yields \(\E_{\nu_0}[X]=0\). Its second moment expands to exactly the variance proxy:
\[
\E_{\nu_0}[X^2] = \sum_{i=1}^J w_i(m_i\sigma)^2 = \sum_{i=1}^J \left(\frac{1}{J m_i^2}\right) m_i^2\sigma^2 = \sum_{i=1}^J \frac{\sigma^2}{J} = \sigma^2.
\]
For \(\nu_1\), the perturbations \(\xi_i\) are explicitly designed to shift the mean. Because the support points are symmetric, the base weights \(w_i/2\) cancel, leaving only the perturbations:
\[
\E_{\nu_1}[X] = \sum_{i=1}^J \left[ \left(\frac{w_i}{2} + \xi_i\right)(m_i\sigma) + \left(\frac{w_i}{2} - \xi_i\right)(-m_i\sigma) \right] = \sum_{i=1}^J 2\xi_i m_i\sigma.
\]
Substituting \(\xi_i = \dfrac{\Delta}{2J m_i \sigma}\), we obtain \(\E_{\nu_1}[X] = \sum_{i=1}^J \dfrac{\Delta}{J} = \Delta\). 

The second moment of \(\nu_1\) remains unchanged because the \(\xi_i\) perturbations are added and subtracted at squared symmetric points:
\[
\E_{\nu_1}[X^2] = \sum_{i=1}^J \left[ \left(\frac{w_i}{2} + \xi_i\right)(m_i\sigma)^2 + \left(\frac{w_i}{2} - \xi_i\right)(-m_i\sigma)^2 \right] = \sum_{i=1}^J w_i m_i^2\sigma^2 = \sigma^2.
\]
Thus, \(\Var_{\nu_1}(X) = \E_{\nu_1}[X^2] - (\E_{\nu_1}[X])^2 = \sigma^2 - \Delta^2 \le \sigma^2\).
\end{proof}

For the $K$-arm embedding, write $P_0:=\nu_0$ and $P_+:=\nu_1$, and define $P_-$ by reversing the signs of the perturbations:
\[
P_- := \left(1 - \sum_{i=1}^J w_i\right)\delta_0 + \sum_{i=1}^J \left(\frac{w_i}{2} - \xi_i\right)\delta_{m_i\sigma} + \sum_{i=1}^J \left(\frac{w_i}{2} + \xi_i\right)\delta_{-m_i\sigma}.
\]
By reflection, $\E_{P_-}[X]=-\Delta$ and $\Var_{P_-}(X)=\sigma^2-\Delta^2\le\sigma^2$. All three distributions admit the same valid anchor $c=0$, since their means have absolute value at most $\Delta\le\sigma$.

\subsection{A uniform one-query KL bound}
\label{app:lower-bound-query-kl}

Consider an arbitrary measurable query set \(A \subseteq \R\). A 1-bit observation \(Y = \1\{X \in A\}\) follows a Bernoulli distribution: \(\Ber(p)\) under \(\nu_0\) and \(\Ber(q)\) under \(\nu_1\), where \(p := \nu_0(A)\) and \(q := \nu_1(A)\).

We categorize the scales based on how they intersect with \(A\). Let \(S(A)\) denote the set of index scales \(i \in \{1, \dots, J\}\) where exactly one of the two support points \(\{\pm m_i\sigma\}\) belongs to \(A\).

\begin{lemma}[One-query probability difference bound]
\label{lem:one-query-diff-bound}
For every measurable set \(A \subseteq \R\), 
\[
(q - p)^2 \le \frac{2\Delta^2}{J\sigma^2}p.
\]
\end{lemma}

\begin{proof}
We analyze the mass difference \(q - p\). For any scale \(i \notin S(A)\), either both \(\pm m_i\sigma \in A\) or neither are in \(A\). If both are in \(A\), the mass difference contributed by scale \(i\) is 
\[
(\frac{w_i}{2} + \xi_i) + (\frac{w_i}{2} - \xi_i) - w_i = 0.
\]
If neither are, the contribution is naturally \(0\). The origin \(0\) also contributes equal mass to both distributions. 

Thus, only scales in \(S(A)\) cause a discrepancy. For \(i \in S(A)\), the difference is either \(+\xi_i\) (if \(m_i\sigma \in A\)) or \(-\xi_i\) (if \(-m_i\sigma \in A\)). We can write \(q - p = \sum_{i \in S(A)} s_i \xi_i\) for some signs \(s_i \in \{-1, 1\}\). By the triangle inequality:
\[
|q - p| \le \sum_{i \in S(A)} \xi_i = \frac{\Delta}{2J\sigma} \sum_{i \in S(A)} \frac{1}{m_i}.
\]
If \(S(A)\) is empty, then \(q-p=0\) and the claimed bound is immediate. Otherwise, let \(S(A) = \{i_1 < i_2 < \dots < i_r\}\) be ordered such that \(i_1\) is the smallest index. Because \(m_i\) grows geometrically as \(2^{i-1}\), the sum is dominated by its largest term:
\[
\sum_{k=1}^r \frac{1}{m_{i_k}} \le \sum_{k=0}^{r-1} \frac{1}{m_{i_1} 2^k} < \frac{2}{m_{i_1}}.
\]
Squaring this strict bound gives:
\[
\left(\sum_{i \in S(A)} \frac{1}{m_i}\right)^2 < \frac{4}{m_{i_1}^2}.
\]
Because \(i_1\) is an element of \(S(A)\), we have \(\dfrac{4}{m_{i_1}^2} \le 4 \sum_{i \in S(A)} \dfrac{1}{m_i^2}\). 

Simultaneously, we bound the baseline probability \(p\). For every \(i \in S(A)\), exactly one support point is in \(A\), contributing \(w_i/2\) to \(\nu_0(A)\). Therefore,
\[
p \ge \sum_{i \in S(A)} \frac{w_i}{2} = \frac{1}{2J} \sum_{i \in S(A)} \frac{1}{m_i^2} \implies \sum_{i \in S(A)} \frac{1}{m_i^2} \le 2Jp.
\]
Combining these inequalities gives the desired result:
\[
(q - p)^2 \le \frac{\Delta^2}{4J^2\sigma^2} \left(4 \sum_{i \in S(A)} \frac{1}{m_i^2}\right) \le \frac{\Delta^2}{J^2\sigma^2} (2Jp) = \frac{2\Delta^2}{J\sigma^2}p.
\]
\end{proof}

\begin{lemma}[Uniform KL divergence bound]
\label{lem:one-query-kl-bound}
There exists an absolute constant \(C_0 > 0\) such that for every measurable set \(A \subseteq \R\), 
\[
\max\{\mathrm{kl}(q,p), \mathrm{kl}(p,q)\} \le C_0 \frac{\Delta^2}{J\sigma^2},
\]
where \(\mathrm{kl}(u,v) = u \log\dfrac{u}{v} + (1-u)\log\dfrac{1-u}{1-v}\) denotes the binary KL divergence. Moreover, writing $p_\pm:=P_\pm(A)$, the opposite perturbations satisfy
\[
\max\{\mathrm{kl}(p_-,p_+),\mathrm{kl}(p_+,p_-)\}\le 4C_0\frac{\Delta^2}{J\sigma^2}.
\]
\end{lemma}

\begin{proof}
If $p=0$ or $p=1$, the perturbed laws assign the same probability to $A$, so all divergences vanish. Otherwise, we first establish that the likelihood ratio is bounded, specifically that 
\[
\frac{q}{p} \in \left[\frac12, \frac32\right].
\]
Using the structural properties of our construction, we bounded \(\xi_i \le w_i/4\). Therefore:
\[
|q - p| \le \sum_{i \in S(A)} \xi_i \le \frac{1}{4} \sum_{i \in S(A)} w_i = \frac{1}{2} \sum_{i \in S(A)} \frac{w_i}{2} \le \frac{1}{2} p.
\]
This implies \(-\dfrac{1}{2}p \le q - p \le \dfrac{1}{2}p\), which rearranges exactly to \(\dfrac{1}{2}p \le q \le \dfrac{3}{2}p\). 

By complement invariance of binary KL divergence, \(\mathrm{kl}(q,p) = \mathrm{kl}(1-q, 1-p)\). We may assume without loss of generality that \(p \le 1/2\) (if not, we simply replace \(A\) with its complement \(A^c\), noting that \(S(A) = S(A^c)\)). 

Using the standard inequality \(\mathrm{kl}(u,v) \le \dfrac{(u-v)^2}{v(1-v)}\) and Lemma~\ref{lem:one-query-diff-bound}, we obtain:
\[
\mathrm{kl}(q,p) \le \frac{(q-p)^2}{p(1-p)} \le \frac{(q-p)^2}{p(1/2)} = 2\frac{(q-p)^2}{p} \le 4\frac{\Delta^2}{J\sigma^2}.
\]
To bound the reverse direction \(\mathrm{kl}(p,q)\), we require \(q(1-q)\) in the denominator. Since \(p \le 1/2\), our likelihood bounds give \(q \le 3/4\), meaning \(1-q \ge 1/4\). Furthermore, \(q \ge p/2\). Therefore, \(q(1-q) \ge (p/2)(1/4) = p/8\). Thus:
\[
\mathrm{kl}(p,q) \le \frac{(p-q)^2}{q(1-q)} \le 8\frac{(q-p)^2}{p} \le 16\frac{\Delta^2}{J\sigma^2}.
\]
The first bound holds with the absolute constant \(C_0 = 16\). For the opposite perturbations, $p_+=q$ and $p_-=2p-q$, so $|p_\pm-p|\le p/2$. With $p\le1/2$ as above, both denominators satisfy $p_\pm(1-p_\pm)\ge p/8$. Since $p_+-p_-=2(q-p)$, the same binary KL inequality gives
\[
\max\{\mathrm{kl}(p_-,p_+),\mathrm{kl}(p_+,p_-)\}
\le 8\frac{4(q-p)^2}{p}
\le 64\frac{\Delta^2}{J\sigma^2}
=4C_0\frac{\Delta^2}{J\sigma^2}.
\]
\end{proof}

\subsection{Embedding and change of measure}
\label{app:lower-bound-change-measure}

Define the $K$-arm base instance $\nu^{(0)}:=(P_0,P_-,\ldots,P_-)$, where arm $1$ is optimal and all suboptimal gaps equal $\Delta$. For each $i\in\{2,\ldots,K\}$, let $\nu^{(i)}$ differ only in arm $i$, whose distribution is changed from $P_-$ to $P_+$. Arm $i$ is then uniquely optimal, with gap $\Delta$ to arm $1$. Let $\Pr_0$ and $\Pr_i$ denote the corresponding probability measures induced by the interaction of the algorithm, and let $\E_0$ and $\E_i$ denote their expectations. The same anchors $c_k=0$ are supplied under every instance. Write $N_i(t)$ for the number of queries to arm $i$ up to time $t$.

Let \(\tau\) be the algorithm's stopping time. At each round \(t \le \tau\), the algorithm selects an arm \(K_t \in \{1,\ldots,K\}\) and a measurable query set \(A_t \subseteq \R\), then observes the 1-bit feedback \(Y_t = \1\{X_t \in A_t\}\) where \(X_t\) is drawn from the distribution of the pulled arm. To facilitate a clean change-of-measure argument over a deterministic time horizon, we formally extend this to an infinite padded transcript \(Z_t := (\tilde{K}_t, \tilde{A}_t, \tilde{Y}_t)\) defined for all integers \(t \ge 1\). 

Let \(\emptyset\) represent an empty query set and \(\perp\) a null observation where \(\perp \notin \{0,1\}\), and define the padded variables piecewise:
\[
\tilde{K}_t := \begin{cases} 
K_t, & \text{if } t \le \tau, \\ 
1, & \text{if } t > \tau, 
\end{cases}
\qquad
\tilde{A}_t := \begin{cases} 
A_t, & \text{if } t \le \tau, \\ 
\emptyset, & \text{if } t > \tau, 
\end{cases}
\qquad
\tilde{Y}_t := \begin{cases} 
Y_t, & \text{if } t \le \tau, \\ 
\perp, & \text{if } t > \tau. 
\end{cases}
\]
By construction, once the algorithm stops (\(t > \tau\)), the transcript deterministically locks to the state \((1, \emptyset, \perp)\), ensuring no further statistical information is gathered under either instance. We denote the history up to time \(t\) as \(Z_{1:t} = (Z_1, \dots, Z_t)\), which, together with an independent initial random seed $U$ containing the algorithm's internal randomness, generates the natural filtration \(\mathcal{F}_t = \sigma(U,Z_{1:t})\). The law of $U$ is the same under every instance. Consequently, both the stopping time \(\tau\) and the final arm recommendation \(\widehat{k}\) are measurable with respect to the infinite-horizon \(\sigma\)-algebra \(\mathcal{F}_\infty = \sigma\big(\bigcup_{t=1}^\infty \mathcal{F}_t\big)\).

\begin{lemma}[Data processing for the error event]
\label{lem:kl-lower-bound-delta}
If the algorithm is \(\delta\)-correct over the instance class, then for each $i\ge2$
\[
\mathrm{KL}(\Pr_0, \Pr_i) \ge c_1 \log(1/\delta)
\]
for an absolute constant \(c_1 > 0\). The same conclusion holds for $(\eps,\delta)$-PAC algorithms when $0<\eps<\Delta/2$, since only the unique best arm is $\eps$-optimal in either instance.
\end{lemma}

\begin{proof}
Let \(E := \{\widehat{k} = 1\}\) be the event that the algorithm recommends arm 1. By the \(\delta\)-correctness requirement, the algorithm must identify the correct arm with probability at least \(1-\delta\). Therefore, \(\Pr_0(E) \ge 1 - \delta\) and \(\Pr_i(E) \le \delta\). Because the recommendation \(\widehat{k}\) is entirely determined by the full transcript \((U,(Z_t)_{t \ge 1})\), the data-processing inequality implies:
\[
\mathrm{KL}(\Pr_0, \Pr_i) \ge \mathrm{kl}(\Pr_0(E), \Pr_i(E)) \ge \mathrm{kl}(1-\delta, \delta).
\]
For any \(\delta \in (0, 1/4)\), 
\[
\mathrm{kl}(1-\delta, \delta)
\ge (1-2\delta)\log\frac{1-\delta}{\delta}
\ge c_1 \log(1/\delta)
\]
for some \(c_1 > 0\).
\end{proof}

\begin{lemma}[Transcript divergence bound]
\label{lem:kl-upper-bound-tau}
For each $i\ge2$, \(\mathrm{KL}(\Pr_0, \Pr_i) \le 4C_0 \dfrac{\Delta^2}{J\sigma^2} \E_0[N_i(\tau)]\).
\end{lemma}

\begin{proof}
Fix $i\ge2$. Let \(\Pr_0^{(T)}\) and \(\Pr_i^{(T)}\) denote the marginal laws of the truncated transcript \((U,Z_{1:T})\). By the chain rule for relative entropy:
\[
\mathrm{KL}(\Pr_0^{(T)}, \Pr_i^{(T)}) = \sum_{t=1}^T \E_0\left[ \mathrm{KL}(\Pr_0(Z_t \mid U,Z_{1:t-1}) \parallel \Pr_i(Z_t \mid U,Z_{1:t-1})) \right].
\]
Let \(\mathcal{F}_{t-1}\) be the \(\sigma\)-algebra generated by \((U,Z_{1:t-1})\). Conditioned on \(\mathcal{F}_{t-1}\), the event \(\{t \le \tau\}\) is determined. 

If \(t > \tau\), the process has stopped, and \(Z_t = (1, \emptyset, \perp)\) is deterministic under both instances. The conditional KL divergence is exactly \(0\). 

If \(t \le \tau\), the arm \(K_t\) and query set \(A_t\) are deterministic functions of \(\mathcal{F}_{t-1}\). The only randomness lies in the Bernoulli observation \(Y_t\). If $K_t\ne i$, its distribution is unchanged between $\nu^{(0)}$ and $\nu^{(i)}$, so the conditional divergence is $0$. If $K_t=i$, then $Y_t\sim\Ber(P_-(A_t))$ under $\nu^{(0)}$ and $Y_t\sim\Ber(P_+(A_t))$ under $\nu^{(i)}$. By Lemma~\ref{lem:one-query-kl-bound}, the conditional divergence is bounded by $4C_0\Delta^2/(J\sigma^2)$. Thus:
\[
\mathrm{KL}(\Pr_0^{(T)}, \Pr_i^{(T)}) \le \sum_{t=1}^T \E_0\left[ \1\{t \le \tau,\ K_t=i\} 4C_0 \frac{\Delta^2}{J\sigma^2} \right] = 4C_0 \frac{\Delta^2}{J\sigma^2} \E_0[N_i(T \wedge \tau)].
\]
Because \(\mathcal{F}_T\) forms an increasing filtration generating the full transcript, \(\mathrm{KL}(\Pr_0, \Pr_i) = \lim_{T \to \infty} \mathrm{KL}(\Pr_0^{(T)}, \Pr_i^{(T)})\). Applying the monotone convergence theorem to \(N_i(T \wedge \tau)\) yields the final bound.
\end{proof}

\begin{proof}[Proof of Theorem~\ref{thm:lower-bound}]
Assuming \(\E_0[\tau] < \infty\) (otherwise the lower bound holds trivially), we link the information complexity from Lemma~\ref{lem:kl-upper-bound-tau} to the statistical requirement from Lemma~\ref{lem:kl-lower-bound-delta}, for each $i=2,\ldots,K$:
\[
4C_0 \frac{\Delta^2}{J\sigma^2} \E_0[N_i(\tau)] \ge c_1 \log\left(\frac{1}{\delta}\right).
\]
Rearranging this inequality for the expected sample complexity yields:
\[
\E_0[N_i(\tau)] \ge \frac{c_1}{4C_0} \frac{\sigma^2 J}{\Delta^2} \log\left(\frac{1}{\delta}\right).
\]
All these inequalities hold under the same base instance. Summing over its $K-1$ suboptimal arms gives
\[
\E_0[\tau]\ge\sum_{i=2}^K\E_0[N_i(\tau)]\ge\frac{c_1}{4C_0}(K-1)\frac{\sigma^2 J}{\Delta^2}\log\left(\frac{1}{\delta}\right).
\]
Recalling our construction choice that \(J = \lfloor \log_2(\frac{\sigma}{\Delta}) \rfloor\), it holds that \(J \ge c_2 \log(\frac{\sigma}{\Delta})\) for some constant \(c_2 > 0\) given the restriction \(\Delta \le \sigma/4\). Substituting this structural scale count gives:
\[
\E_0[\tau] \ge c(K-1) \frac{\sigma^2}{\Delta^2} \log\left(\frac{\sigma}{\Delta}\right) \log\left(\frac{1}{\delta}\right),
\]
where \(c = \frac{c_1 c_2}{4C_0} > 0\) is an absolute constant, completing the proof.
\end{proof}
\newpage
\section{Detailed algorithms}
\label{app:detailed-algorithms}

This section gives the full pseudocode for the routines described in Section~\ref{sec:algorithm}. The PAC theorems in the main text assume that valid anchors $c_k$ satisfying $|\mu_k-c_k|\le\sigma$ are supplied. If anchors are not supplied but the means are known to lie in $[-\lambda,\lambda]$, the algorithms call \textsc{MedianEst} (Algorithm~\ref{alg:medianest} in Appendix~\ref{app:medianest}) for each arm and then run the BAI routine with the remaining confidence budget. This adds the localization cost stated in Remark~\ref{rem:assumptions}.

\begin{algorithm}[ht]
\caption{One randomized tail-integral query, \textsc{TailIntegralQuery}}
\label{alg:tail-integral-query}
\begin{algorithmic}[1]
\Require arm $k$, anchor $c_k$, clipping level $L\ge\sigma$, variance proxy $\sigma$
\State Set $\alpha_L\gets 1+\log(L/\sigma)$ and $g_L(t)\gets [\alpha_L\max\{\sigma,t\}]^{-1}$ for $t\in[0,L]$
\State Draw $S\in\{+1,-1\}$ uniformly and draw $T\sim g_L$
\If{$S=+1$}
    \State Query arm $k$ with threshold set $[c_k+T,\infty)$ and observe $Y=\1\{X_k\ge c_k+T\}$
\Else
    \State Query arm $k$ with threshold set $(-\infty,c_k-T]$ and observe $Y=\1\{X_k\le c_k-T\}$
\EndIf
\State \Return $Z\gets 2S Y/g_L(T)$
\end{algorithmic}
\end{algorithm}

\begin{algorithm}[ht]
\caption{\textsc{FixedClip-1BitBAI}}
\label{alg:fixed-clip-bai}
\begin{algorithmic}[1]
\Require arms $\A=[K]$, variance proxy $\sigma$, accuracy $\eps$, confidence $\delta$, anchors $(c_k)_{k=1}^K$ or range $\lambda$
\If{anchors are not supplied}
    \For{$k=1,\dots,K$}
        \State $c_k\gets \Call{MedianEst}{k,\lambda,\sigma,\delta/(2K)}$
    \EndFor
    \State Set $\delta_0\gets \delta/2$
\Else
    \State Set $\delta_0\gets \delta$
\EndIf
\State Set $L_\eps\gets 2\sigma^2/\eps$, $\gamma\gets \delta_0/K$
\For{$k=1,\dots,K$}
    \State Initialize an empty list $\mathcal Z_k$
    \State Append $\Call{TailIntegralQuery}{k,c_k,L_\eps,\sigma}$ to $\mathcal Z_k$
\EndFor
\Loop
    \For{$k=1,\dots,K$}
        \State $N_k\gets |\mathcal Z_k|$, $\widehat\mu_k\gets c_k+N_k^{-1}\sum_{Z\in\mathcal Z_k}Z$
        \State $\beta_k\gets \beta_\eps(N_k,\gamma)$
        \State $\LCB_k\gets \widehat\mu_k-\beta_k$, $\UCB_k\gets \widehat\mu_k+\beta_k$
    \EndFor
    \For{$k=1,\dots,K$}
        \State $B_k\gets \max_{\ell\ne k}\UCB_\ell-\LCB_k$
    \EndFor
    \State $h\gets \argmin_{k\in\A} B_k$, $\ell\gets \argmax_{a\ne h}\UCB_a$
    \If{$B_h\le\eps$}
        \State \Return $h$
    \EndIf
    \State $k_t\gets \argmax_{a\in\{h,\ell\}}\beta_a$
    \State Append $\Call{TailIntegralQuery}{k_t,c_{k_t},L_\eps,\sigma}$ to $\mathcal Z_{k_t}$
\EndLoop
\end{algorithmic}
\end{algorithm}

\begin{algorithm}[ht]
\caption{\textsc{RestartPhasedClip-1BitBAI}}
\label{alg:phased-clip-bai}
\begin{algorithmic}[1]
\Require arms $\A=[K]$, variance proxy $\sigma$, accuracy $\eps$, confidence $\delta$, anchors $(c_k)_{k=1}^K$ or range $\lambda$
\If{anchors are not supplied}
    \For{$k=1,\dots,K$}
        \State $c_k\gets \Call{MedianEst}{k,\lambda,\sigma,\delta/(2K)}$
    \EndFor
    \State Set $\delta_0\gets \delta/2$
\Else
    \State Set $\delta_0\gets \delta$
\EndIf
\State Set $J\gets\lceil\log_2(\sigma/\eps)\rceil$, active set $A\gets\A$, and $\gamma\gets\delta_0/[K(J+1)]$
\For{$j=0,1,\dots,J$}
    \State Set $\rho_j\gets 2^{-j}\sigma$, $L_j\gets 32\sigma^2/\rho_j$, $r_j\gets \rho_j/16$
    \For{$k\in A$}
        \State Restart a fresh phase list $\mathcal Z_{k,j}\gets\emptyset$
    \EndFor
    \Loop
        \For{$k\in A$}
            \If{$|\mathcal Z_{k,j}|=0$}
                \State $\beta_{k,j}\gets+\infty$, $\LCB_{k,j}\gets-\infty$, $\UCB_{k,j}\gets+\infty$
            \Else
                \State $N_{k,j}\gets |\mathcal Z_{k,j}|$, $\widehat\mu_{k,j}\gets c_k+N_{k,j}^{-1}\sum_{Z\in\mathcal Z_{k,j}}Z$
                \State $\beta_{k,j}\gets \beta_j(N_{k,j})$
                \State $\LCB_{k,j}\gets\widehat\mu_{k,j}-\beta_{k,j}$, $\UCB_{k,j}\gets\widehat\mu_{k,j}+\beta_{k,j}$
            \EndIf
        \EndFor
        \State $m\gets \max_{i\in A}\LCB_{i,j}$ and remove every $k\in A$ with $\UCB_{k,j}<m-\eps$
        \If{$|A|=1$}
            \State \Return the unique arm in $A$
        \EndIf
        \State $h\gets \argmax_{k\in A}\LCB_{k,j}$, $\ell\gets \argmax_{a\in A\setminus\{h\}}\UCB_{a,j}$
        \If{$\UCB_{\ell,j}-\LCB_{h,j}\le\eps$}
            \State \Return $h$
        \EndIf
        \If{$\max_{k\in A}\beta_{k,j}\le r_j$}
            \If{$j=J$}
                \State \Return $h$
            \Else
                \State \textbf{break} and move to phase $j+1$
            \EndIf
        \EndIf
        \If{$\max\{\beta_{h,j},\beta_{\ell,j}\}>r_j$}
            \State $k_t\gets \argmax_{a\in\{h,\ell\}}\beta_{a,j}$
        \Else
            \State $k_t\gets \argmax_{a\in A}\beta_{a,j}$
        \EndIf
        \State Append $\Call{TailIntegralQuery}{k_t,c_{k_t},L_j,\sigma}$ to $\mathcal Z_{k_t,j}$
    \EndLoop
\EndFor
\end{algorithmic}
\end{algorithm}
\newpage
\section{Numerical experiment details}
\label{app:numerical}
\subsection{1-bit mean estimation}
\label{app:numerical-estimation}

\paragraph{Experimental setting.}
We compare our tail-integral estimator with the estimator of \cite{laujohnathansequential} on three scalar distributions:
Normal($\mu,1$), Student-t($5$) scaled to variance 1, and a shifted exponential law with variance 1.
In all cases we set
\[
\mu = 1, \qquad \sigma = 1, \qquad c = 0,
\]
so that the centered variable is \(W=X-c\). We evaluate both estimators over \(100\) logarithmically spaced sample sizes between \(10\) and \(10^7\), and estimate the MAE
\[
\mathbb{E}\bigl[|\hat{\mu}_n-\mu|\bigr]
\]
by averaging over \(1000\) Monte Carlo repetitions.

\paragraph{Bias-target calibration.}
We use a common deterministic bias target
\[
b = 0.1
\]
for both estimators. For our estimator, Theorem~\ref{thm:confidence-bound} contains the clipping-bias term \(\sigma^2/(2L)\), so we choose
\[
L = \frac{\sigma^2}{2b}.
\]
In the implementation of the baseline experiment, the \cite{laujohnathansequential} truncation parameter \(t\) is calibrated to the same target \(b\). With \(\sigma=1\) and \(b=0.1\), this gives
\[
L=t=5.
\]

\paragraph{True bias computation.}
Beyond the common nominal bias budget \(b\), we also compute the exact deterministic bias induced by the chosen clipping/truncation level for each distribution. For our estimator, the true bias is
\[
B_{\mathrm{ours}}
:=
\left| \mathbb{E}\bigl[[W]_L - W\bigr] \right|,
\]
where \([W]_L=\max\{-L,\min\{W,L\}\}\). For the \cite{laujohnathansequential} estimator, the true bias is
\[
B_{\mathrm{LS}}
:=
\left| \mathbb{E}\bigl[W\mathbf{1}\{|W|<t\}\bigr]-\mathbb{E}[W]\right|
=
\left|\mathbb{E}\bigl[W\mathbf{1}\{|W|\ge t\}\bigr]\right|.
\]
These quantities are evaluated from the known distribution in each environment and are plotted as horizontal reference lines in Figure~\ref{fig:1-bit-estimator-mae} whenever they exceed \(10^{-3}\).

\begin{table}[h]
\centering
\begin{tabular}{lcc}
\toprule
Distribution & \(B_{\mathrm{ours}}\) & \(B_{\mathrm{LS}}\) \\
\midrule
Normal & \(7.15\times 10^{-6}\) & \(1.66\times 10^{-4}\) \\
Student-\(t(5)\) & \(1.56\times 10^{-3}\) & \(9.06\times 10^{-3}\) \\
Exponential & \(6.74\times 10^{-3}\) & \(4.04\times 10^{-2}\) \\
\bottomrule
\end{tabular}
\vspace{0.5em}
\caption{True bias under the common theoretic bias target \(b=0.1\).}
\label{tab:1-bit-bias-values}
\end{table}

\begin{figure}[ht]
\centering
\includegraphics[width=\linewidth]{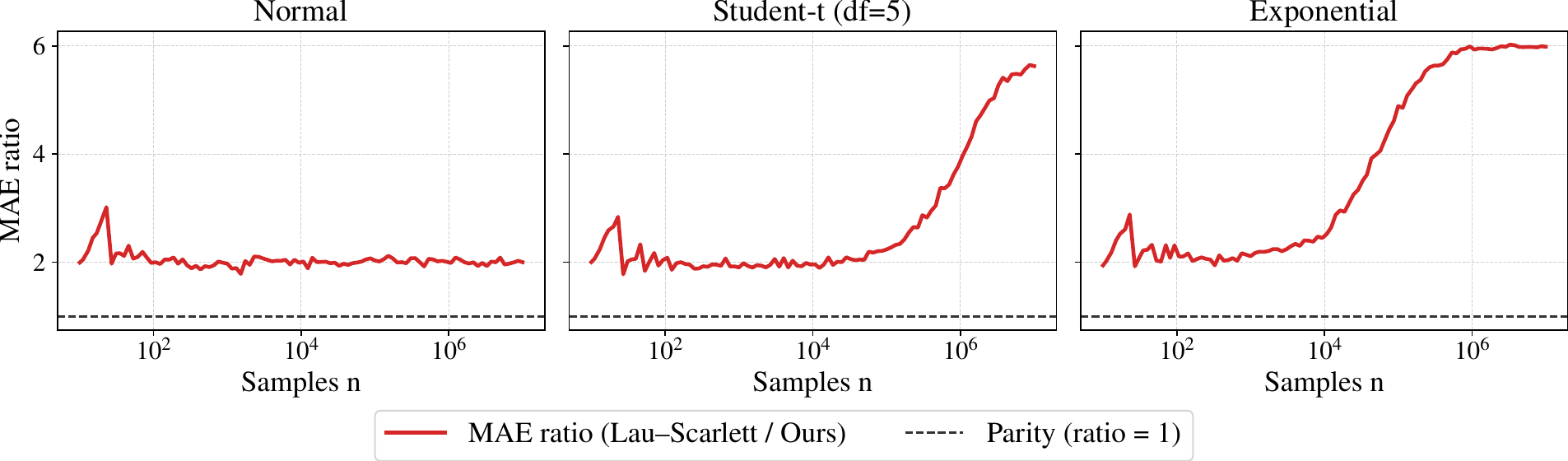}
\caption{Ratio \(\mathrm{MAE}_{\mathrm{LS}}(n)/\mathrm{MAE}_{\mathrm{ours}}(n)\) as a function of \(n\). Values above \(1\) indicate that the \cite{laujohnathansequential} estimator has larger error.}
\label{fig:1-bit-estimator-ratio}
\end{figure}

\paragraph{Discussion of the ratio plot.}
Figure~\ref{fig:1-bit-estimator-ratio} highlights the MAE comparison more directly. In the pre-asymptotic regime, where neither estimator has yet reached its bias floor, the ratio is close to \(2\) across all three distributions. In other words, before bias dominates, the \cite{laujohnathansequential} estimator has roughly double the MAE of our estimator. As \(n\) grows, the ratio increases further in the heavy-tailed and asymmetric cases because the larger true bias of the baseline causes its MAE to plateau earlier and at a higher level.

\subsection{Variance comparison with Lau--Scarlett}
\label{app:numerical-variance}
The bands in Figure~\ref{fig:1-bit-estimator-mae} describe absolute errors and do not directly measure the variance of the signed estimates. To compare dispersion, we use the signed $n$-sample estimates from the same $R=1000$ independent repetitions and the same calibration $\mu=1$, $\sigma=1$, $c=0$, $L=t=5$. For either estimator, compute
\[
\widehat{\Var}(\hat\mu_n)=\frac1{R-1}\sum_{r=1}^R
\left(\hat\mu_n^{(r)}-\overline{\hat\mu}_n\right)^2,
\qquad
\overline{\hat\mu}_n=\frac1R\sum_{r=1}^R\hat\mu_n^{(r)}.
\]
Table~\ref{tab:1-bit-variance} reports its square root.

\begin{table}[ht]
\centering
\small
\begin{tabular}{lrrrr}
\toprule
Distribution & $n$ & Std. (ours) & Std. (Lau--Scarlett) & Ratio (LS/ours) \\
\midrule
Normal & $10^4$ & $0.02515$ & $0.04917$ & $1.955$ \\
 & $10^5$ & $0.007806$ & $0.01624$ & $2.081$ \\
 & $10^6$ & $0.002472$ & $0.005166$ & $2.090$ \\
 & $10^7$ & $0.0007860$ & $0.001537$ & $1.956$ \\
\midrule
Student-$t(5)$ & $10^4$ & $0.02532$ & $0.04995$ & $1.973$ \\
 & $10^5$ & $0.007799$ & $0.01585$ & $2.033$ \\
 & $10^6$ & $0.002452$ & $0.005002$ & $2.040$ \\
 & $10^7$ & $0.0007863$ & $0.001557$ & $1.980$ \\
\midrule
Exponential & $10^4$ & $0.02469$ & $0.04917$ & $1.992$ \\
 & $10^5$ & $0.007285$ & $0.01506$ & $2.067$ \\
 & $10^6$ & $0.002413$ & $0.004916$ & $2.037$ \\
 & $10^7$ & $0.0007716$ & $0.001485$ & $1.924$ \\
\bottomrule
\end{tabular}
\caption{Empirical standard deviations of signed estimates over $1000$ repetitions, under the common calibration. Ratios above one indicate greater dispersion for the Lau--Scarlett baseline.}
\label{tab:1-bit-variance}
\end{table}

Across these twelve configurations, the baseline's standard deviation is $1.924$--$2.090$ times ours, corresponding to approximately $3.70$--$4.37$ times the empirical variance. Under this calibration our estimator therefore has smaller dispersion as well as smaller clipping bias. This comparison concerns the implemented fixed-budget estimators on these distributions; it establishes neither a universal variance ordering nor a sample-complexity rate separation. Variance is computed about each estimator's empirical mean and is distinct from mean squared error, which also includes bias.

\subsection{1-bit best-arm identification}
\label{app:numerical-bai}

\paragraph{Experiment Setting}
The empirical performance of the proposed 1-bit best-arm identification algorithms was evaluated against the full-feedback \textsc{UGapE} baseline across three distinct environments. For all experiments, we evaluate a $K=5$ armed bandit setting with a target accuracy of $\epsilon = 0.05$ and a maximum query cap of $N = 20000$. To thoroughly analyze the sample complexity and safety of each approach, the confidence parameter $\delta$ is varied over 10 logarithmically spaced values from $0.01$ to $0.5$. The results are averaged over $5000$ independent Monte Carlo runs for each configuration.

\paragraph{BAI Environments and Distributions}
The experiments utilize three distinct reward environments, all bounded by a global variance proxy of $\sigma = 0.5$:
\begin{itemize}
    \item \textbf{Bernoulli}: The reward distributions are Bernoulli with means $\mu = [0.5, 0.4, 0.4, 0.4, 0.4]$. This represents a uniform gap setting where $\Delta_k = 0.1$ for all suboptimal arms. The arm-wise anchors are $c_k=0.5$ for all arms and satisfy the localization condition.
    \item \textbf{Normal}: The rewards follow a Gaussian distribution with linearly decaying means $\mu = [0.5, 0.4, 0.3, 0.2, 0.1]$. The arm standard deviations are set to $\sigma_k = 0.5$, and the anchors are uniformly set to $c_k = 0$.
    \item \textbf{Student-t}: The rewards are drawn from a heavy-tailed Student-t distribution with degrees of freedom $df=3$. The distributions are scaled by a factor of $\sigma \sqrt{(df-2)/df}$ so that the true variance is precisely $\sigma^2$. The arm means are $\mu = [0.5, 0.4, 0.3, 0.2, 0.1]$ and the anchors are $c_k = 0$.
\end{itemize}

\paragraph{Algorithms and Exploration Parameter Tuning}
We evaluate three variations of our 1-bit algorithm—\textsc{FixedClip}, \textsc{RestartPhase}, and \textsc{ReusePhase}—against the full-feedback \textsc{UGapE} baseline.

In practice, theoretically derived confidence intervals can be overly conservative. For the empirical comparison, we introduce a constant exploration multiplier to the leading confidence radius term. For the 1-bit algorithms, this is implemented as an \texttt{exploration} parameter that directly scales the logarithmic $x_n(\gamma)$ term within the tail-integral confidence bounds.

The exploration parameters were tuned as follows:
\begin{itemize}
    \item \textbf{\textsc{FixedClip}}: The exploration parameter was tuned to $0.04$ for the Bernoulli environment and $0.05$ for the Normal and Student-t environments.
    \item \textbf{\textsc{RestartPhase} \& \textsc{ReusePhase}}: The exploration parameter was consistently set to $0.03$ across all three environments.
    \item \textbf{\textsc{UGapE} Baseline}: The exploration constant $c$ scaling the empirical Bernstein/Hoeffding radius was tuned per environment to $0.06$ for Bernoulli, $0.0002$ for Normal, and $0.0005$ for Student-t.
\end{itemize}

\paragraph{\textsc{UGapE} Modifications for Unbounded Support}
For bounded-support environments (e.g., Bernoulli), the \textsc{UGapE} baseline uses the raw full-feedback rewards. However, for unbounded-support environments (such as the Normal and Student-t settings), the \textsc{UGapE} baseline is modified to include finite-variance clipping. In these unbounded settings, it still receives full continuous feedback, but it clips every observed reward around the arm's anchor using the same fixed clipping scale $L$ as the \textsc{FixedClip} algorithm. The corresponding finite-variance clipping bias is then explicitly added to the algorithm's confidence radius.

The reported empirical error frequencies remain at or below the target $\delta$ in the tested configurations. Reducing the exploration multipliers below their theorem values and imposing the experimental query cap do not preserve the proved PAC guarantee automatically. These results are empirical assessments of the tuned implementations; Theorems~\ref{thm:1-bit-bai-upper} and~\ref{thm:phased-bai-upper} concern the stated, untuned confidence radii and stopping rules.

\subsection{Compute resources}
\label{subsec:compute-resources}
All numerical simulations were executed on a dedicated Linux workstation equipped with an Intel(R) Core(TM) Ultra 9 285K processor (24 cores, up to 5.1 GHz) and 128 GB of RAM.

In terms of computational cost, the standalone 1-bit mean estimation experiments required approximately 40 minutes of total execution time. The best-arm identification (BAI) experiments are more computationally intensive due to the adaptive sampling and large number of Monte Carlo iterations (5000 runs per configuration); each BAI experiment (Bernoulli, Normal, and Student-t) required approximately 10 hours to complete.
\newpage
\section*{NeurIPS Paper Checklist}

\begin{enumerate}

\item {\bf Claims}
    \item[] Question: Do the main claims made in the abstract and introduction accurately reflect the paper's contributions and scope?
    \item[] Answer: \answerYes{}
    \item[] Justification: The abstract and Introduction (Section~\ref{sec:introduction}) clearly outline the development of the 1-bit mean estimator (Section~\ref{sec:estimator}), the fixed and phased clipping algorithms (Section~\ref{sec:algorithm}), the lower bound (Section~\ref{sec:lower_bound}), and empirical validation (Section~\ref{sec:numerical}), which match the paper's theoretical and experimental results.

\item {\bf Limitations}
    \item[] Question: Does the paper discuss the limitations of the work performed by the authors?
    \item[] Answer: \answerYes{} 
    \item[] Justification: The paper discusses the required variance proxy and anchor point assumptions in Remark~\ref{rem:assumptions} and Appendix~\ref{app:medianest}, explicitly notes the sample complexity mixed-term penalty incurred by the fixed-clipping approach in Section~\ref{subsec:sample-complexity-discussion}, and further summarizes these limitations in Section~\ref{sec:conclusion}.

\item {\bf Theory assumptions and proofs}
    \item[] Question: For each theoretical result, does the paper provide the full set of assumptions and a complete (and correct) proof?
    \item[] Answer: \answerYes{} 
    \item[] Justification: The required assumptions are explicitly stated in Section~\ref{sec:introduction} (Problem Setting), and complete rigorous proofs for all lemmas and theorems are provided in Appendix~\ref{app:estimator}, Appendix~\ref{app:algorithm-proofs}, and Appendix~\ref{app:lower_bound}.

    \item {\bf Experimental result reproducibility}
    \item[] Question: Does the paper fully disclose all the information needed to reproduce the main experimental results of the paper to the extent that it affects the main claims and/or conclusions of the paper (regardless of whether the code and data are provided or not)?
    \item[] Answer: \answerYes{} 
    \item[] Justification: Section~\ref{sec:numerical} and Appendix~\ref{app:numerical} specify the synthetic environments (Bernoulli, Normal, Student-t), target accuracy, confidence parameters, algorithmic hyperparameters, and the number of Monte Carlo runs.

\item {\bf Open access to data and code}
    \item[] Question: Does the paper provide open access to the data and code, with sufficient instructions to faithfully reproduce the main experimental results, as described in supplemental material?
    \item[] Answer: \answerYes{} 
    \item[] Justification: An anonymized .zip file containing the code, dependency specifications, and instructions necessary to faithfully reproduce the main experimental results is included in the supplementary material, in accordance with the NeurIPS Code and Data Submission Guidelines.

\item {\bf Experimental setting/details}
    \item[] Question: Does the paper specify all the training and test details (e.g., data splits, hyperparameters, how they were chosen, type of optimizer) necessary to understand the results?
    \item[] Answer: \answerYes{} 
    \item[] Justification: Appendix~\ref{app:numerical} clearly defines the environment setups, arm means, gaps, variances, and algorithmic exploration hyperparameters for both the estimator and BAI experiments.

\item {\bf Experiment statistical significance}
    \item[] Question: Does the paper report error bars suitably and correctly defined or other appropriate information about the statistical significance of the experiments?
    \item[] Answer: \answerYes{} 
    \item[] Justification: The results are averaged over 1000 to 5000 independent Monte Carlo runs. Figure~\ref{fig:1-bit-estimator-mae} and Figure~\ref{fig:1-bit-best-arm-identification} reports standard deviations of the plotted quantities, while Appendix~\ref{app:numerical-variance} compares dispersion of signed estimates. Empirical error rates of the tuned BAI implementations are distinguished from the theoretical PAC guarantees in Appendix~\ref{app:numerical-bai}.

\item {\bf Experiments compute resources}
    \item[] Question: For each experiment, does the paper provide sufficient information on the computer resources (type of compute workers, memory, time of execution) needed to reproduce the experiments?
    \item[] Answer: \answerYes{} 
    \item[] Justification: Appendix~\ref{app:numerical} explicitly details the compute resources (Intel Core Ultra 9 285K CPU, 128 GB RAM) and the execution times for both the estimator experiment (40 minutes) and the BAI experiments (10 hours each).
    
\item {\bf Code of ethics}
    \item[] Question: Does the research conducted in the paper conform, in every respect, with the NeurIPS Code of Ethics \url{https://neurips.cc/public/EthicsGuidelines}?
    \item[] Answer: \answerYes{} 
    \item[] Justification: The research focuses on foundational theoretical and algorithmic machine learning methodology (Sections~\ref{sec:estimator} to~\ref{sec:lower_bound}), which fully conforms to the NeurIPS Code of Ethics.

\item {\bf Broader impacts}
    \item[] Question: Does the paper discuss both potential positive societal impacts and negative societal impacts of the work performed?
    \item[] Answer: \answerNA{} 
    \item[] Justification: This is foundational theoretical research on best-arm identification (Sections~\ref{sec:algorithm} and~\ref{sec:lower_bound}) evaluated on synthetic distributions (Section~\ref{sec:numerical}), presenting no immediate foreseeable negative societal impacts or direct paths to malicious applications.
    
\item {\bf Safeguards}
    \item[] Question: Does the paper describe safeguards that have been put in place for responsible release of data or models that have a high risk for misuse (e.g., pre-trained language models, image generators, or scraped datasets)?
    \item[] Answer: \answerNA{} 
    \item[] Justification: The paper does not introduce or release high-risk models, language models, or scraped datasets (all data is simulated as per Appendix~\ref{app:numerical}).

\item {\bf Licenses for existing assets}
    \item[] Question: Are the creators or original owners of assets (e.g., code, data, models), used in the paper, properly credited and are the license and terms of use explicitly mentioned and properly respected?
    \item[] Answer: \answerNA{} 
    \item[] Justification: The research relies on standard mathematical formulations and newly generated synthetic data (Normal, Student-t, Bernoulli) as detailed in Appendix~\ref{app:numerical}, rather than existing third-party proprietary assets.

\item {\bf New assets}
    \item[] Question: Are new assets introduced in the paper well documented and is the documentation provided alongside the assets?
    \item[] Answer: \answerNA{} 
    \item[] Justification: The paper introduces new algorithms (Appendix~\ref{app:detailed-algorithms}) but does not release new datasets or pre-trained models.

\item {\bf Crowdsourcing and research with human subjects}
    \item[] Question: For crowdsourcing experiments and research with human subjects, does the paper include the full text of instructions given to participants and screenshots, if applicable, as well as details about compensation (if any)? 
    \item[] Answer: \answerNA{} 
    \item[] Justification: The research evaluates algorithms on synthetic statistical environments (Appendix~\ref{app:numerical}) and does not involve human subjects.

\item {\bf Institutional review board (IRB) approvals or equivalent for research with human subjects}
    \item[] Question: Does the paper describe potential risks incurred by study participants, whether such risks were disclosed to the subjects, and whether Institutional Review Board (IRB) approvals (or an equivalent approval/review based on the requirements of your country or institution) were obtained?
    \item[] Answer: \answerNA{} 
    \item[] Justification: The research utilizes synthetic distributions (Appendix~\ref{app:numerical}) and does not involve human subjects.

\item {\bf Declaration of LLM usage}
    \item[] Question: Does the paper describe the usage of LLMs if it is an important, original, or non-standard component of the core methods in this research? Note that if the LLM is used only for writing, editing, or formatting purposes and does \emph{not} impact the core methodology, scientific rigor, or originality of the research, declaration is not required.
    \item[] Answer: \answerNA{} 
    \item[] Justification: LLMs are not used as a component of the core theoretical methodology or algorithm design (Sections~\ref{sec:estimator} to~\ref{sec:lower_bound}).

\end{enumerate}
\end{document}